\documentclass[twoside]{article}

\usepackage[accepted]{aistats2023}
\usepackage[round]{natbib}

\usepackage{hyperref}       
\usepackage{url}
\usepackage{subcaption}     
\usepackage{amsfonts} 
\usepackage{amsthm}
\usepackage[shortlabels]{enumitem}
\usepackage{algorithmic}
\usepackage[ruled,vlined]{algorithm2e}
\newtheorem{theorem}{Theorem}[section]

\newtheorem{lemma}[theorem]{Lemma}

\newtheorem{definition}{Definition}
\usepackage{graphicx}
\usepackage{tikz}           
\usepackage{amsmath}
\newcommand{\indep}{\perp \!\!\! \perp}

\usepackage{xr}

\begin{document}

%

%

\twocolumn[

\aistatstitle{Ultra-marginal Feature Importance: Learning from Data with Causal Guarantees}

\aistatsauthor{ Joseph Janssen \And Vincent Guan \And Elina Robeva }

\aistatsaddress{ University of British Columbia \And University of British Columbia\And University of British Columbia } ]

\begin{abstract}
  Scientists frequently prioritize learning from data rather than training the best possible model; however, research in machine learning often prioritizes the latter. Marginal contribution feature importance (MCI) was developed to break this trend by providing a useful framework for quantifying the relationships in data. In this work, we aim to improve upon the theoretical properties, performance, and runtime of MCI by introducing ultra-marginal feature importance (UMFI), which uses dependence removal techniques from the AI fairness literature as its foundation. We first propose axioms for feature importance methods that seek to explain the causal and associative relationships in data, and we prove that UMFI satisfies these axioms under basic assumptions. We then show on real and simulated data that UMFI performs better than MCI, especially in the presence of correlated interactions and unrelated features, while partially learning the structure of the causal graph and reducing the exponential runtime of MCI to super-linear. 
\end{abstract}

\section{INTRODUCTION}

Scientists often seek to understand the relationships between a set of characteristics and some outcome of interest \citep{kruskal1984concepts}. These relationships are ideally determined by performing carefully controlled experiments so that causality can be established. However, experiments can be difficult and costly to pursue, unethical to perform, or impossible to control \citep{wright1921correlation,vowels2021d}, leaving only observational data available. The relationships that are hidden within vast quantities of observational data are often difficult to determine, so statistical tools, such as feature importance, have been explored. 

Recently, feature importance methods such as Shapley values \citep{shapley1953value,cohen2007feature,lundberg2017unified}, Shapley additive global importance (SAGE) \citep{covert2020understanding}, accumulated local effects (ALE) \citep{apley2020visualizing}, permutation importance (PI) \citep{Breiman2001}, and conditional permutation importance (CPI) \citep{debeer2020conditional}, have been used in high-impact journal papers by scientists who want to explain the mechanisms behind observational data \citep{addor2018ranking,bazaga2020genome,stein2021climate,johnsen2021new,schmidt2020challenges,gill2017capacity,janssen2022application}. However, these methods are predominantly for model explanation or feature selection, so they have many shortcomings when used for other purposes such as scientific inference \citep{freiesleben2022scientific,catav}. ALE can nicely display how changes in inputs lead to altered model predictions but important higher order effects are omitted \citep{molnar2020interpretable}, and although CPI improves upon some limitations of PI, CPI gives zero importance to perfectly correlated features even if they offer significant explanatory power towards the response \citep{covert2020understanding}. Similarly, Shapley values diminish the importance of duplicated or highly correlated features \citep{catav}. Further, only one model is trained in ALE, CPI, and PI. Thus, correlated features, which can alter the model assembly process, could be given artificially low importance if the goal is to explain the data \citep{hooker2021unrestricted}. Due to the multiplicity of near optimal models with vastly different functioning, one cannot use a single model to explain the data generating processes \citep{marx2022but,molnar2021relating}. Instead of exploring a single model, the developers of SAGE, SPVIM, and marginal contribution feature importance (MCI) evaluate the difference in accuracy between a model trained with the feature of interest and a model trained without it, across all feature subsets \citep{catav, covert2020understanding,williamson2020efficient}. However, these methods have not been accepted by a wider scientific audience because of their high computational cost. In particular, we note that MCI is a recently developed method for explaining data, and it was shown in extensive experiments to have better quality and robustness when compared to Shapley values, SAGE, ablation, and bivariate methods \citep{catav}. 

Although MCI is a powerful and innovative method for explaining data, it has three key shortcomings. First, MCI demands an exponential number of model trainings, making it unsuitable even for small-to-medium-sized datasets. Second, although it can handle complex feature interactions and data with correlated features, MCI underestimates the importance of correlated features that form interaction effects because MCI usually ignores features that share information with the feature of interest, as explained further in Section \ref{sec:CorrelatedInteractions}. Third, MCI can give non-zero importance to features that are completely unrelated to the response variable, as experimentally shown in \citet[Figure S3]{catav} and theoretically shown in \citet{harel2022inherent}. We hypothesize that constructing information-preserving representations of the data that are independent of the feature of interest could resolve these three issues. With this in mind, we introduce ultra-marginal feature importance (UMFI), a new variable importance method that can better describe data while drastically reducing runtime.

The rest of this paper is organized as follows. Axioms for explaining data are proposed in Section \ref{sec:axioms}. The framework for UMFI is then formally presented in Section \ref{sec:UMFI} along with its theoretical properties and its simple algorithm. In Section \ref{sec:experiments}, we conduct experiments on simulated and real data to assess the quality, robustness, and time complexity of UMFI compared to MCI. Finally, an overview of the work, its limitations, and ideas for future work are discussed in Section \ref{Conclusion}.

\subsection*{Related Work}

This paper is greatly inspired by the development of marginal contribution feature importance (MCI) by \citet{catav}. Let $F= \{x_1, ..., x_p\}$ be the set of features used to predict the response variable, $y$. The universal predictive power of a set of features $S \subseteq F$ is given by
\begin{equation}
    \nu(S)= \min_{f \in G(\emptyset)} \mathbb{E}[l(f(\emptyset),y)]- \min_{f \in G(S)} \mathbb{E}[l(f(S),y)],
    \label{machine_nu}
\end{equation}
where $l$ is a specified loss function and $G(S)$ is the set of all predictive models restricted to using features in $S \subseteq F$. $\nu$ is closely related to mutual information, with equality under ideal conditions \citep{covert2020understanding}, but in practice, $\nu$ is often approximated by machine learning evaluation functions \citep{covert2020understanding,catav}. Using this, \citet{catav} defined the marginal contribution feature importance (MCI) of a feature $x_i \in F$ as
\begin{equation}
    I_\nu(x_i) = \max\limits_{S \subseteq F} \quad \nu(S \cup \{x_i\}) - \nu(S).
    \label{MCI}
\end{equation}

To achieve our goal of improving upon the shortcomings of MCI, we evaluate the importance of a feature of interest $x_i$ after preprocessing the data to remove dependencies on $x_i$. Finding independent representations of predictors for creating improved feature importance methods is a novel objective, though similar ideas have been alluded to as future work in \citet{konig2021decomposition}, \citet{chen2020true}, and \citet{fan2008sure}. The weaker concept of finding orthogonal representations of data has been discussed previously \citep{gibson1962orthogonal}, though the discussion has been limited to relative importance measures for multiple linear regression \citep{bi2012review,wurm2014residualizing}. Methods which can not only orthogonalize features, but also remove more general dependencies, have seen great progress within the domains of AI fairness and privacy. Some examples of these \textit{preprocessing} techniques include regression \citep{bird2020fairlearn}, optimal transport \citep{johndrow2019algorithm}, neural networks \citep{song2019learning,moyer2018invariant,gitiaux2021learning,gitiaux2021fair}, convex optimization \citep{calmon2017optimized}, kernels \citep{tan2020learning}, and principal inertial components \citep{wang2017estimation}. Linear regression and optimal transport are used in this paper. \citet{gurushankar2022extracting} studies the problem of using fairness techniques to perform partial information decomposition, which leads to interesting connections with our work.

\section{AXIOMS FOR EXPLAINING DATA}
\label{sec:axioms}
\subsection{Background and Notation}

We use $x_i$ to denote an observed feature from the feature set $F$, and $X_i$ to denote the random variable that $x_i$ is sampled from. Suppose that $(F,y)=(x_1, ..., x_p, y)$ is sampled jointly from the joint distribution $( X_{obs},Y)$, where $X_{obs} \subseteq X_{full}= \{X_1, X_2,... X_r \}$, and $p \leq r$. To accommodate dependency removal, we will require that the universal predictive power $\nu$, defined in Equation \eqref{machine_nu}, is also defined for transformations of the feature set. We therefore define the space of \textit{information subsets} of a feature set $F$ as 
\begin{equation}
\mathcal{I}(F)=\{g(F): g \text{ is a deterministic function of }F\}.
\label{info_subset}
\end{equation}
We call these information subsets of $F$ because $I(Y;g(F)) \le I(Y;F)$ holds for any deterministic function $g$ by Theorem \ref{mut_info_bounds}. 

Given the data $(F,y)$ and an evaluation function $\nu$, the feature importance of $x_i \in F$ is denoted by 

\begin{equation}
Imp_{\nu}^{F,y}(x_i) \in \mathbb{R}_{\geq 0}.
\label{imp}
\end{equation}




In contrast to $X_{obs}$, we assume that $X_{full}$ is causally sufficient, meaning that there are no latent confounders in the underlying data generating process \citep{yu2018mining}. Hence, we may consider the full causal graphical model with graph $G=(V,E)$, such that each vertex from $V=\{1, ..., r+1\}$ is associated to a random variable from $X_{full} \cup Y:= \{V_1, ..., V_r, V_{r+1}\}$, and the directed edge set $E$ enables the graphical model to obey the global Markov property \citep{kang2009markov}. The graphical model can often be given by a structural causal model $C=(G, U,\mathcal{E})$, where $G$ is as defined above, $\mathcal{E}$ is a set of mutually independent noise variables $\epsilon_1, ..., \epsilon_{r+1}$, and $U=\{u_1, ..., u_{r+1}\}$ is the set of functions relating the variables $V_1, ..., V_{r+1}$ to their parents in $G$ via the relation $V_i=u_i(V_{pa(i)}, \epsilon_i)$. 

\subsection{Axioms}
Any attempt to build a method that explains the data should begin by rigorously defining what explaining the data truly means. Different, but closely related definitions and goals have been formulated by \citet{chen2020true}, \citet{catav}, \citet{gromping2009variable}, \citet{benard2022shaff}, and \citet{tolocsi2011classification}. Inspired by the above works, we provide three intuitive axioms for feature importance methods intended for explaining data and scientific inference.



\begin{enumerate}
    \item \textbf{Elimination axiom:} Eliminating a feature $x_j$ from $F$ can only decrease the importance of other features:  $$\forall x_i \in F \setminus \{ x_j\}, Imp_{\nu}^{F \setminus \{x_j\},y}(x_i) \leq Imp_{\nu}^{F,y}(x_i).$$
    \item \textbf{Invariance under Redundant Information and Symmetry under Duplication (IRI \& SD) axiom:} Let $\hat{x}$ be a redundant feature, i.e. $\hat{x} \in \mathcal{I}(F)$. Then adding $\hat{x}$ to the feature set $F$ to create $F'= F \cup \hat{x}$ will not change the importance given to any preexisting feature in $F$. Also, in the case where the added feature $\hat{x}$ is a duplicate of a preexisting feature $x_j \in F$ (i.e., $\hat{x} = h(x_j)$ and $h$ is bijective) both $\hat{x}$ and $x_j$ will be given equal importance: 
    $$\hat{x} \in \mathcal{I}(F) \implies Imp_{\nu}^{F,y}(x_i)=Imp_{\nu}^{F',y}(x_i) \textrm{ } \forall x_i \in F$$
    $$\hat{x} = h(x_j) \implies Imp_{\nu}^{F',y}(\hat{x})=Imp_{\nu}^{F',y}(x_j).$$
    \item \textbf{Blood relation axiom:}
    A feature $x_i$ will be given non-zero and positive importance if and only if it is blood related to the response $Y$ in the full causal graphical model $G=(V,E)$. Two vertices are said to be blood related if there is a directed path between them or if there is a backdoor path between them via a common ancestor. Figure \ref{fig:causalGraph_bloodrel} illustrates an example.
    $$Imp_{\nu}^{F,y}(x_i)>0 \iff X_i \in BR_G(Y).$$
\end{enumerate}

The \textit{elimination axiom} comes directly from \citet{catav}. Once a feature is observed to be related to the response, the relationship strength between the feature and response should not drop, regardless of the additional features added. In fact, the importance should often increase, since adding features could reveal further synergistic information about the response \citep{griffith2014quantifying,williams2010nonnegative}. Hence, the elimination axiom suggests that a feature importance method intended for scientific inference should be expressive enough to capture complex feature interactions.

The \textit{invariance under redundant information (IRI)} portion of our second axiom is a generalization of the duplication invariance property introduced in \citet{catav} and the group size invariance property introduced in \citet{tolocsi2011classification}. If a dataset contains two duplicate features, a model may use them equally often and therefore divide the importance equally between them (random forests), or only one of the features may be given importance (lasso) \citep{chen2020true}. However, from the data's perspective, the original importance found before duplication should be maintained \citep{tolocsi2011classification,catav}. Further, after adding duplicate features, no additional interaction capability is available \citep{griffith2014quantifying}, so the importance of all other features should remain the same. The same arguments can be made for features offering only redundant information, which arise much more frequently than exact duplicates, thus motivating IRI. The \textit{symmetry under duplication (SD)} aspect of our second axiom is functionally equivalent to the symmetry property presented in \citet{catav} and is similar to the correlated group property in \citet{tolocsi2011classification}. The purpose of this part of the axiom, as pointed out by \citet{tolocsi2011classification}, is to prevent the feature importance method from dismissing features that are not needed for better predictions, but which may be part of the true underlying causal model.

We call the final axiom the \textit{blood relation axiom} because it asserts that a feature should have non-zero importance with respect to the response $y$ if and only if it is connected to $y$ in the causal graph through a directed path or via a common ancestor, much like how blood relatives are connected within a family tree. This axiom suggests that feature importance scores intended for data explanation should extract reliable knowledge about the underlying causal graph and data generating process. When two features are blood related, or equivalently, when there is an open path between them, the two features are said to be associated (see \citet{greenland1999causal} and \citet{williams2018directed}). Thus, a feature importance metric satisfying this axiom would give non-zero importance to a feature if and only if there is a statistical association between that feature and the response. Statistical association is not only a quality of interest for many applications (e.g., genome-wide association studies), it also forms the foundation of Pearl's causal hierarchy \citep{shpitser2008complete}. Some may argue that we should only assign importance to features with a directed path to the response, but as pointed out by \citet{gromping2009variable}, even if this were desired, this is an impossible goal since $X \rightarrow Z \rightarrow Y$ is Markov equivalent to $X \leftarrow Z \rightarrow Y$. It is argued by \citet{gromping2009variable} that any importance measure aimed at explanatory or causal purposes must give importance to $X$ and $Z$ in both scenarios. 

Additionally, a feature importance metric satisfying this axiom can partition the feature set into features that are blood related to the response and features that are not blood related to the response. Although it does not enable us to immediately recover the full causal graph, this partitioning may be a helpful supplemental tool for other fast partial causal graph discovery methods \citep{soleymani2022causal}. This would be an especially relevant research direction since \citet{reisach2021beware} recently showed that classic causal discovery algorithms are more inaccurate than previously thought. This partitioning could also be useful for sure independence screening \citep{fan2008sure,schellhas2020distance}.

The axioms introduced in this section are only intended to define ideal properties for feature importance methods that are focused on scientific inference. We do not claim that a unique solution exists since all three axioms are invariant to scalar multiplication, nor do we claim that a method exists that satisfies all axioms in all scenarios. 

\section{ULTRA-MARGINAL FEATURE IMPORTANCE}
\label{sec:UMFI}

Let $F= \{x_1, ..., x_p\}$ be a set of $p$ features of arbitrary type used to predict the response $Y$. 



\begin{definition}
\label{def:remove_dep}
We denote $S^F_{x_i}$ as a preprocessed feature set after dependencies on the feature of interest $x_i$ have been removed from $F$. An optimally preprocessed feature set is denoted by $\hat{S}^F_{x_i}$, and we say that a preprocessing $S^F_{x_i}$ is optimal if it obeys the following properties:
\begin{enumerate}
\label{list }
    \item $S^F_{x_i} \in \mathcal{I}(F)$
    \item $S^F_{x_i} \indep X_i$
    \item $I(Y;S^F_{x_i},X_i)=I(Y;F)$
\end{enumerate}
\label{optimal_preprocessing}
\end{definition}

The first property ensures that $S^F_{x_i}=g(F)$ for some deterministic function $g$, and hence no information from outside of $F$ is gained during the transformation. The second property upholds that $S^F_{x_i}$ is independent of $X_i$. The last property affirms that there is no unnecessary information loss or distortion incurred during preprocessing \citep{gitiaux2022sofair}. For simplicity, we use $F$ and $\hat{S}^F_{x_i}$ to denote both their respective data and random vector instantiations.

Provided that it exists, an optimal preprocessing $\hat{S}^F_{x_i}$ is not unique since scaling by a constant does not affect any of the optimality conditions. Therefore, given a feature of interest $x_i$ and a feature set $F$, we may consider the (possibly empty) equivalence class $[\hat{S}^F_{x_i}]$ of optimal preprocessings. In practice, the last two properties of Definition \ref{optimal_preprocessing} can be difficult to guarantee, but we later observe in Section \ref{sec:experiments} that non-optimal preprocessings can often be good enough to provide highly accurate feature importance scores for scientific inference.

\begin{definition}
\label{def:umfi}
Given an evaluation function $\nu: \mathcal{I} (F) \to \mathbb{R}_{\geq 0}$, a feature set $F$, and a response $y$, we define the ultra-marginal feature importance (UMFI) of the feature of interest $x_i \in F$ as 
\begin{equation}
    U^{F,y}_\nu (x_i) = \nu (S^F_{x_i} \cup \{x_i\}) - \nu(S^F_{x_i}).
    \label{umfi}
\end{equation}
\end{definition}
We note that this definition allows for the use of any preprocessing $S^F_{x_i}$. When the utilized preprocessing $\hat{S}^F_{x_i}$ is optimal, we say that UMFI is optimally computed.
\begin{theorem}
\label{theorem:umfi_master}
Suppose that the data $(F,y)$ comes from a multivariate Gaussian distribution and that $\nu(\cdot)$ is positively linearly related to $I(Y;\cdot)$, then we can ensure that  $U^{F,y}_{\nu}$ satisfies (i) the elimination axiom and (ii) the redundant information invariance and duplication symmetry axiom, by using linear regression to optimally compute $U^{F,y}_{\nu}$. If we additionally assume that the joint distribution of the variables $(X_{full},Y)$ is faithful to the causal graph $G$, then $U^{F,y}_{\nu}$ also satisfies (iii) the blood relation axiom.
\end{theorem}
\begin{proof} Since $\nu(\cdot) = aI(Y;\cdot)+b$ for some $a \ge 0$ and $b \in \mathbb{R}$, $U^{F,y}_{\nu} = \nu (S^F_{x_i} \cup \{x_i\}) - \nu(S^F_{x_i})$, and scaling by a constant factor does not alter any axiom, without loss of generality, it suffices to prove the axioms for the case where $\nu(\cdot) = I(Y;\cdot)$. \\

(i) Elimination: $U^{F \setminus \{x_j\},y}_{\nu}(x_i)\le U^{F,y}_{\nu}(x_i)$. 

Let $x_j$ be the eliminated feature, and let $x_i \in F \setminus \{x_j\}$ be the feature of interest. Theorem \ref{existence_optimal_preprocessing} in the Supplement shows that pairwise linear regression attains an optimal preprocessing $\hat{S}^F_{x_i}$ since $(F,y)$ is multivariate Gaussian. After mean centering, pairwise linear regression transforms each feature $x_k \in F \setminus \{x_j\}$ into $\tilde{x}_k=x_k-\beta_k x_i$, where $\beta_k$ is the regression coefficient of $x_k$ on $x_i$. Hence, $\hat{S}^F_{x_i} = (\hat{S}^{F \setminus \{x_j\}}_{x_i}, \tilde{x}_j)$ and
\begin{align*}
    &U^{F,y}_{\nu}(x_i)=I(Y; \hat{S}^{F}_{x_i}, X_i) - I(Y; \hat{S}^{F}_{x_i})\\
    &= I(Y; \hat{S}^{F \setminus \{x_j\}}_{x_i}, \tilde{X}_j, X_i) - I(Y; \hat{S}^{F \setminus \{x_j\}}_{x_i}, \tilde{X}_j)\\
     & \ge I(Y; \hat{S}^{F \setminus \{x_j\}}_{x_i}, X_i) - I(Y; \hat{S}^{F \setminus \{x_j\}}_{x_i})\\
     &= U^{F \setminus \{x_j\},y}_{\nu}(x_i),
\end{align*}
where the inequality is given by the supermodularity of mutual information under independence (Theorem \ref{supermodularity_theorem}), since $X_i \indep (\hat{S}^{F \setminus \{x_j\}}_{x_i}, \tilde{X}_j)$ follows from the optimality of $\hat{S}^F_{x_i}$.

(ii) IRI: Let $\hat{x} \in \mathcal{I}(F)$, then $U^{F,y}_{\nu}(x_i)=U^{F \cup \{\hat{x}\},y}_{\nu}(x_i)$.

By the optimality of the preprocessings and the invariance of mutual information under redundant information,
\begin{align*}
    U^{F,y}_{\nu}(x_i)&=I(Y; \hat{S}^F_{x_i},X_i)-I(Y; \hat{S}^F_{x_i})\\ 
    &=I(Y;F) - I(Y;\hat{S}^F_{x_i})\\
    U^{F \cup \{\hat{x}\},Y}_{\nu}(x_i)&=I(Y; \hat{S}^{F \cup \{\hat{x}\}}_{x_i},X_i)-I(Y; \hat{S}^{F \cup \{\hat{x}\}}_{x_i})\\
    &= I(Y; F) - I(Y; \hat{S}^{F \cup \{\hat{x}\}}_{x_i})
\end{align*}

From Lemma \ref{lemma:equivalence_class}, $[\hat{S}^F_{x_i}] \equiv [\hat{S}^{F \cup \{\hat{x}\}}_{x_i}]$. In other words, every optimal preprocessing $\hat{S}^F_{x_i} \in [\hat{S}^F_{x_i}]$ also exists in $[\hat{S}^{F \cup \{\hat{x}\}}_{x_i}]$ and vice versa. Thus, we may select the same optimal preprocessing from each equivalence class to ensure that $I(Y; \hat{S}^{F}_{x_i}) = I(Y; \hat{S}^{F \cup \{\hat{x}\}}_{x_i})$, which concludes the proof.

SD: Let $\hat{x} = h(x_j)$ for some bijective function $h$. Then, $U^{F \cup \{ \hat{x} \},y}_{\nu}(\hat{x})=U^{F \cup \{\hat{x}\},y}_{\nu}(x_j)$.

We have $U^{F \cup \{ \hat{x} \},y}_{\nu}(\hat{x})=
I(Y;F) - I(Y;\hat{S}^{F \cup \{\hat{x}\}}_{\hat{x} })$ and $U^{F \cup \{\hat{x}\},y}_{\nu}(x_j) = I(Y;F) - I(Y;\hat{S}^{F \cup \{\hat{x}\}}_{x_j})$. Similarly to the proof of IRI, it suffices to show that $[\hat{S}^{F \cup \{\hat{x}\}}_{x_j}] \equiv [\hat{S}^{F \cup \{\hat{x}\}}_{\hat{x}}]$, which is proven in Lemma \ref{lemma:equivalence_class} in the Supplement. 

(iii) Blood relation: $U^{F,y}_\nu(x_i)>0 \iff x_i \in BR_G(Y)$.

We equivalently prove $U^{F,y}_\nu(x_i)=0 \iff X_i \not \in BR_G(Y)$. First, we write $U^{F,y}_{\nu}(x_i)=I(Y;S^F_{x_i},X_i)-I(Y;S^F_{x_i})=I(Y;x_i|S^F_{x_i})$. Also, from the definition of conditional mutual information, $$U^{F,y}_{\nu}(x_i)=0 \iff I(Y;X_i|S^F_{x_i})=0 \iff X_i \indep Y | S^F_{x_i}.$$ Next, we claim that $U^{F,y}_{\nu}(x_i)=0 \iff x_i \indep Y$. A conditional independence axiom that holds for the variables in a structural causal model is the contraction axiom \citep{dawid1979conditional}: $X \indep Y \text{ and } X \indep W | Y \iff X \indep (W,Y)$. From this, we obtain the statement
\begin{align*}
    X_i \indep S^F_{x_i} \text{ and } X_i \indep Y | S^F_{x_i} \iff X_i \indep (Y,S^F_{x_i})
\end{align*}
Since $X_i \indep S^F_{x_i}$ is true from the optimality of the preprocessing, this property can be reduced to 
\begin{equation}
    X_i \indep Y | S^F_{x_i} \iff X_i \indep (Y,S^F_{x_i}) \iff X_i \indep Y
    \label{BR_gaussian_eq}
\end{equation}
where the last equivalence is due to $X_i \indep S^F_{x_i}$ and the fact that $(\hat{S}^F_{x_i}, X_i, Y)$ is multivariate Gaussian \citep{steudel2015information,lauritzen2018unifying}.

All that is left to prove is $X_i \indep Y \iff X_i \not \in BR_G(Y)$. The claim $X_i \not \in BR_G(Y) \implies X_i \indep Y$ follows from the global Markov property. Indeed, it is easy to check that $X_i$ and $Y$ are d-separated by the empty set in this case. The converse holds because we have assumed that the graphical model obeys faithfulness. Indeed, $X_i$ and $Y$ would be d-connected by the empty set if $X_i \in BR_G(Y)$. 
\end{proof}

We note that the linear relation between $\nu(\cdot)$ and $I(Y;\cdot)$ is pivotal to the proof of Theorem \ref{theorem:umfi_master}. Under ideal conditions, this relationship holds \citep{covert2020understanding}, but in practice, the accuracy of the approximation depends on the quality of the method, the loss function, and the response variable's distribution. See \citet{covert2021explaining} and Supplement \ref{MI_ML_eval} for a more thorough overview. Simulations have demonstrated that Gaussian graphical models with sparse graphs are generally faithful to the graph; hence, assuming faithfulness is often reasonable \citep{malouche2008estimating}. 

We remark that UMFI does not require Gaussianity to satisfy all axioms. Indeed, since the last equivalence in Equation \eqref{BR_gaussian_eq} can be relaxed to be an implication without Gaussianity, UMFI satisfies one direction of the blood relation axiom, $U_\nu^{F,y}(x_i)=0 \implies x_i \not \in BR_G(Y)$, for arbitrary distributions. Additionally, we show that UMFI satisfies the blood relation axiom under alternate assumptions in Supplement \ref{sec:UMFI_appendix}. We also note that the proof of Theorem \ref{theorem:umfi_master}(ii) does not require the Gaussian assumption, making the IRI \& SD axiom true for arbitrary distributions.

\begin{algorithm}[h!]
\caption{Algorithm for computing UMFI}\label{algo:umfi}
\begin{algorithmic}[1] 
\STATE Let $y$ be the response variable of the set of predictors $F$. Choose a feature $x_i \in F$.\\
\STATE Obtain $S^{F}_{x_i}$ by using a technique that removes dependencies on $x_i$ from $F$.\\
\STATE Specify a method $f$ and a corresponding evaluation function $\nu_f$.\\
\STATE Estimate the predictive power, $\nu_f(S^{F}_{x_i})$, that $S^{F}_{x_i}$ has about $y$. \\
\STATE Estimate the predictive power, $\nu_f(S^{F}_{x_i} \cup \{x_i\})$, that $S^{F}_{x_i} \cup \{x_i\}$ has about $y$. \\
\RETURN $U^{F,y}_{\nu_f}(x_i)=\nu_f(S^{F}_{x_i} \cup \{x_i\})- \nu_f(S^{F}_{x_i})$
\end{algorithmic}
\end{algorithm}

Since UMFI is model-agnostic, we provide a general algorithm, which can be applied using any pair of preprocessing functions and evaluation functions $\nu_f$ (Algorithm \ref{algo:umfi}). We note that $\nu_f$ is not restricted to the domain of machine learning models or even models in general. For example, one could also implement UMFI with measures of dependence such as QMD \citep{griessenberger2022multivariate} or FOCI \citep{azadkia2021simple}. 

\section{EXPERIMENTS}
\label{sec:experiments}

We perform experiments to compare UMFI and MCI with respect to quality, robustness, and time complexity. To implement UMFI, we consider optimal transport \citep{johndrow2019algorithm} (UMFI\_OT) and linear regression \citep{bird2020fairlearn} (UMFI\_LR) as methods to remove dependencies from the data. A detailed overview of these implementations is shown in Supplement \ref{sec:RemovingDependencies} and experiments comparing their performance appear in Supplement \ref{sec:LR_vs_OT}. For all experiments, we use random forests' out-of-bag accuracy (cross-validation $R^2$ for regression and cross-validation overall accuracy for classification) as the evaluation metric $\nu_f$ since it can capture nonlinearities and interaction effects \citep{Breiman2001,taufiq2023manifold,wright2016little}. We use the \emph{ranger} package to implement random forests with $100$ trees and default hyperparameters \citep{wright2015ranger}. All experiments were run in Microsoft R Open Version 4.0.2. Supplement \ref{sec:FurtherEXP_feature_imp} contains additional experiments comparing UMFI with other feature importance metrics including ablation, permutation importance, and conditional permutation importance. In the same section, we rerun the experiments comparing MCI and UMFI using extremely randomized trees instead of random forests and do an additional comparison on a real dataset from hydrology \citep{addor2017camels}. Code for all experiments in R, inlcuding basic MCI and UMFI functions, can be found at \url{https://github.com/HydroML/UMFI}. The analagous Python package can be found at \url{https://github.com/guanton/umfi_python}.

\subsection{Experiments on Simulated Data}
\label{sec:SimulatedData}
In the following subsections, we compare UMFI and MCI on simulated data. The data in all scenarios contains one response variable $y$, four explanatory features $x_1, x_2, x_3, x_4$, and $1000$ randomly generated observations. Each study is repeated $100$ times to test and ensure stability \citep{yu2013stability}.

\subsubsection{Nonlinear Interactions}
\label{sec:Interactions}

Interaction effects are common in many scientific disciplines where assessing feature importance is prevalent, including hydrology \citep{janssen2021hydrologic,addor2018ranking,le2022snow,li2022statistical}, genomics \citep{catav,wang2021genome,orlenko2021comparison,wright2016little}, and glaciology \citep{edwards2021projected,bach2018sensitive,brenning2010statistical,sevestre2015climatic}. So, as was done in \citet{catav}, we assess the ability of MCI and UMFI to detect nonlinear interaction effects in the data \citep{marx2021weaker}. We consider:
\begin{align*}
    &x_1,x_2,x_3,x_4 \sim \mathcal{N}(0,1)\\
    &y=x_1+x_2+sign(x_1*x_2)+x_3+x_4.
\end{align*}

Feature importance metrics should ideally conclude that $x_1$ and $x_2$ have higher importance compared to $x_3$ and $x_4$ because of the extra interaction term, $sign(x_1*x_2)$. Figure \ref{fig:int} shows consistently good performance across all methods. Each method gave high relative importance scores to $x_1$ and $x_2$, while $x_3$ and $x_4$ received less, but still substantial importance. All methods show similar variability.

\subsubsection{Correlated Interactions}
\label{sec:CorrelatedInteractions}
Interacting features are often correlated \citep{jakulin2003quantifying,janssen2021hydrologic}. So, this simulation study aims to repeat the nonlinear interactions study, except now $x_1$ and $x_2$ are highly correlated with eachother. In the same way, $x_3$ and $x_4$ are highly correlated with eachother. Let $A,B,C,D,E,G \sim \mathcal{N}(0,1)$. We consider:
\begin{align*}
    &x_1=A+B, \text{ } x_2=B+C, \text{ }x_3=D+E,\text{ } x_4=E+G \\
    &y=x_1+x_2+sign(x_1*x_2)+x_3+x_4.
\end{align*}

Just as with the interaction experiment with independent features, we would expect $x_1$ and $x_2$ to be more important than $x_3$ and $x_4$ because of the extra interaction term, $sign(x_1*x_2)$. The results in Figure \ref{fig:corInt} show that UMFI provides substantially better feature importance scores compared to MCI when correlated interactions are present. MCI estimates that all features have approximately the same feature importance scores, while both UMFI methods appropriately give significantly greater importance to $x_1$ and $x_2$ compared with $x_3$ and $x_4$. MCI fails in this experiment because it penalizes feature subsets that share information with the feature of interest (Equation \eqref{MCI}). For example, if we are assessing the MCI score for $x_1$, since $x_2$ is strongly correlated with $x_1$, then the predictive power offered by $x_1$ on top of a subset $S$ would be diminished by the presence of $x_2 \in S$. Therefore, $x_2$ is not utilized in the MCI score for $x_1$, which prevents the detection of the interaction term $sign(x_1*x_2)$. UMFI is able to detect this interaction because it can extract the information from $x_2$ that interacts with $x_1$ while keeping this extracted feature independent of $x_1$. We suspect that similar results could be demonstrated in the presence of dependent, but uncorrelated interactions.

\begin{figure}
\begin{subfigure}{0.49592\linewidth}
\centering
  \includegraphics[width=\linewidth]{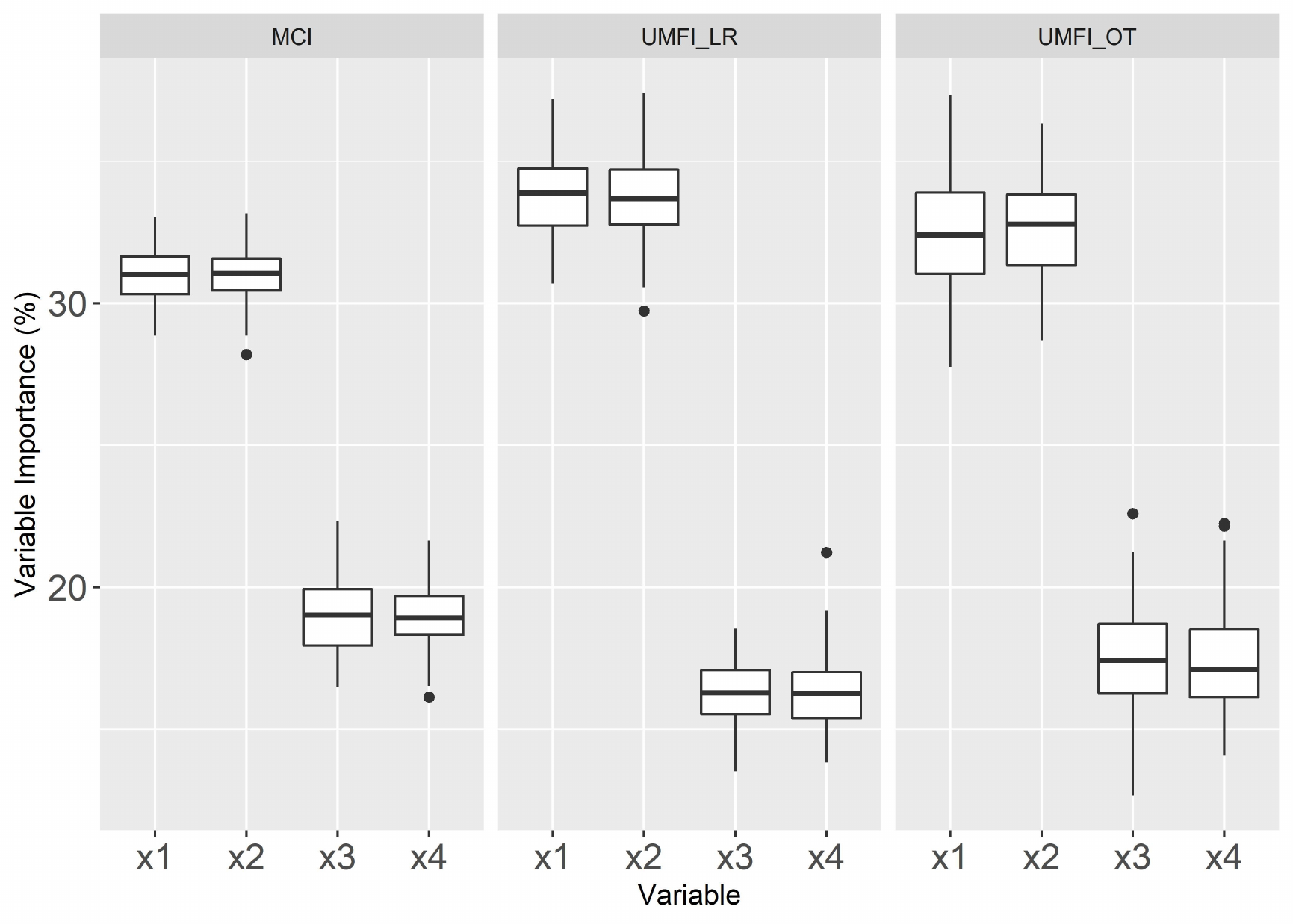}
  \caption{Nonlinear interactions}
  \label{fig:int}
\end{subfigure}
\begin{subfigure}{.49592\linewidth}
  \centering
  \includegraphics[width=\linewidth]{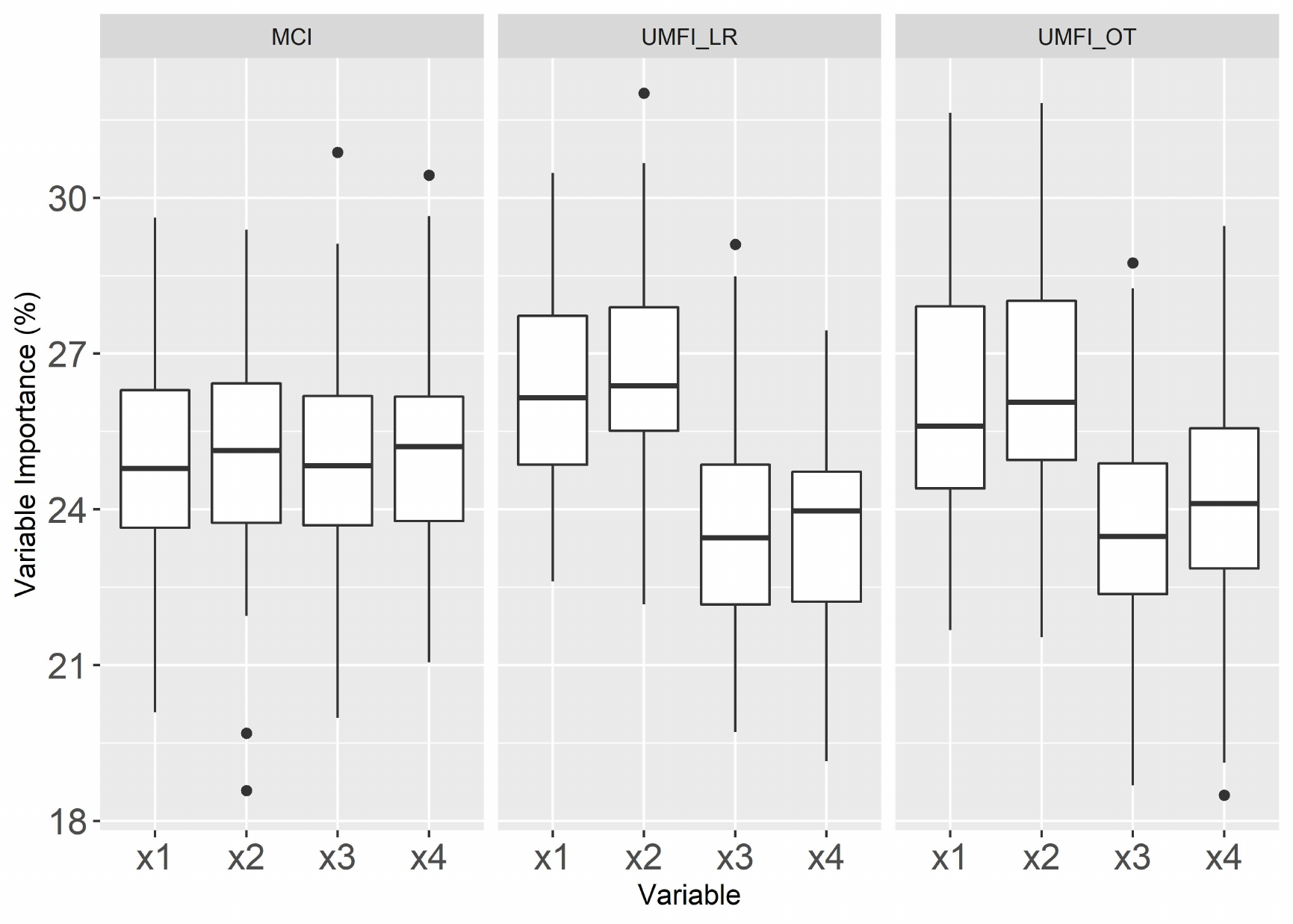}
  \caption{Correlated interactions}
  \label{fig:corInt}
\end{subfigure}\\[1ex]
\begin{subfigure}{.49592\linewidth}
  \centering
  \includegraphics[width=\linewidth]{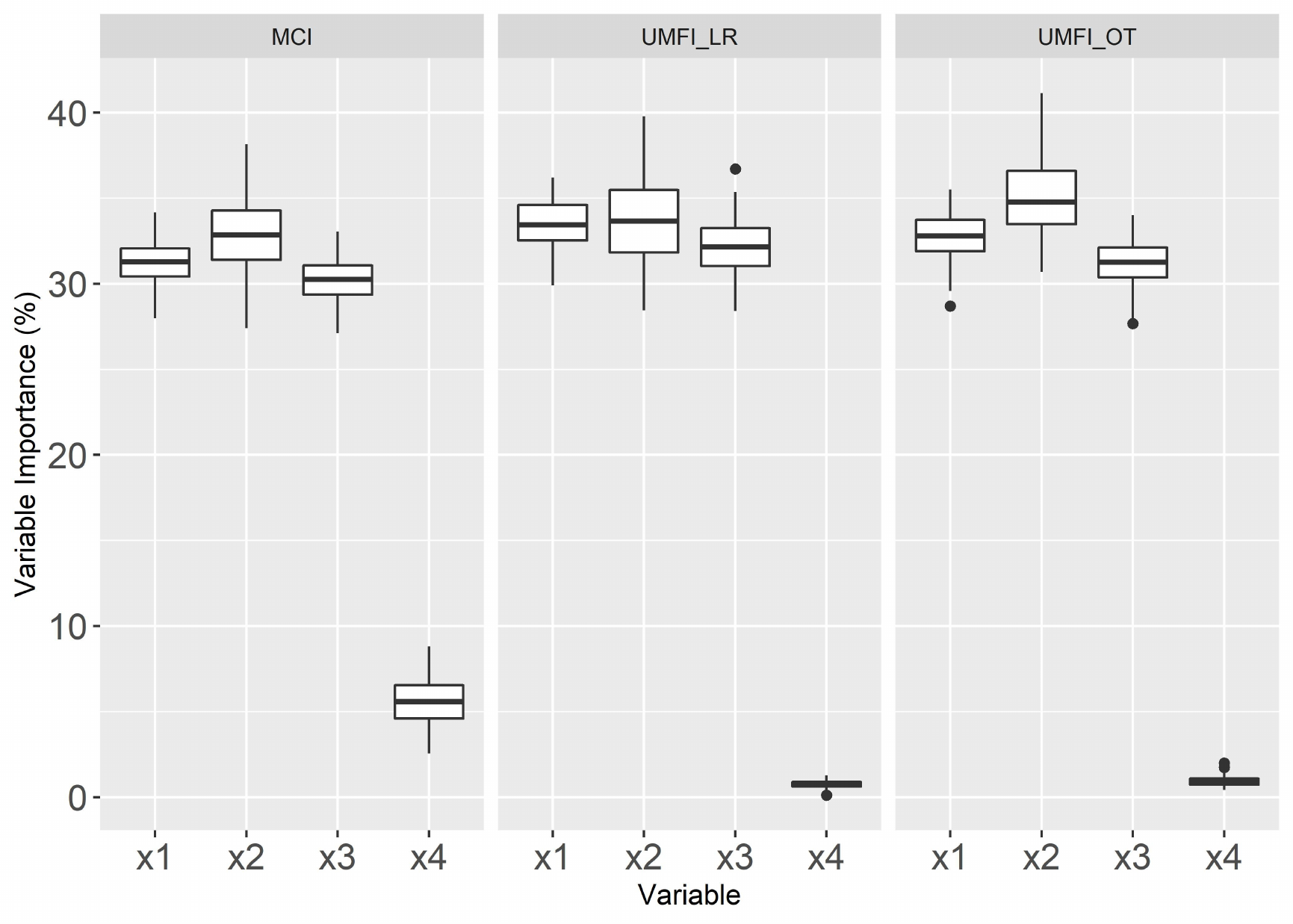}
  \caption{Correlation}
  \label{fig:cor}
\end{subfigure}
\begin{subfigure}{0.49592\linewidth}
\centering
  \includegraphics[width=\linewidth]{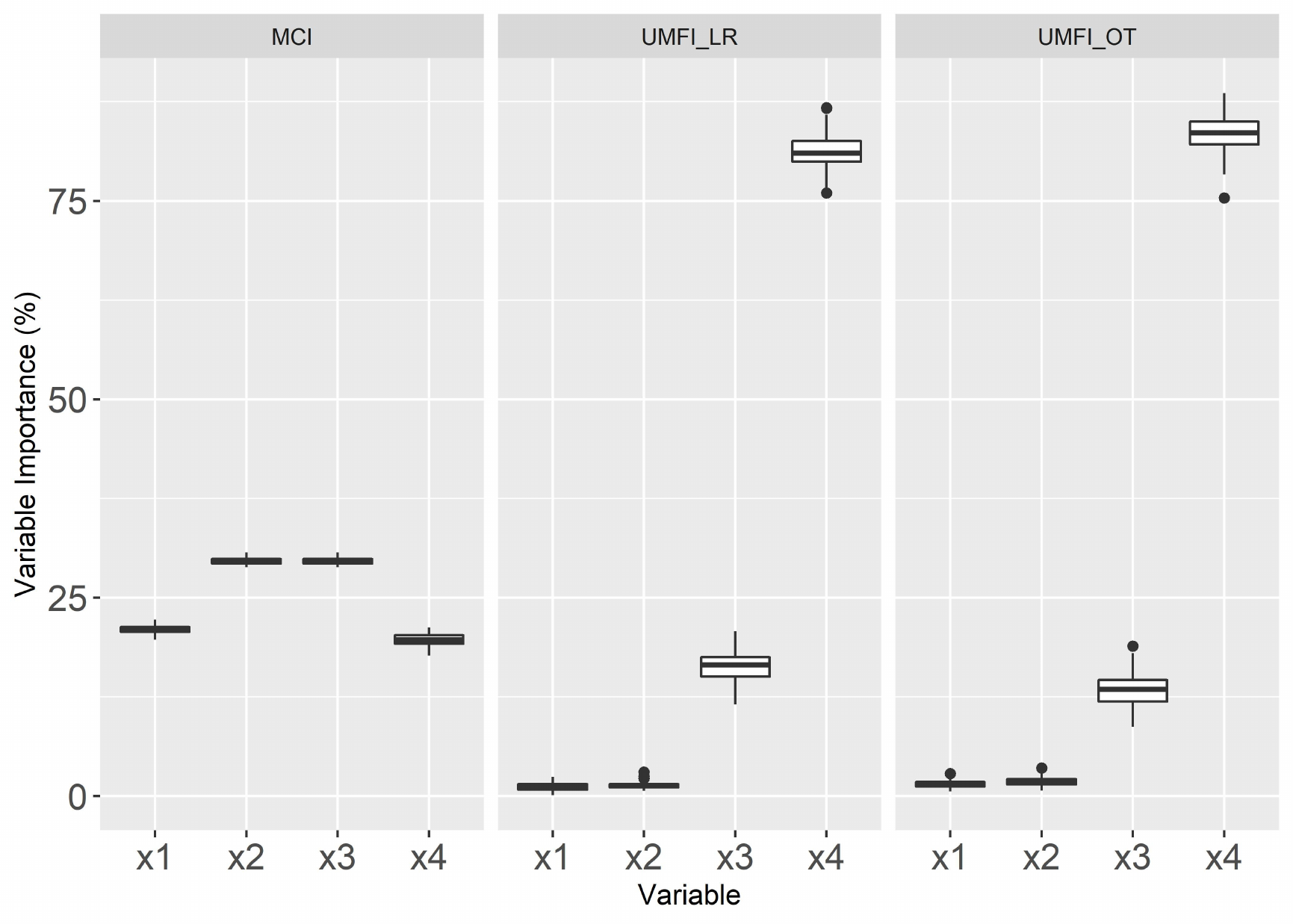}
  \caption{Blood relation}
  \label{fig:Blood}
\end{subfigure}
\newline
\caption{Results for the experiments on simulated data from Subsection \ref{sec:SimulatedData}. Feature importance scores are shown as a percentage of the total for each of $x_1$ to $x_4$ from $100$ replications. Results are shown for marginal contribution feature importance (MCI), ultra-marginal feature importance with linear regression (UMFI\_LR), and ultra-marginal feature importance with pairwise optimal transport (UMFI\_OT).}
\label{fig:test}
\end{figure}

\subsubsection{Correlation}
\label{sec:Correlations}

Feature importance methods that seek to explain data should not change the measured importance of features in the presence of redundant variables according to the IRI axiom. To test this, we implement a simulation study similar to the ones found in \citet{catav}. Let $\epsilon \sim \mathcal{N}(0,0.01)$. We consider:
\begin{align*}
    & x_1,x_2,x_4 \sim \mathcal{N}(0,1), \text{ } x_3=x_1+\epsilon \\
    & y=x_1+x_2.
\end{align*}
The addition of $x_3$, which is approximately a duplicate of $x_1$, should not alter the importance of $x_1$, which should remain equally as important as $x_2$ since they have the same influence on the response $y$. The results shown in Figure \ref{fig:cor} show that both MCI and UMFI work reasonably well. As with the previous simulation experiment, the variability is consistent across methods. As was desired, UMFI with linear regression shows approximately equal relative importance scores for $x_1$ and $x_2$. The importance given to $x_2$ was slightly greater than $x_1$ according to MCI and UMFI with optimal transport. Interestingly, MCI assigns some importance to $x_4$, which was independent of the response, while both UMFI methods assign importance scores close to zero. Because of this, we conclude that UMFI with linear regression performs the best in this simulated scenario.

\subsubsection{Blood Relation}
\label{sec:Blood_relation_exp}

To ensure that UMFI is suitable for scientific inference, and that it could be used to learn part of the structure of causal graphs in theory as well as in practice, we implement the blood relation simulation experiment. In this study, data is generated from the causal graph in Figure \ref{fig:causalGraph_bloodrel}, which was inspired by the collider causal graph found in \citet{harel2022inherent}. The feature $S$ is unobserved, thus $x_3$ and $x_4$ are the only observed features that are blood related to the response $y$. According to the blood relation axiom, $x_3$ and $x_4$ should be given high and positive importance while $x_1$ and $x_2$ should receive zero importance. In Section \ref{sec:UMFI} and Supplement \ref{sec:UMFI_appendix}, we prove that in ideal scenarios, UMFI will satisfy the blood relation axiom. We hypothesize that we can extend this to real-world scenarios where non-Gaussian features and interaction information appear. To test this, we consider:
\begin{align*}
    & \delta \sim \mathcal{U}(-1,1),\ \gamma \sim Exp(1), \ \epsilon \sim \mathcal{U}(-0.5,0.5) \\
    & x_1,S \sim \mathcal{N}(0,1), \ x_2 = 3x_1 + \delta,\ x_3 = x_2+S \\
    & y = S+ \epsilon, \ x_4 = y + \gamma.
\end{align*}

\begin{figure}[h!]
\begin{center}
\begin{tikzpicture}[node distance={15mm}, thick, main/.style = {draw, circle}] 
\node[circle, draw=black, fill=red!40] (1) {$S$}; 
\node[main] (2) [left of=1] {$X_1$}; 
\node[main] (3) [below of=2] {$X_2$}; 
\node[circle, draw=black, fill=red!40] (4) [below of=3] {$X_3$}; 
\node[main, draw=black, fill=blue!40] (5) [below right of=1] {$Y$}; 
\node[circle, draw=black, fill=red!40] (6) [below of=5] {$X_4$}; 
\draw[->] (1) -- (4); 
\draw[->] (1) -- (5); 
\draw[->] (2) -- (3); 
\draw[->] (3) -- (4); 
\draw[->] (5) -- (6); 
\end{tikzpicture} 
\end{center}
\caption{The full causal graph generating the data for the blood relation simulation experiment in Section \ref{sec:Blood_relation_exp}. Blood related vertices to the response $Y$ (blue) are coloured in red. $S$ and $X_4$ are directly causally related to $Y$, whereas $X_3$ is related to $Y$ via the common ancestor $S$.}
\label{fig:causalGraph_bloodrel}
\end{figure}
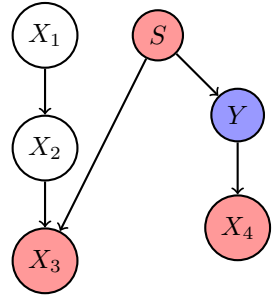

The results shown in Figure \ref{fig:Blood} indicate that MCI fails to distinguish the blood related features since more than half of the total importance is shared between $x_1, x_2 \not \in BR(Y)$, while $x_3,x_4 \in BR(Y)$ together received less than half of the total importance. In contrast, UMFI\_LR and UMFI\_OT detect that $x_1$ and $x_2$ should have approximately zero importance while giving most of the importance to $x_4$ and the rest of the relative importance to $x_3$.

\subsection{BRCA Experiments}
\label{sec:BRCA}
We use the same breast cancer (BRCA) classification dataset \citep{tomczak2015cancer} used in previous feature importance studies including \citet{catav} and \citet{covert2020understanding} to test the quality and robustness of UMFI on real data. The original data contains over $17,000$ genes and $571$ anonymous patients with one of four types of breast cancer. We consider the same subset of $50$ genes as in \citet{catav} and \citet{covert2020understanding} for easier computation and result visualization. Of the $50$ selected genes, $10$ are known to be associated with breast cancer, while the other $40$ genes are randomly sampled. This data was downloaded from \href{https://github.com/TAU-MLwell/Marginal-Contribution-Feature-Importance/tree/main/BRCA_dataset}{the MCI GitHub page}. In \citet{catav} and \citet{covert2020understanding}, these $40$ randomly sampled genes are assumed to be unassociated with breast cancer. However, to ensure a more definitive ground truth, we also randomly permute the values of these $40$ genes across their respective $571$ observations to further reduce the chance that these genes have any association with breast cancer. 

Quality is then measured with the true positive and true negative rates: the $10$ BRCA associated genes should have some non-zero importance (positive), and the other $40$ genes should have exactly zero importance (negative). These experiments were run $200$ times on different seeds and with a different random sample of $500$ patients for each iteration. Robustness is measured using the standardized interquartile range (SIQR) from the repeated experiments, which is calculated by dividing the average IQR across the $50$ features by the average median. This experiment is too computationally intensive for MCI to be calculated exactly, so we implement MCI assuming soft 2-size submodularity (see the supplement of \citet{catav} for details).

\begin{figure}[h!]
  \centering
  \includegraphics[width=0.48\textwidth]{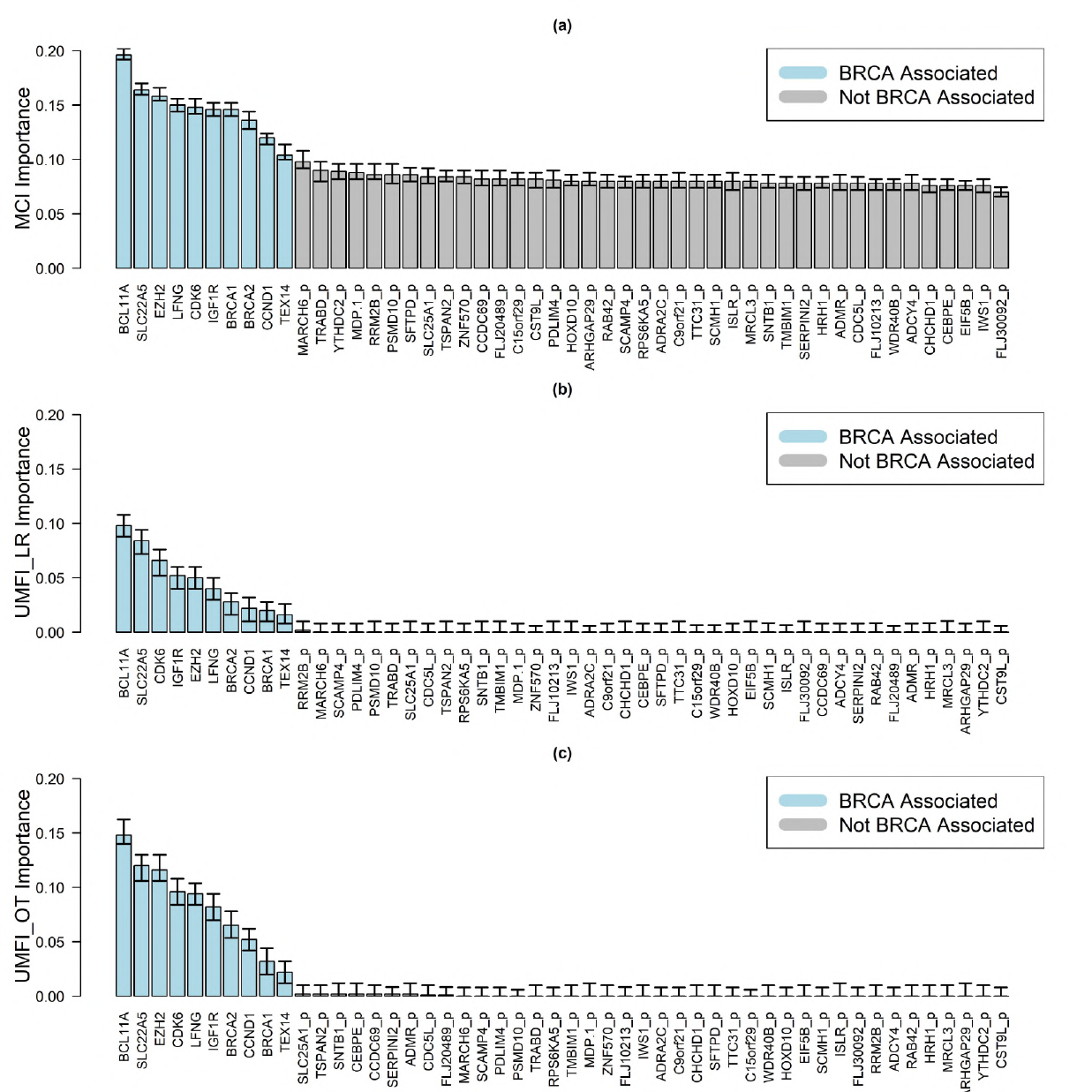}
  \caption{Median feature importance scores provided by (a) MCI, (b) UMFI with linear regression, and (c) UMFI with pairwise optimal transport, for each gene in the BRCA dataset after 200 iterations. Genes colored in blue are known to be associated with breast cancer while genes colored in grey are random permutations of randomly selected genes, which we assume to be unassociated with breast cancer. The first and third quantiles of the scores are visualized as error bars for each gene.}
  \label{fig:BRCA}
\end{figure}

We found that MCI and UMFI (UMFI\_LR and UMFI\_OT) correctly gave significant importance to the $10$ genes that are known to be associated with breast cancer (Figure \ref{fig:BRCA}). Interestingly, the order of important features was similar across methods, with BCL11A and SLC22A5 always ranking first, and TEX14 always being the least important of the $10$ BRCA-associated genes. However, MCI consistently gives non-zero median importance to all features, while UMFI correctly gives zero median importance to the majority of the randomized genes. Furthermore, UMFI's performance in this experiment improves with increased iterations. After running the experiment $5000$ times, both UMFI methods achieve perfect accuracy when distinguishing between important and permuted features (Supplement \ref{subsec:BRCA5000}). Although UMFI scores have higher variability than MCI (Table \ref{tab:table}), Figure \ref{fig:BRCA} shows that UMFI provides a more accurate and interpretable scoring, since it better separates the $10$ associated genes from the $40$ unassociated genes.

\begin{table}[h!]
 \caption{The standardized interquartile range (SIQR), true positive rate (TPR), true negative rate (TNR), overall accuracy (OA), and the number of features for which feature importance can be calculated within 60 minutes are displayed after running the methods on the BRCA data.}
  \centering
  \begin{tabular}{llllllll}             \\
    \textbf{Method} & \textbf{SIQR} & \textbf{TPR} & \textbf{TNR} & \textbf{OA} & \textbf{@1hr}\\
    \hline
    MCI (k=2) & \textbf{6.6 \%}  & \textbf{1} & 0 & 0.20 & 130  \\
    UMFI (LR) & 41.9\%  & \textbf{1} & \textbf{0.975} & \textbf{0.98} &  \textbf{4010}     \\
    UMFI (OT) & 28.5\%  & \textbf{1} & 0.775 & 0.82 & 3000 \\
    \hline \\
  \end{tabular}
  \label{tab:table}
\end{table}

\subsection{Computational Complexity}
\label{sec:time}

MCI must train and evaluate a model for each element of the power set of the feature set, which implies $O(2^p)$ model trainings if there are $p$ features. If the evaluation function $\nu$ obeys soft $k$-size submodularity, then the maximizing subset has no more than $k$ elements, which reduces the number of model trainings to $O(p^{k+1})$ \citep{catav}. UMFI circumvents the exponential training time since it can be evaluated after $O(p)$ model trainings, though the extra step of dependency removal is required. To confirm the above statements, and to show that the extra model trainings required for MCI dominate the computation time for removing dependencies in UMFI, we ran a simple timed experiment. For a range of dataset sizes from the BRCA data, we evaluate the computation time for calculating the feature importance scores of all features using MCI and UMFI. We ran this experiment for a dataset with $5$ features, and then slowly added features until our given time budget of $1$ hour ran out. Once all $50$ BRCA features were used, more features were randomly generated. All datasets had 571 observations. These experiments were run using an Intel Core i9-9980HK CPU 2.40GHz with 32GB of RAM. Code was parallelized in R, and $12$ of the $16$ available threads were used. 
\begin{figure}[h!]
\centering
{\includegraphics[width=0.48\textwidth]{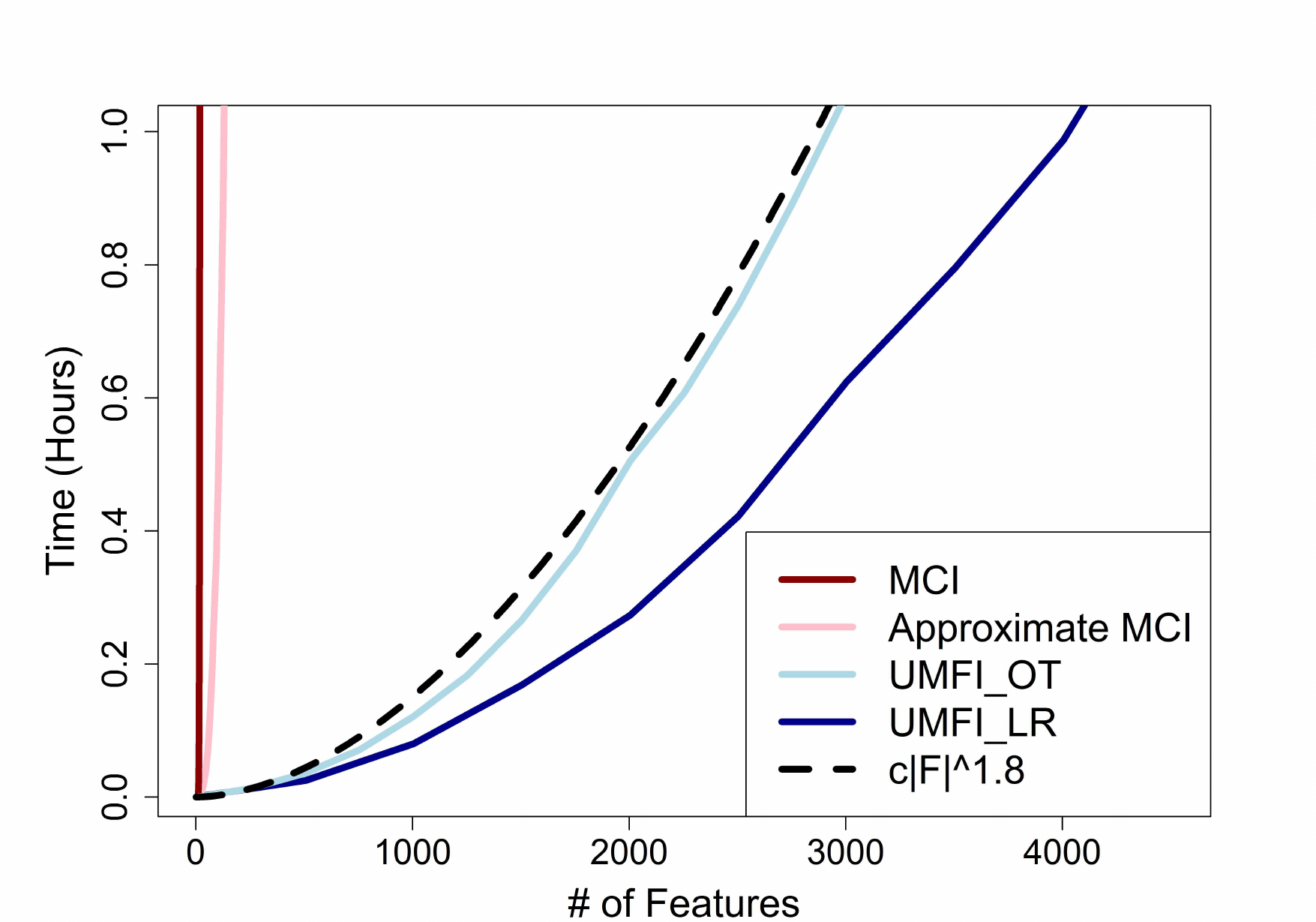}
 { \caption{Computation time for a single iteration of each method including: MCI (dark red), MCI with the soft 2-size-submodularity assumption (pink), UMFI\_OT (light blue), and UMFI\_LR (dark blue), plotted against the number of processed features.}
  \label{fig:time}}
}
\end{figure}

From Figure \ref{fig:time}, we can observe that UMFI is approximately super-linear, with UMFI\_OT incurring more computational cost compared to UMFI\_LR. Giving each method one hour to run, MCI processed 19 features, MCI with the soft 2-size submodularity assumption processed 130 features, UMFI\_OT processed about 3000 features, and UMFI\_LR processed about 4000 features (Table \ref{tab:table}).

\section{CONCLUSION}
\label{Conclusion}
In this study, we introduced three ideal axioms that feature importance measures should satisfy if they claim to be useful for learning from data. We then introduced ultra-marginal feature importance (UMFI), a new method that uses preprocessing techniques, originally developed in the domain of AI fairness, to provide fast and accurate feature importance scores for the purpose of explaining data. We proved that UMFI satisfies all three of the introduced axioms under certain assumptions. When compared with MCI, experimental results showed that UMFI, implemented with linear regression and optimal transport as preprocessing techniques, was able to provide more accurate estimates of feature importance on real and simulated data, particularly in the presence of correlated interactions and unrelated features. Supplement \ref{sec:FurtherEXP_feature_imp} shows that UMFI also compares favourably against other baseline methods including permutation importance, conditional permutation importance, and ablation.

Throughout the work on this paper, several shortcomings appeared. First, we only considered two simple methods for removing dependencies, linear regression and pairwise optimal transport. Other methods certainly exist in the literature, including optimal transport with chaining \citep{johndrow2019algorithm}, neural networks \citep{calmon2017optimized,song2019learning}, or principal inertial components \citep{wang2019privacy}. Though our two methods performed fairly well on the real and simulated datasets in Section \ref{sec:experiments}, optimal transport and linear regression failed to find representations of the data that were independent of the protected attribute when we tested the methods on a hydrology dataset with more shared information compared to BRCA \citep{addor2017camels} (Supplement \ref{sec:hydrology_exp}). However, neural nets, principal inertial components, or implementing optimal transport with better estimates of the conditional CDF certainly could have given better results. Although dependencies were not removed optimally for the hydrology dataset, the importance scores were still reasonably accurate. 
Second, UMFI scores are less robust than MCI since they have higher variability, however, because of the significantly lower computational cost, UMFI can be run multiple times and averaged to increase robustness and compute confidence intervals. Third, it is not clear how closely random forests or other measures of multivariate dependence can approximate the universal predictive power $\nu$ and mutual information in practice. In the same vein, it is unclear if multivariate measures of dependence usually satisfy properties such as monotonicity, redundancy invariance, and supermodularity, which we used UMFI's adherence to our proposed axioms. Finally, though UMFI can work for any arbitrary feature type, in this paper, we have only considered datasets with continuous explanatory variables. 


In future work, we would like to test how well other methods, such as neural networks, pair with UMFI while further testing on a wider variety of variable types (i.e., binary, categorical, and ordinal). We also believe that exploring the effectiveness of dependency removal techniques is worthwhile in its own right. We would also like to test how well UMFI scales to large datasets. In these settings, UMFI would benefit from a fast algorithm for computing confidence intervals and p-values to counteract its variability. Comparing different measures of multivariate dependence in their ability to approximate universal predictive power and satisfy common assumptions such as monotonicity would also be of interest to the broader feature importance community. Currently, UMFI is limited to providing a single measure of global variable importance. Extending our methods such that UMFI can be decomposed into different orders of interaction effects or redundant information in a similar way to functional ANOVA \citep{lengerich2020purifying,martens2020neural,hooker2007generalized,huang1998projection,stone1994use}, commonality analysis \citep{amado1999partitioning,daube2019quantitatively,seibold1979commonality,stoffel2021partr2,ray2014using}, partial information decomposition \citep{rosas2020reconciling,varley2023flickering,gutknecht2021bits,mediano2022greater,kolchinsky2022novel,suzuki2022decomposing,goodwell2022source,goodwell2020debates,gurushankar2022extracting,wollstadt2021rigorous}, or functional decomposition \citep{bordt2022shapley,hiabu2022unifying} could provide interesting future research directions. 

To reiterate, UMFI is a powerful tool for detecting and explaining the relationships hidden within observational data. We emphasise that UMFI is just a framework. A variety of other methods can be used to estimate the universal predictive power $\nu$ including, but not limited to, XGBoost, neural networks, QMD, or FOCI \citep{griessenberger2022multivariate,azadkia2021simple}. Furthermore, new preprocessing techniques for dependence removal are still being developed in the AI fairness community, so these, in addition to other existing methods, can be used in future applications of UMFI for additional improvements. We hope that UMFI will be a useful tool for learning from data in a variety of disciplines including bioinfomatics, earth sciences, psychology, and health science.

\subsubsection*{Acknowledgements}
We thank Dr. Ali A. Ameli for his support and motivation during this process. We thank all of the authors of the papers we cited, especially Dr. Chandra Nair for his help with concepts in information theory and Boyang Fu for his experimental advice.

\bibliography{references}

\begin{thebibliography}{}

\bibitem[Addor et~al., 2018]{addor2018ranking}
Addor, N., Nearing, G., Prieto, C., Newman, A., Le~Vine, N., and Clark, M.~P.
  (2018).
\newblock A ranking of hydrological signatures based on their predictability in
  space.
\newblock {\em Water Resources Research}, 54(11):8792--8812.

\bibitem[Addor et~al., 2017]{addor2017camels}
Addor, N., Newman, A.~J., Mizukami, N., and Clark, M.~P. (2017).
\newblock The camels data set: catchment attributes and meteorology for
  large-sample studies.
\newblock {\em Hydrology and Earth System Sciences}, 21(10):5293--5313.

\bibitem[Al-Ani et~al., 2003]{al2003new}
Al-Ani, A., Deriche, M., and Chebil, J. (2003).
\newblock A new mutual information based measure for feature selection.
\newblock {\em Intelligent Data Analysis}, 7(1):43--57.

\bibitem[Amado, 1999]{amado1999partitioning}
Amado, A.~J. (1999).
\newblock Partitioning predicted variance into constituent parts: A primer on
  regression commonality analysis.

\bibitem[Apley and Zhu, 2020]{apley2020visualizing}
Apley, D.~W. and Zhu, J. (2020).
\newblock Visualizing the effects of predictor variables in black box
  supervised learning models.
\newblock {\em Journal of the Royal Statistical Society: Series B (Statistical
  Methodology)}, 82(4):1059--1086.

\bibitem[Azadkia and Chatterjee, 2021]{azadkia2021simple}
Azadkia, M. and Chatterjee, S. (2021).
\newblock A simple measure of conditional dependence.
\newblock {\em The Annals of Statistics}, 49(6):3070--3102.

\bibitem[Bach et~al., 2018]{bach2018sensitive}
Bach, E., Radi{\'c}, V., and Schoof, C. (2018).
\newblock How sensitive are mountain glaciers to climate change? insights from
  a block model.
\newblock {\em Journal of Glaciology}, 64(244):247--258.

\bibitem[Battiti, 1994]{battiti1994using}
Battiti, R. (1994).
\newblock Using mutual information for selecting features in supervised neural
  net learning.
\newblock {\em IEEE Transactions on neural networks}, 5(4):537--550.

\bibitem[Bazaga et~al., 2020]{bazaga2020genome}
Bazaga, A., Leggate, D., and Weisser, H. (2020).
\newblock Genome-wide investigation of gene-cancer associations for the
  prediction of novel therapeutic targets in oncology.
\newblock {\em Scientific reports}, 10(1):1--10.

\bibitem[B{\'e}nard et~al., 2022]{benard2022shaff}
B{\'e}nard, C., Biau, G., Da~Veiga, S., and Scornet, E. (2022).
\newblock Shaff: Fast and consistent shapley effect estimates via random
  forests.
\newblock In {\em International Conference on Artificial Intelligence and
  Statistics}, pages 5563--5582. PMLR.

\bibitem[Bennasar et~al., 2015]{bennasar2015feature}
Bennasar, M., Hicks, Y., and Setchi, R. (2015).
\newblock Feature selection using joint mutual information maximisation.
\newblock {\em Expert Systems with Applications}, 42(22):8520--8532.

\bibitem[Bi, 2012]{bi2012review}
Bi, J. (2012).
\newblock A review of statistical methods for determination of relative
  importance of correlated predictors and identification of drivers of consumer
  liking.
\newblock {\em Journal of Sensory Studies}, 27(2):87--101.

\bibitem[Bird et~al., 2020]{bird2020fairlearn}
Bird, S., Dud{\'\i}k, M., Edgar, R., Horn, B., Lutz, R., Milan, V., Sameki, M.,
  Wallach, H., and Walker, K. (2020).
\newblock Fairlearn: A toolkit for assessing and improving fairness in ai.
\newblock {\em Microsoft, Tech. Rep. MSR-TR-2020-32}.

\bibitem[Bordt and von Luxburg, 2022]{bordt2022shapley}
Bordt, S. and von Luxburg, U. (2022).
\newblock From shapley values to generalized additive models and back.
\newblock {\em arXiv preprint arXiv:2209.04012}.

\bibitem[Breiman, 2001]{Breiman2001}
Breiman, L. (2001).
\newblock Random forests.
\newblock {\em Machine learning}, 45(1):5--32.

\bibitem[Breiman et~al., 2017]{breiman2017classification}
Breiman, L., Friedman, J.~H., Olshen, R.~A., and Stone, C.~J. (2017).
\newblock {\em Classification and regression trees}.
\newblock Routledge.

\bibitem[Brenning and Az{\'o}car, 2010]{brenning2010statistical}
Brenning, A. and Az{\'o}car, G. (2010).
\newblock Statistical analysis of topographic and climatic controls and
  multispectral signatures of rock glaciers in the dry andes, chile (27--33 s).
\newblock {\em Permafrost and Periglacial Processes}, 21(1):54--66.

\bibitem[Calmon et~al., 2017]{calmon2017optimized}
Calmon, F., Wei, D., Vinzamuri, B., Natesan~Ramamurthy, K., and Varshney, K.~R.
  (2017).
\newblock Optimized pre-processing for discrimination prevention.
\newblock {\em Advances in neural information processing systems}, 30.

\bibitem[Catav et~al., 2020]{catav2020marginal}
Catav, A., Fu, B., Ernst, J., Sankararaman, S., and Gilad-Bachrach, R. (2020).
\newblock Marginal contribution feature importance--an axiomatic approach for
  the natural case.
\newblock {\em arXiv preprint arXiv:2010.07910}.

\bibitem[Catav et~al., 2021]{catav}
Catav, A., Fu, B., Zoabi, Y., Meilik, A. L.~W., Shomron, N., Ernst, J.,
  Sankararaman, S., and Gilad-Bachrach, R. (2021).
\newblock Marginal contribution feature importance - an axiomatic approach for
  explaining data.
\newblock In Meila, M. and Zhang, T., editors, {\em Proceedings of the 38th
  International Conference on Machine Learning}, volume 139 of {\em Proceedings
  of Machine Learning Research}, pages 1324--1335. PMLR.

\bibitem[Chen et~al., 2020]{chen2020true}
Chen, H., Janizek, J.~D., Lundberg, S., and Lee, S.-I. (2020).
\newblock True to the model or true to the data?
\newblock {\em arXiv preprint arXiv:2006.16234}.

\bibitem[Chen et~al., 2015]{chen2015xgboost}
Chen, T., He, T., Benesty, M., Khotilovich, V., Tang, Y., Cho, H., Chen, K.,
  et~al. (2015).
\newblock Xgboost: extreme gradient boosting.
\newblock {\em R package version 0.4-2}, 1(4):1--4.

\bibitem[Cohen et~al., 2007]{cohen2007feature}
Cohen, S., Dror, G., and Ruppin, E. (2007).
\newblock Feature selection via coalitional game theory.
\newblock {\em Neural Computation}, 19(7):1939--1961.

\bibitem[Consortium et~al., 2007]{wellcome2007genome}
Consortium, W. T. C.~C. et~al. (2007).
\newblock Genome-wide association study of 14,000 cases of seven common
  diseases and 3,000 shared controls.
\newblock {\em Nature}, 447(7145):661.

\bibitem[Cover and Thomas, 2006]{Cover2006}
Cover, T.~M. and Thomas, J.~A. (2006).
\newblock {\em Elements of Information Theory 2nd Edition (Wiley Series in
  Telecommunications and Signal Processing)}.
\newblock Wiley-Interscience.

\bibitem[Covert et~al., 2020]{covert2020understanding}
Covert, I., Lundberg, S.~M., and Lee, S.-I. (2020).
\newblock Understanding global feature contributions with additive importance
  measures.
\newblock {\em Advances in Neural Information Processing Systems},
  33:17212--17223.

\bibitem[Covert et~al., 2021]{covert2021explaining}
Covert, I., Lundberg, S.~M., and Lee, S.-I. (2021).
\newblock Explaining by removing: A unified framework for model explanation.
\newblock {\em J. Mach. Learn. Res.}, 22:209--1.

\bibitem[Darlington, 1968]{darlington1968multiple}
Darlington, R.~B. (1968).
\newblock Multiple regression in psychological research and practice.
\newblock {\em Psychological bulletin}, 69(3):161.

\bibitem[Daube et~al., 2019]{daube2019quantitatively}
Daube, C., Giordano, B., Schyns, P.~G., and Ince, R.~A. (2019).
\newblock Quantitatively comparing predictive models with the partial
  information decomposition.

\bibitem[Dawid, 1979]{dawid1979conditional}
Dawid, A.~P. (1979).
\newblock Conditional independence in statistical theory.
\newblock {\em Journal of the Royal Statistical Society: Series B
  (Methodological)}, 41(1):1--15.

\bibitem[Debeer et~al., 2021]{debeer2021package}
Debeer, D., Hothorn, T., Strobl, C., and Debeer, M.~D. (2021).
\newblock Package ‘permimp’.

\bibitem[Debeer and Strobl, 2020]{debeer2020conditional}
Debeer, D. and Strobl, C. (2020).
\newblock Conditional permutation importance revisited.
\newblock {\em BMC bioinformatics}, 21(1):1--30.

\bibitem[DelSole and Tippett, 2007]{delsole2007predictability}
DelSole, T. and Tippett, M.~K. (2007).
\newblock Predictability: Recent insights from information theory.
\newblock {\em Reviews of Geophysics}, 45(4).

\bibitem[Dutta et~al., 2020]{dutta2020information}
Dutta, S., Venkatesh, P., Mardziel, P., Datta, A., and Grover, P. (2020).
\newblock An information-theoretic quantification of discrimination with exempt
  features.
\newblock In {\em Proceedings of the AAAI Conference on Artificial
  Intelligence}, volume~34, pages 3825--3833.

\bibitem[Dutta et~al., 2021]{dutta2021fairness}
Dutta, S., Venkatesh, P., Mardziel, P., Datta, A., and Grover, P. (2021).
\newblock Fairness under feature exemptions: Counterfactual and observational
  measures.
\newblock {\em IEEE Transactions on Information Theory}, 67(10):6675--6710.

\bibitem[Easton et~al., 2007]{easton2007genome}
Easton, D.~F., Pooley, K.~A., Dunning, A.~M., Pharoah, P.~D., Thompson, D.,
  Ballinger, D.~G., Struewing, J.~P., Morrison, J., Field, H., Luben, R.,
  et~al. (2007).
\newblock Genome-wide association study identifies novel breast cancer
  susceptibility loci.
\newblock {\em Nature}, 447(7148):1087--1093.

\bibitem[Edwards et~al., 2021]{edwards2021projected}
Edwards, T.~L., Nowicki, S., Marzeion, B., Hock, R., Goelzer, H., Seroussi, H.,
  Jourdain, N.~C., Slater, D.~A., Turner, F.~E., Smith, C.~J., et~al. (2021).
\newblock Projected land ice contributions to twenty-first-century sea level
  rise.
\newblock {\em Nature}, 593(7857):74--82.

\bibitem[Fan and Lv, 2008]{fan2008sure}
Fan, J. and Lv, J. (2008).
\newblock Sure independence screening for ultrahigh dimensional feature space.
\newblock {\em Journal of the Royal Statistical Society: Series B (Statistical
  Methodology)}, 70(5):849--911.

\bibitem[Freiesleben et~al., 2022]{freiesleben2022scientific}
Freiesleben, T., K{\"o}nig, G., Molnar, C., and Tejero-Cantero, A. (2022).
\newblock Scientific inference with interpretable machine learning: Analyzing
  models to learn about real-world phenomena.
\newblock {\em arXiv preprint arXiv:2206.05487}.

\bibitem[Galton, 1889]{galton1889co}
Galton, F. (1889).
\newblock I. co-relations and their measurement, chiefly from anthropometric
  data.
\newblock {\em Proceedings of the Royal Society of London},
  45(273-279):135--145.

\bibitem[Geurts et~al., 2006]{geurts2006extremely}
Geurts, P., Ernst, D., and Wehenkel, L. (2006).
\newblock Extremely randomized trees.
\newblock {\em Machine learning}, 63(1):3--42.

\bibitem[Gibson, 1962]{gibson1962orthogonal}
Gibson, W. (1962).
\newblock Orthogonal predictors: A possible resolution of the hoffman-ward
  controversy.
\newblock {\em Psychological reports}, 11(1):32--34.

\bibitem[Gill et~al., 2017]{gill2017capacity}
Gill, D.~A., Mascia, M.~B., Ahmadia, G.~N., Glew, L., Lester, S.~E., Barnes,
  M., Craigie, I., Darling, E.~S., Free, C.~M., Geldmann, J., et~al. (2017).
\newblock Capacity shortfalls hinder the performance of marine protected areas
  globally.
\newblock {\em Nature}, 543(7647):665--669.

\bibitem[Gitiaux and Rangwala, 2021a]{gitiaux2021fair}
Gitiaux, X. and Rangwala, H. (2021a).
\newblock Fair representations by compression.
\newblock In {\em Proceedings of the AAAI Conference on Artificial
  Intelligence}, volume~35, pages 11506--11515.

\bibitem[Gitiaux and Rangwala, 2021b]{gitiaux2021learning}
Gitiaux, X. and Rangwala, H. (2021b).
\newblock Learning smooth and fair representations.
\newblock In {\em International conference on artificial intelligence and
  statistics}, pages 253--261. PMLR.

\bibitem[Gitiaux and Rangwala, 2022]{gitiaux2022sofair}
Gitiaux, X. and Rangwala, H. (2022).
\newblock Sofair: Single shot fair representation learning.
\newblock {\em arXiv preprint arXiv:2204.12556}.

\bibitem[Gong et~al., 2013]{gong2013estimating}
Gong, W., Gupta, H.~V., Yang, D., Sricharan, K., and Hero~III, A.~O. (2013).
\newblock Estimating epistemic and aleatory uncertainties during hydrologic
  modeling: An information theoretic approach.
\newblock {\em Water resources research}, 49(4):2253--2273.

\bibitem[Goodwell and Bassiouni, 2022]{goodwell2022source}
Goodwell, A.~E. and Bassiouni, M. (2022).
\newblock Source relationships and model structures determine information flow
  paths in ecohydrologic models.
\newblock {\em Water Resources Research}, 58(9):e2021WR031164.

\bibitem[Goodwell et~al., 2020]{goodwell2020debates}
Goodwell, A.~E., Jiang, P., Ruddell, B.~L., and Kumar, P. (2020).
\newblock Debates—does information theory provide a new paradigm for earth
  science? causality, interaction, and feedback.
\newblock {\em Water Resources Research}, 56(2):e2019WR024940.

\bibitem[Greenland et~al., 1999]{greenland1999causal}
Greenland, S., Pearl, J., and Robins, J.~M. (1999).
\newblock Causal diagrams for epidemiologic research.
\newblock {\em Epidemiology}, pages 37--48.

\bibitem[Griessenberger et~al., 2022]{griessenberger2022multivariate}
Griessenberger, F., Junker, R.~R., and Trutschnig, W. (2022).
\newblock On a multivariate copula-based dependence measure and its estimation.
\newblock {\em Electronic Journal of Statistics}, 16(1):2206--2251.

\bibitem[Griffith and Koch, 2014]{griffith2014quantifying}
Griffith, V. and Koch, C. (2014).
\newblock Quantifying synergistic mutual information.
\newblock {\em Guided self-organization: inception}, pages 159--190.

\bibitem[Gr{\"o}mping, 2009]{gromping2009variable}
Gr{\"o}mping, U. (2009).
\newblock Variable importance assessment in regression: linear regression
  versus random forest.
\newblock {\em The American Statistician}, 63(4):308--319.

\bibitem[Gurushankar et~al., 2022]{gurushankar2022extracting}
Gurushankar, K., Venkatesh, P., and Grover, P. (2022).
\newblock Extracting unique information through markov relations.
\newblock In {\em 2022 58th Annual Allerton Conference on Communication,
  Control, and Computing (Allerton)}, pages 1--6. IEEE.

\bibitem[Gutknecht et~al., 2021]{gutknecht2021bits}
Gutknecht, A.~J., Wibral, M., and Makkeh, A. (2021).
\newblock Bits and pieces: Understanding information decomposition from
  part-whole relationships and formal logic.
\newblock {\em Proceedings of the Royal Society A}, 477(2251):20210110.

\bibitem[Harder et~al., 2013]{harder2013bivariate}
Harder, M., Salge, C., and Polani, D. (2013).
\newblock Bivariate measure of redundant information.
\newblock {\em Physical Review E}, 87(1):012130.

\bibitem[Harel et~al., 2022]{harel2022inherent}
Harel, N., Gilad-Bachrach, R., and Obolski, U. (2022).
\newblock Inherent inconsistencies of feature importance.
\newblock {\em arXiv preprint arXiv:2206.08204}.

\bibitem[Hiabu et~al., 2022]{hiabu2022unifying}
Hiabu, M., Meyer, J.~T., and Wright, M.~N. (2022).
\newblock Unifying local and global model explanations by functional
  decomposition of low dimensional structures.
\newblock {\em arXiv preprint arXiv:2208.06151}.

\bibitem[Hooker, 2007]{hooker2007generalized}
Hooker, G. (2007).
\newblock Generalized functional anova diagnostics for high-dimensional
  functions of dependent variables.
\newblock {\em Journal of Computational and Graphical Statistics},
  16(3):709--732.

\bibitem[Hooker et~al., 2021]{hooker2021unrestricted}
Hooker, G., Mentch, L., and Zhou, S. (2021).
\newblock Unrestricted permutation forces extrapolation: variable importance
  requires at least one more model, or there is no free variable importance.
\newblock {\em Statistics and Computing}, 31(6):1--16.

\bibitem[Huang, 1998]{huang1998projection}
Huang, J.~Z. (1998).
\newblock Projection estimation in multiple regression with application to
  functional anova models.
\newblock {\em The annals of statistics}, 26(1):242--272.

\bibitem[Jakulin and Bratko, 2003]{jakulin2003quantifying}
Jakulin, A. and Bratko, I. (2003).
\newblock Quantifying and visualizing attribute interactions: An approach based
  on entropy.

\bibitem[Janssen et~al., 2022]{janssen2022application}
Janssen, A., Hoogendoorn, M., Cnossen, M.~H., Math{\^o}t, R.~A., Group,
  O.-C.~S., Consortium, S., Cnossen, M., Reitsma, S., Leebeek, F., Math{\^o}t,
  R., Fijnvandraat, K., et~al. (2022).
\newblock Application of shap values for inferring the optimal functional form
  of covariates in pharmacokinetic modeling.
\newblock {\em CPT: Pharmacometrics \& Systems Pharmacology}.

\bibitem[Janssen and Ameli, 2021]{janssen2021hydrologic}
Janssen, J. and Ameli, A.~A. (2021).
\newblock A hydrologic functional approach for improving large-sample hydrology
  performance in poorly gauged regions.
\newblock {\em Water Resources Research}, 57(9):e2021WR030263.

\bibitem[Jehn et~al., 2020]{jehn2020using}
Jehn, F.~U., Bestian, K., Breuer, L., Kraft, P., and Houska, T. (2020).
\newblock Using hydrological and climatic catchment clusters to explore drivers
  of catchment behavior.
\newblock {\em Hydrology and Earth System Sciences}, 24(3):1081--1100.

\bibitem[Johndrow and Lum, 2019]{johndrow2019algorithm}
Johndrow, J.~E. and Lum, K. (2019).
\newblock An algorithm for removing sensitive information: application to
  race-independent recidivism prediction.
\newblock {\em The Annals of Applied Statistics}, 13(1):189--220.

\bibitem[Johnsen et~al., 2021]{johnsen2021new}
Johnsen, P.~V., Riemer-S{\o}rensen, S., DeWan, A.~T., Cahill, M.~E., and
  Langaas, M. (2021).
\newblock A new method for exploring gene--gene and gene--environment
  interactions in gwas with tree ensemble methods and shap values.
\newblock {\em BMC bioinformatics}, 22(1):1--29.

\bibitem[Kang and Tian, 2009]{kang2009markov}
Kang, C. and Tian, J. (2009).
\newblock Markov properties for linear causal models with correlated errors.
\newblock {\em Journal of Machine Learning Research}, 10(1).

\bibitem[Kinney and Atwal, 2014]{kinney2014equitability}
Kinney, J.~B. and Atwal, G.~S. (2014).
\newblock Equitability, mutual information, and the maximal information
  coefficient.
\newblock {\em Proceedings of the National Academy of Sciences},
  111(9):3354--3359.

\bibitem[Kolchinsky, 2022]{kolchinsky2022novel}
Kolchinsky, A. (2022).
\newblock A novel approach to the partial information decomposition.
\newblock {\em Entropy}, 24(3):403.

\bibitem[K{\"o}nig et~al., 2021]{konig2021decomposition}
K{\"o}nig, G., Freiesleben, T., Bischl, B., Casalicchio, G., and
  Grosse-Wentrup, M. (2021).
\newblock Decomposition of global feature importance into direct and
  associative components (dedact).
\newblock {\em arXiv preprint arXiv:2106.08086}.

\bibitem[Kraskov et~al., 2004]{kraskov2004estimating}
Kraskov, A., St{\"o}gbauer, H., and Grassberger, P. (2004).
\newblock Estimating mutual information.
\newblock {\em Physical review E}, 69(6):066138.

\bibitem[Kruskal, 1984]{kruskal1984concepts}
Kruskal, W. (1984).
\newblock Concepts of relative importance.
\newblock {\em Q{\"u}estii{\'o}. 1984, vol. 8, n{\'u}m. 1}.

\bibitem[Lau et~al., 2022]{lau2022mutual}
Lau, K., Nair, C., and Ng, D. (2022).
\newblock A mutual information inequality and some applications.
\newblock In {\em 2022 IEEE International Symposium on Information Theory
  (ISIT)}, pages 951--956. IEEE.

\bibitem[Lauritzen and Sadeghi, 2018]{lauritzen2018unifying}
Lauritzen, S. and Sadeghi, K. (2018).
\newblock Unifying markov properties for graphical models.
\newblock {\em The Annals of Statistics}, 46(5):2251--2278.

\bibitem[Le et~al., 2022]{le2022snow}
Le, E., Ameli, A., Janssen, J., and Hammond, J. (2022).
\newblock Snow persistence explains stream high flow and low flow signatures
  with differing relationships by aridity and climatic seasonality.
\newblock {\em Hydrology and Earth System Sciences Discussions}, pages 1--22.

\bibitem[Lengerich et~al., 2020]{lengerich2020purifying}
Lengerich, B., Tan, S., Chang, C.-H., Hooker, G., and Caruana, R. (2020).
\newblock Purifying interaction effects with the functional anova: An efficient
  algorithm for recovering identifiable additive models.
\newblock In {\em International Conference on Artificial Intelligence and
  Statistics}, pages 2402--2412. PMLR.

\bibitem[Li and Ameli, 2022]{li2022statistical}
Li, H. and Ameli, A. (2022).
\newblock A statistical approach for identifying factors governing streamflow
  recession behaviour.
\newblock {\em Hydrological Processes}, 36(10):e14718.

\bibitem[Liaw et~al., 2002]{liaw2002classification}
Liaw, A., Wiener, M., et~al. (2002).
\newblock Classification and regression by randomforest.
\newblock {\em R news}, 2(3):18--22.

\bibitem[Louppe et~al., 2013]{louppe2013understanding}
Louppe, G., Wehenkel, L., Sutera, A., and Geurts, P. (2013).
\newblock Understanding variable importances in forests of randomized trees.
\newblock {\em Advances in neural information processing systems}, 26.

\bibitem[Lundberg and Lee, 2017]{lundberg2017unified}
Lundberg, S.~M. and Lee, S.-I. (2017).
\newblock A unified approach to interpreting model predictions.
\newblock {\em Advances in neural information processing systems}, 30.

\bibitem[Malouche and Sevestre-Ghalila, 2008]{malouche2008estimating}
Malouche, D. and Sevestre-Ghalila, S. (2008).
\newblock Estimating high dimensional faithful gaussian graphical models by
  low-order conditioning.
\newblock In {\em Proceeding, of 26th IASTED International Multi-Conference on
  Applied Informatics, Artificial Intelligence and Applications}, pages
  595--025.

\bibitem[M{\"a}rtens and Yau, 2020]{martens2020neural}
M{\"a}rtens, K. and Yau, C. (2020).
\newblock Neural decomposition: Functional anova with variational autoencoders.
\newblock In {\em International Conference on Artificial Intelligence and
  Statistics}, pages 2917--2927. PMLR.

\bibitem[Marx et~al., 2021]{marx2021weaker}
Marx, A., Gretton, A., and Mooij, J.~M. (2021).
\newblock A weaker faithfulness assumption based on triple interactions.
\newblock In {\em Uncertainty in Artificial Intelligence}, pages 451--460.
  PMLR.

\bibitem[Marx et~al., 2022]{marx2022but}
Marx, C., Park, Y., Hasson, H., Wang, Y.~B., Ermon, S., and Huan, J. (2022).
\newblock But are you sure? an uncertainty-aware perspective on explainable ai.

\bibitem[Matthijs et~al., 2019]{scikitFairness}
Matthijs, Warmerdam, V., and ManyOthers (2019).
\newblock scikit-fairness.
\newblock \url{scikit-fairness. https://github.com/koaning/scikit-fairness}.

\bibitem[Mediano et~al., 2022]{mediano2022greater}
Mediano, P.~A., Rosas, F.~E., Luppi, A.~I., Jensen, H.~J., Seth, A.~K.,
  Barrett, A.~B., Carhart-Harris, R.~L., and Bor, D. (2022).
\newblock Greater than the parts: a review of the information decomposition
  approach to causal emergence.
\newblock {\em Philosophical Transactions of the Royal Society A},
  380(2227):20210246.

\bibitem[Molnar, 2020]{molnar2020interpretable}
Molnar, C. (2020).
\newblock {\em Interpretable machine learning}.
\newblock Lulu. com.

\bibitem[Molnar et~al., 2021]{molnar2021relating}
Molnar, C., Freiesleben, T., K{\"o}nig, G., Casalicchio, G., Wright, M.~N., and
  Bischl, B. (2021).
\newblock Relating the partial dependence plot and permutation feature
  importance to the data generating process.
\newblock {\em arXiv preprint arXiv:2109.01433}.

\bibitem[Moyer et~al., 2018]{moyer2018invariant}
Moyer, D., Gao, S., Brekelmans, R., Galstyan, A., and Ver~Steeg, G. (2018).
\newblock Invariant representations without adversarial training.
\newblock {\em Advances in Neural Information Processing Systems}, 31.

\bibitem[Orlenko and Moore, 2021]{orlenko2021comparison}
Orlenko, A. and Moore, J.~H. (2021).
\newblock A comparison of methods for interpreting random forest models of
  genetic association in the presence of non-additive interactions.
\newblock {\em BioData mining}, 14(1):1--17.

\bibitem[Probst et~al., 2019]{probst2019tunability}
Probst, P., Boulesteix, A.-L., and Bischl, B. (2019).
\newblock Tunability: importance of hyperparameters of machine learning
  algorithms.
\newblock {\em The Journal of Machine Learning Research}, 20(1):1934--1965.

\bibitem[Ray-Mukherjee et~al., 2014]{ray2014using}
Ray-Mukherjee, J., Nimon, K., Mukherjee, S., Morris, D.~W., Slotow, R., and
  Hamer, M. (2014).
\newblock Using commonality analysis in multiple regressions: a tool to
  decompose regression effects in the face of multicollinearity.
\newblock {\em Methods in Ecology and Evolution}, 5(4):320--328.

\bibitem[Reisach et~al., 2021]{reisach2021beware}
Reisach, A., Seiler, C., and Weichwald, S. (2021).
\newblock Beware of the simulated dag! causal discovery benchmarks may be easy
  to game.
\newblock {\em Advances in Neural Information Processing Systems},
  34:27772--27784.

\bibitem[Rosas et~al., 2020]{rosas2020reconciling}
Rosas, F.~E., Mediano, P.~A., Jensen, H.~J., Seth, A.~K., Barrett, A.~B.,
  Carhart-Harris, R.~L., and Bor, D. (2020).
\newblock Reconciling emergences: An information-theoretic approach to identify
  causal emergence in multivariate data.
\newblock {\em PLoS computational biology}, 16(12):e1008289.

\bibitem[Schellhas et~al., 2020]{schellhas2020distance}
Schellhas, D., Neupane, B., Thammineni, D., Kanumuri, B., and Green, R.~C.
  (2020).
\newblock Distance correlation sure independence screening for accelerated
  feature selection in parkinson’s disease vocal data.
\newblock In {\em 2020 International Conference on Computational Science and
  Computational Intelligence (CSCI)}, pages 1433--1438. IEEE.

\bibitem[Schmidt et~al., 2020]{schmidt2020challenges}
Schmidt, L., He{\ss}e, F., Attinger, S., and Kumar, R. (2020).
\newblock Challenges in applying machine learning models for hydrological
  inference: A case study for flooding events across germany.
\newblock {\em Water Resources Research}, 56(5):e2019WR025924.

\bibitem[Seibold and McPHEE, 1979]{seibold1979commonality}
Seibold, D.~R. and McPHEE, R.~D. (1979).
\newblock Commonality analysis: A method for decomposing explained variance in
  multiple regression analyses.
\newblock {\em Human Communication Research}, 5(4):355--365.

\bibitem[Sevestre and Benn, 2015]{sevestre2015climatic}
Sevestre, H. and Benn, D.~I. (2015).
\newblock Climatic and geometric controls on the global distribution of
  surge-type glaciers: implications for a unifying model of surging.
\newblock {\em Journal of Glaciology}, 61(228):646--662.

\bibitem[Shapley, 1953]{shapley1953value}
Shapley, L.~S. (1953).
\newblock A value for n-person games, contributions to the theory of games, 2,
  307--317.

\bibitem[Shpitser and Pearl, 2008]{shpitser2008complete}
Shpitser, I. and Pearl, J. (2008).
\newblock Complete identification methods for the causal hierarchy.
\newblock {\em Journal of Machine Learning Research}, 9:1941--1979.

\bibitem[Soleymani et~al., 2022]{soleymani2022causal}
Soleymani, A., Raj, A., Bauer, S., Sch{\"o}lkopf, B., and Besserve, M. (2022).
\newblock Causal feature selection via orthogonal search.
\newblock {\em Transactions on Machine Learning Research}.

\bibitem[Song et~al., 2019]{song2019learning}
Song, J., Kalluri, P., Grover, A., Zhao, S., and Ermon, S. (2019).
\newblock Learning controllable fair representations.
\newblock In {\em The 22nd International Conference on Artificial Intelligence
  and Statistics}, pages 2164--2173. PMLR.

\bibitem[Spearman, 1961]{spearman1961general}
Spearman, C. (1961).
\newblock " general intelligence" objectively determined and measured.

\bibitem[Stein et~al., 2021]{stein2021climate}
Stein, L., Clark, M.~P., Knoben, W.~J., Pianosi, F., and Woods, R.~A. (2021).
\newblock How do climate and catchment attributes influence flood generating
  processes? a large-sample study for 671 catchments across the contiguous usa.
\newblock {\em Water Resources Research}, 57(4):e2020WR028300.

\bibitem[Steudel and Ay, 2015]{steudel2015information}
Steudel, B. and Ay, N. (2015).
\newblock Information-theoretic inference of common ancestors.
\newblock {\em Entropy}, 17(4):2304--2327.

\bibitem[Stoffel et~al., 2021]{stoffel2021partr2}
Stoffel, M.~A., Nakagawa, S., and Schielzeth, H. (2021).
\newblock partr2: partitioning r2 in generalized linear mixed models.
\newblock {\em PeerJ}, 9:e11414.

\bibitem[Stone, 1994]{stone1994use}
Stone, C.~J. (1994).
\newblock The use of polynomial splines and their tensor products in
  multivariate function estimation.
\newblock {\em The annals of statistics}, 22(1):118--171.

\bibitem[Sun et~al., 2021]{sun2021revisiting}
Sun, S., Dong, B., and Zou, Q. (2021).
\newblock Revisiting genome-wide association studies from statistical modelling
  to machine learning.
\newblock {\em Briefings in Bioinformatics}, 22(4):bbaa263.

\bibitem[Sutera et~al., 2021]{sutera2021global}
Sutera, A., Louppe, G., Huynh-Thu, V.~A., Wehenkel, L., and Geurts, P. (2021).
\newblock From global to local mdi variable importances for random forests and
  when they are shapley values.
\newblock {\em Advances in Neural Information Processing Systems}, 34.

\bibitem[Suzuki et~al., 2022]{suzuki2022decomposing}
Suzuki, K., Matsuzaki, S.-i.~S., and Masuya, H. (2022).
\newblock Decomposing predictability to identify dominant causal drivers in
  complex ecosystems.
\newblock {\em Proceedings of the National Academy of Sciences},
  119(42):e2204405119.

\bibitem[Tan et~al., 2020]{tan2020learning}
Tan, Z., Yeom, S., Fredrikson, M., and Talwalkar, A. (2020).
\newblock Learning fair representations for kernel models.
\newblock In {\em International Conference on Artificial Intelligence and
  Statistics}, pages 155--166. PMLR.

\bibitem[Taufiq et~al., 2023]{taufiq2023manifold}
Taufiq, M.~F., Bl{\"o}baum, P., and Minorics, L. (2023).
\newblock Manifold restricted interventional shapley values.
\newblock {\em arXiv preprint arXiv:2301.04041}.

\bibitem[Tolo{\c{s}}i and Lengauer, 2011]{tolocsi2011classification}
Tolo{\c{s}}i, L. and Lengauer, T. (2011).
\newblock Classification with correlated features: unreliability of feature
  ranking and solutions.
\newblock {\em Bioinformatics}, 27(14):1986--1994.

\bibitem[Tomczak et~al., 2015]{tomczak2015cancer}
Tomczak, K., Czerwi{\'n}ska, P., and Wiznerowicz, M. (2015).
\newblock The cancer genome atlas (tcga): an immeasurable source of knowledge.
\newblock {\em Contemporary oncology}, 19(1A):A68.

\bibitem[Varley, 2023]{varley2023flickering}
Varley, T.~F. (2023).
\newblock Flickering emergences: The question of locality in
  information-theoretic approaches to emergence.
\newblock {\em Entropy}, 25(1):54.

\bibitem[Vowels et~al., 2021]{vowels2021d}
Vowels, M.~J., Camgoz, N.~C., and Bowden, R. (2021).
\newblock D’ya like dags? a survey on structure learning and causal
  discovery.
\newblock {\em ACM Computing Surveys (CSUR)}.

\bibitem[Wang et~al., 2021]{wang2021genome}
Wang, H., Bennett, D.~A., De~Jager, P.~L., Zhang, Q.-Y., and Zhang, H.-Y.
  (2021).
\newblock Genome-wide epistasis analysis for alzheimer’s disease and
  implications for genetic risk prediction.
\newblock {\em Alzheimer's research \& therapy}, 13(1):1--13.

\bibitem[Wang and Calmon, 2017]{wang2017estimation}
Wang, H. and Calmon, F.~P. (2017).
\newblock An estimation-theoretic view of privacy.
\newblock In {\em 2017 55th Annual Allerton Conference on Communication,
  Control, and Computing (Allerton)}, pages 886--893. IEEE.

\bibitem[Wang et~al., 2019]{wang2019privacy}
Wang, H., Vo, L., Calmon, F.~P., M{\'e}dard, M., Duffy, K.~R., and Varia, M.
  (2019).
\newblock Privacy with estimation guarantees.
\newblock {\em IEEE Transactions on Information Theory}, 65(12):8025--8042.

\bibitem[Williams and Beer, 2010]{williams2010nonnegative}
Williams, P.~L. and Beer, R.~D. (2010).
\newblock Nonnegative decomposition of multivariate information.
\newblock {\em arXiv preprint arXiv:1004.2515}.

\bibitem[Williams et~al., 2018]{williams2018directed}
Williams, T.~C., Bach, C.~C., Matthiesen, N.~B., Henriksen, T.~B., and
  Gagliardi, L. (2018).
\newblock Directed acyclic graphs: a tool for causal studies in paediatrics.
\newblock {\em Pediatric research}, 84(4):487--493.

\bibitem[Williamson and Feng, 2020]{williamson2020efficient}
Williamson, B. and Feng, J. (2020).
\newblock Efficient nonparametric statistical inference on population feature
  importance using shapley values.
\newblock In {\em International Conference on Machine Learning}, pages
  10282--10291. PMLR.

\bibitem[Wollstadt et~al., 2021]{wollstadt2021rigorous}
Wollstadt, P., Schmitt, S., and Wibral, M. (2021).
\newblock A rigorous information-theoretic definition of redundancy and
  relevancy in feature selection based on (partial) information decomposition.
\newblock {\em arXiv preprint arXiv:2105.04187}.

\bibitem[Wright and Ziegler, 2015]{wright2015ranger}
Wright, M.~N. and Ziegler, A. (2015).
\newblock ranger: A fast implementation of random forests for high dimensional
  data in c++ and r.
\newblock {\em arXiv preprint arXiv:1508.04409}.

\bibitem[Wright et~al., 2016]{wright2016little}
Wright, M.~N., Ziegler, A., and K{\"o}nig, I.~R. (2016).
\newblock Do little interactions get lost in dark random forests?
\newblock {\em BMC bioinformatics}, 17:1--10.

\bibitem[Wright, 1921]{wright1921correlation}
Wright, S. (1921).
\newblock Correlation and causation.

\bibitem[Wurm and Fisicaro, 2014]{wurm2014residualizing}
Wurm, L.~H. and Fisicaro, S.~A. (2014).
\newblock What residualizing predictors in regression analyses does (and what
  it does not do).
\newblock {\em Journal of memory and language}, 72:37--48.

\bibitem[Yang and Ong, 2012]{yang2012effective}
Yang, J.-B. and Ong, C.-J. (2012).
\newblock An effective feature selection method via mutual information
  estimation.
\newblock {\em IEEE Transactions on Systems, Man, and Cybernetics, Part B
  (Cybernetics)}, 42(6):1550--1559.

\bibitem[Yeung, 2002]{yeung2002first}
Yeung, R.~W. (2002).
\newblock {\em A first course in information theory}.
\newblock Springer Science \& Business Media.

\bibitem[Yu, 2013]{yu2013stability}
Yu, B. (2013).
\newblock Stability.
\newblock {\em Bernoulli}, 19(4):1484--1500.

\bibitem[Yu et~al., 2018]{yu2018mining}
Yu, K., Liu, L., Li, J., and Chen, H. (2018).
\newblock Mining markov blankets without causal sufficiency.
\newblock {\em IEEE transactions on neural networks and learning systems},
  29(12):6333--6347.

\end{thebibliography}

\appendix
\onecolumn

\section{Mutual information}

\subsection{Properties of mutual information}
\label{Properties_of_mutual_information}

\begin{theorem}[Symmetry of conditional mutual information \citep{yeung2002first}]
$$I(Y;X|Z) = I(X;Y|Z)$$
\label{symmetry_mutual_info}
\end{theorem}

\begin{theorem}[Chain rule for mutual information \citep{yeung2002first}]
$$I(Y;X,Z)= I(Y;Z) + I(Y;X|Z) = I(Y;X) + I(Y;Z|X)$$
\label{chain_rule_mutual_info}
\end{theorem}

\begin{theorem}[Supermodularity under independence]
Let $S,X_1, X_2$ be random variables such that $X_1 \indep (S,X_2)$. Then, $I(Y; S, X_1, X_2) - I(Y; S,X_2) \ge I(Y;S,X_1)-I(Y;S)$ \citep{lau2022mutual,steudel2015information}.
\label{supermodularity_theorem}
\end{theorem}

\begin{proof}
\begin{align*}
    & I(Y; S, X_1, X_2) - I(Y; S,X_2) \\
    &= I(Y;S,X_2) + I(Y;X_1|S,X_2) - I(Y; S,X_2) \quad\text{(by chain rule)} \\
    & =I(Y;X_1|S,X_2) = I(X_1;Y|S,X_2) \quad\text{(by symmetry)}\\
    & =I(X_1; Y, S, X_2) - I(X_1;S,X_2) \quad\text{(by chain rule)} \\
    &= I(X_1; Y,S,X_2)= I(Y,S,X_2;X_1) \quad\text{(by $X_1 \indep (S,X_2)$ and symmetry)}\\
    &\ge I(Y,S;X_1) = I(X_1;Y,S)  \quad\text{(by monotonicity of mutual information and symmetry)} \\
    &= I(X_1;Y|S) = I(Y;X_1|S) \quad\text{(by chain rule, $X_1 \indep S$, and symmetry)}\\
    &= I(Y;S,X_1)-I(Y;S) \quad\text{(by chain rule)}
\end{align*}
\end{proof}

\begin{theorem}[Data processing inequality]
Let $X,Y,Z$ be three random variables forming a Markov chain $X \to Y \to Z$, i.e. $X \indep Z | Y$. Then, $I(X;Y) \ge I(X;Z)$.
\label{data_processing}
\end{theorem}
\begin{proof}
The proof can be found in \citet[p. 32]{Cover2006}.
\end{proof}

\begin{theorem}
Let $F$ be a set of features used to predict the response $Y$. Then $I(Y; F) \ge I(Y;g(F))$ for any deterministic function $g$. If $g$ is injective, then $I(Y; F) = I(Y;g(F))$. 
\label{mut_info_bounds}
\end{theorem}
\begin{proof}
The first claim $I(Y; F) \ge I(Y;g(F))$ follows from the data processing inequality (Theorem \ref{data_processing}) since $Y \to F \to g(F)$ forms a Markov chain.


If $g$ is injective, then we may write $F=h(g(F))$ where $h: Im(g) \to F$ is the inverse of $g$ restricted to the image of $g$. By the data processing inequality and the fact that $Y \indep h(g(F)) | g(F)$ we know that $I(Y;g(F)) \geq I(Y;h(g(F)))$. Then by the definition of $h$ and $g$, we know that $I(Y;h(g(F)))=I(Y;F)$, thus $I(Y;g(F)) \geq I(Y;F)$. Combining with $I(Y; F) \ge I(Y;g(F))$ yields the desired claim, $I(Y;g(F)) = I(Y;F)$ when $g$ is injective. 
\end{proof}

\subsection{Mutual information and feature importance}
\label{sec: mutual_info_feature_importance}
Let $F=\{x_1, ..., x_p\}$ be a set of features used to predict $Y$. As shown in \citet{griffith2014quantifying}, the mutual information $I(Y;F)=I(Y; X_1, ..., X_p)$ can be visualized using a partial information (PI) diagram \citep{williams2010nonnegative}. We may interpret the mutual information shared between $Y$ and $F$ as a collection of non-negative pieces of information, whose sum forms $I(Y;F)$. Each of these pieces of information can be classified as unique, redundant, or synergistic (Figure \ref{fig:PI_diagram}). Unique information is the information about $Y$ that comes from only one feature and nowhere else. Redundant information is information about $Y$ that comes from a single feature, but which can also be found elsewhere in $F$. Synergistic information is information about $Y$ that cannot be extracted from a single feature, but is available when multiple features are considered. See \citet{dutta2020information, dutta2021fairness} for interesting connections between PID and fairness.

We note that the distinction between feature importance methods that seek to explain data versus methods that seek to explain or optimize a model comes from their treatment of redundant information \citep{wollstadt2021rigorous}. Methods for explaining data, such as MCI or UMFI, aim to count all of the redundant information pertaining to $X_i$ in $I(Y;F)$ towards the feature importance of $x_i$. Indeed, even though this information can be found elsewhere by a model, redundant information still constitutes part of the information that $X_i$ shares about $Y$ in the data. Conversely, a method mainly made for feature selection, such as conditional permutation importance (CPI), aims to count none of the redundant information towards the evaluation of a feature's importance, since this information is already found in another feature. 

\begin{figure}[h!]
\centering
{\includegraphics[width=0.65\textwidth]{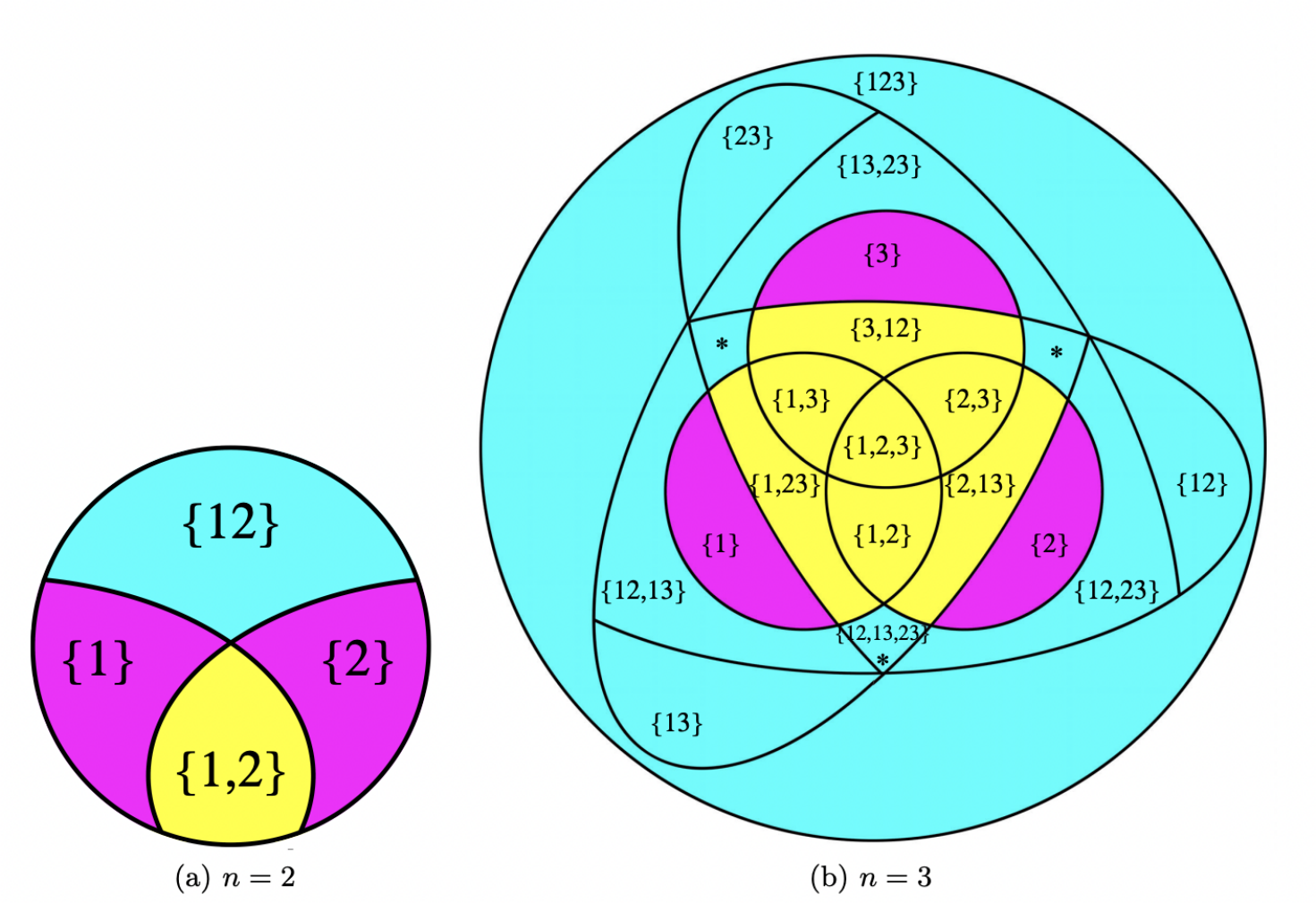}
 { \caption{PI-diagrams taken from \citet{griffith2014quantifying} for $I(Y;F)$ when $|F|=2$ (left) and $|F|=3$ (right). Magenta represents unique information, redundant information is colored with yellow, and synergistic information is in cyan. The starred regions represent a single region.}
  \label{fig:PI_diagram}}
}
\end{figure}

Mutual information itself is a common choice in the context of feature selection \citep{battiti1994using, al2003new, yang2012effective, bennasar2015feature}. However, due to the computational cost and the limited number of observations available for the calculation of the high-dimensional joint probability density function, it is not practical to compute $I(Y;S)$. For feature selection, users are only interested in the importance given to the top $k$ features. Therefore, 
mutual information-based feature selection methods typically bypass the computation of $I(Y;S)$ by instead studying the mutual information between the candidate feature and the response along with the mutual information between the candidate and the previously selected features \citep{bennasar2015feature, battiti1994using}. These methods are much less suitable for feature importance when the goal is to explain the data since interactions cannot be considered, which is why the most common approach is to train machine learning models to determine an approximation of the universal predictive power \citep{catav,chen2020true,covert2020understanding,williamson2020efficient}. 

Another connection between feature importance and mutual information comes from \citet{louppe2013understanding}, who showed that when extremely randomized trees' mean decrease in impurity (MDI) is used as a feature importance score, the MDI of a single feature converges to a weighted sum of conditional mutual information as the number of trees and the number of observations goes to infinity \citep{louppe2013understanding}. Also, the sum of the MDI scores across the feature set $F$ converges to $I(Y;F)$.

\subsection{Mutual information and machine learning evaluation functions}
\label{MI_ML_eval}

The evaluation function for a machine learning model $\nu_f(S)$ measures how well the response $Y$ can be predicted using the model $f$ given the information subset $S \in \mathcal{I}(F)$. Intuitively, the predictability or best possible accuracy $\nu_f(S)$ should ideally mirror or at least covary with the mutual information $I(Y;S)$ \citep{delsole2007predictability,gong2013estimating}. While entropy and mutual information are usually examined in discrete settings, the continuous entropy and mutual information are linearly related to the analogous discretized version \citep{gong2013estimating}. In the case of regression, one can also closely relate mutual information to the explained variance of a model. Indeed, with some assumptions, mutual information and $R^2$ accuracy are related. If we assume the response and predictions are jointly Gaussian and the predictions are unbiased \citep{Cover2006}, we can approximate the mutual information between $Y$ and $F$ as:  $$I(Y; F) \geq I(Y;g(F)) = I(Y;\hat{Y}) = -\frac{1}{2}\log[{1-\rho^2(Y,\hat{Y})}]=-\frac{1}{2}\log[{1-R^2}].$$ 


Machine learning evaluation functions and mutual information have been equated many times in the feature importance literature. \citet{covert2020understanding} demonstrated equivalence when the Bayes classifier is known and cross entropy loss is used. In a simple example, \citet{catav2020marginal} used mutual information directly as the evaluation function. The connection between machine learning evaluation functions and mutual information was further used by \citet{sutera2021global} to relate random forest feature importance with Shapely values.

\section{Additional information about marginal contribution feature importance (MCI)}
\label{sec: MCI}

Two of the methods that are compared with MCI in \citet{catav} include ablation and bivariate association. Ablation methods determine feature importance based on the difference in accuracy between the full model and the full model without the feature of interest, i.e. $A_{\nu}(x_i)=\nu(F) - \nu(F \setminus \{x_i\})$. Bivariate methods are among the most popular methods for genome-wide association studies \citep{wellcome2007genome,easton2007genome,sun2021revisiting}. In this method, the feature importance is given by the difference in the evaluation function of the model with just the feature of interest and the null model, i.e. $B_{\nu}(x_i)= \nu(x_i) - \nu(\emptyset)$. The three feature importance axioms proposed by \citet{catav} were partially motivated by the shortcomings of these two methods.

\begin{enumerate}
    \item \textbf{Marginal contribution}: Ablation methods may underestimate the importance of features when the correlation between features is high. In these scenarios, $\nu(F)$ may be approximately equal to $\nu(F \setminus \{x_i\})$ even in cases where $x_i$ is highly related to the response. Because of this, the importance of a feature $I_\nu(x_i)$ should be at least as large as the importance given by ablation methods: $I_\nu(x_i) \ge A_\nu(x_i)=\nu(F) - \nu(F \setminus \{x_i\})$ $\forall x_i \in F$.
    \item \textbf{Elimination}: Bivariate methods may underestimate the importance of features in cases where interactions exist between features. Many high-order interactions may be present in the data, so eliminating features from the feature set could prevent the detection of an important interaction. Thus, eliminating features from $F$ should only be able to decrease the feature importance of $x_i$.
    \item \textbf{Minimalism}: \citet{catav} decided to impose the minimalism axiom so that MCI can be unique. If $I_\nu(x_i)$ satisfies the first two axioms, then multiplying $I_\nu(x_i)$ by any constant $\lambda >1$ would not change this. The minimalism axiom helps disambiguate MCI from these trivial variations.
\end{enumerate}

We intentionally excluded some of the MCI axioms and properties included by \citet{catav} when proposing axioms for explaining data in Section \ref{sec:axioms}. Most importantly, the marginal contribution axiom is not included because it conflicts directly with the blood relation axiom. Indeed, ablation methods could give too much importance from a scientific inference perspective. For example, in the collider example presented by \citet{harel2022inherent}, they present the causal graph $Y \leftarrow S \rightarrow G \leftarrow E$, where $S$ is unmeasured. Let $F=\{E, G\}$ be used to predict $Y$. Then, the marginal contribution axiom requires that feature $E$ is given importance. Indeed, if we know $G$, then feature $E$ can help predict the response by denoising $G$ to recover information from the unobserved cause $S$. Thus, $I_\nu(E) \ge A_\nu(E)=\nu(\{E, G\})- \nu(\{G\})>0$. However, as stated in \citet{harel2022inherent}, feature $E$ has no relation to the response $Y$, since it can be thought of as a noise variable, so it would be more reasonable to give $E$ zero importance. We note that $E$ is given zero importance under the blood relation axiom, so the blood relation axiom is more reasonable and justified compared to the marginal contribution axiom. In contrast, $G$ inherently contains information about $Y$ via $S$, but this information is noised up by $E$. Therefore, although $E$ can be used to denoise $G$ and predict $Y$ better, only $G$ should be given importance when explaining the data when $F=\{E, G\}$, and indeed, $G$ is blood related to $Y$. We note that UMFI obeys the blood relation axiom under some assumptions, and hence does not obey the marginal contribution axiom. We additionally exclude the minimalism axiom since we do not prioritize uniqueness.

\section{Additional information about ultra-marginal feature importance (UMFI)}
\label{sec:UMFI_appendix}

\begin{theorem}[Existence of optimal preprocessing $\hat{S}^F_{x_i}$ when all features are jointly Gaussian]
\label{existence_optimal_preprocessing} Let $x_i \in F$ and suppose that all random variables in the random vector $F$ are joint normally distributed, then there exists a preprocessing $S^F_{x_i}$ that is optimal.
\end{theorem}

\begin{proof}
A preprocessing $S^F_{x_i}$ can be obtained via multiple linear regression (after mean centering) with the model:
$$
F \setminus \{x_i\} =  \beta x_i  +\epsilon,
$$
where $\epsilon = S^F_{x_i}$, $x_i$ is a feature in $F$, and $\beta$ is the column vector of size $p-1$ containing regression coefficients $\beta_1,\beta_2,...,\beta_{p-1}$ that minimize the sum of squared errors between $x_i$ and a linear function of each other variable in $F \setminus \{x_i\}$.

To show that $S^F_{x_i}$ is an optimal preprocessing (Definition \ref{def:remove_dep}), it suffices to show that $S^F_{x_i} \indep X_i$ and that $I(Y;F)=I(Y;S^F_{x_i},X_i)$, since $S^F_{x_i}$ is a function of $F$ by construction. 

From the normal equations and the definition of covariance, we know that $Cov(S^F_{x_i},X_i)=0$, as shown in the proof of Theorem \ref{Lin_optimal}. Since $S^F_{x_i} = F\setminus\{x_i\}-\beta x_i$, and all features in $F$ are joint normally distributed, it follows that $(S^F_{x_i},x_i)$ is joint normally distributed as well, since $(S^F_{x_i},x_i)$ can be obtained via the linear transformation $AF=(S^F_{x_i},x_i)$, where the main diagonal entries of $A$ are $1$, the other $|F|-1$ entries of the column corresponding to $x_i$ are given by the entries of $-\beta$, and all other entries are $0$. Without loss of generality, we may reorder the columns of the matrix such that the last column is attributed to feature $x_i$, and write
\begin{align*}
    A = \begin{bmatrix} 
    1 & 0 & \dots & \dots & -\beta_1 \\
    0 & 1 & 0 & \dots & -\beta_2\\
    \vdots & & \ddots\\
    0 & 0 & \dots & \dots & 1
\end{bmatrix} \hspace{1cm} A^{-1} = \begin{bmatrix} 
    1 & 0 & \dots & \dots & \beta_1 \\
    0 & 1 & 0 & \dots & \beta_2\\
    \vdots & & \ddots\\
    0 & 0 & \dots & \dots & 1
\end{bmatrix}.
\end{align*}
Hence, $Cov(X_i, S^F_{x_i})=0 \implies S^F_{x_i} \indep X_i$ from the properties of multivariate Gaussians.

To prove the second claim $I(Y;F)=I(Y;S^F_{x_i},X_i)$, by Theorem \ref{mut_info_bounds}, it suffices to show that the map $h(F)=(S^F_{x_i},x_i)=AF$ is injective. This is immediate from the fact that the matrix $A$, defined above, is invertible and thus bijective. 


\end{proof}

\begin{theorem}[Elimination axiom assuming optimal transport with chaining]
\label{theorem:elimination}
    Let $x_i \in F$, $x_{p+1} \not \in F$, and $\nu(S)$ is positively linearly related to $I(Y;S)$. When preprocessing is performed using optimal transport with chaining, $U^{F,Y}_{\nu}(x_i) \leq U^{F \cup \{x_{p+1}\},Y}_{\nu}(x_i)$.
\end{theorem}

\begin{proof}
Let $S^{F \cup \{x_{p+1}\}}_{x_i}$ be the preprocessed version of $F \cup \{x_{p+1}\}$ relative to $x_i$ and let $S^F_{x_i}$ be the preprocessed version of $F$ relative to $x_i$. By optimal transport with chaining \citep{johndrow2019algorithm}, we may assume that $S^{F \cup \{x_{p+1}\}}_{x_i}$ obeys the form $S^{F \cup \{x_{p+1}\}}_{x_i} = S^F_{x_i} \cup \tilde{x}$ and that $S^F_{x_i}, X_i, \tilde{X}$ are mutually independent. It follows from the supermodularity of mutual information under independence (Theorem \ref{supermodularity_theorem}) that

\begin{align*}
    &U^{F \cup \{x_{p+1}\},Y}_{\nu}(x_i)=aI(Y; S^{F \cup \{x_{p+1}\}}_{x_i}, X_i) +c - (aI(Y; S^{F \cup \{x_{p+1}\}}_{x_i})+c)\\
    &= aI(Y; S^F_{x_i}, \tilde{X}, X_i)+c - (aI(Y; S^F_{x_i}, \tilde{X})+c)\\
    &\ge aI(Y; S^F_{x_i}, X_i)+c - (aI(Y; S^F_{x_i})+c)= U^{F,Y}_{\nu}(x_i).
\end{align*}

\end{proof}

\begin{lemma}[Equivalence of optimal preprocessing equivalence classes under redundant information and duplicate features of interest]
 If $\hat{x} \in \mathcal{I}(F)$, then $[\hat{S}^F_{x_i}] \equiv [\hat{S}^{F \cup \{\hat{x}\}}_{x_i}]$. If $\hat{x} = h(x_j)$ and $h$ is bijective, then $[\hat{S}^{F \cup \{\hat{x}\}}_{x_j}] \equiv  [\hat{S}^{F \cup \{\hat{x}\}}_{\hat{x}}]$.
    \label{lemma:equivalence_class}
\end{lemma}

\begin{proof}
Recall that an optimal preprocessing given a feature set $F$ and a feature of interest $x_i$ is defined in Definition \ref{def:remove_dep}. To prove $[\hat{S}^F_{x_i}] \equiv [\hat{S}^{F \cup \{\hat{x}\}}_{x_i}]$, we prove that for any $x_i \in F$ 
all optimal preprocessings $\hat{S}^F_{x_i}$ are also optimal preprocessings in $[\hat{S}^{F \cup \{\hat{x}\}}_{x_i}]$, and that all optimal preprocessings $\hat{S}^{F \cup \{\hat{x}\}}_{x_i}$ are also optimal preprocessings in $[\hat{S}^F_{x_i}]$.


We first note that properties 1 and 2 in Definition \ref{def:remove_dep} are equivalent for $\hat{S}^F_{x_i}$ and $\hat{S}^{F \cup \{\hat{x}\}}_{x_i}$. For property 1, a function with repeated arguments can be defined to be equal to the same function without repeated arguments. For property 2, the feature of interest $x_i$ is consistent across both optimal preprocessings, so both preprocessings are independent of $X_i$. Lastly, since mutual information is invariant under duplicate information and since $\hat{S}^F_{x_i}$ and $\hat{S}^{F \cup \{\hat{x}\}}_{x_i}$ each satisfy their respective 3rd property, 
\begin{equation}
    I(Y;F,\hat{X})= I(Y;\hat{S}^{F \cup \{\hat{x}\}}_{x_i}, {X_i}) = I(Y;F) = I(Y;\hat{S}^F_{x_i}, {X_i}) .
    \label{interchange_info}
\end{equation}
Hence, the final property is also equivalent across both equivalence classes of optimal preprocessings, and we conclude $[\hat{S}^F_{x_i}] \equiv [\hat{S}^{F \cup \{\hat{x}\}}_{x_i}]$.

Similarly, to prove that $[\hat{S}^{F \cup \{\hat{x}\}}_{x_j}] \equiv [\hat{S}^{F \cup \{\hat{x}\}}_{\hat{x}}]$ if $\hat{x} = h(x_j)$ for some bijective function $h$, we note that the first property is equivalent for both optimal preprocessings, since they take place over the same feature set $F \cup \{\hat{x}\}$. Then, we note that $\hat{S}^{F \cup \{\hat{x}\}}_{x_j} \indep X_j \implies \hat{S}^{F \cup \{\hat{x}\}}_{x_j} \indep h(X_j) = \hat{X}$, and similarly, $\hat{S}^{F \cup \{\hat{x}\}}_{\hat{x}} \indep \hat{X} \implies \hat{S}^{F \cup \{\hat{x}\}}_{\hat{x}} \indep h^{-1}(\hat{X}) = X_j$. Finally, since mutual information is invariant under homeomorphic reparametrization of marginal variables \cite{kraskov2004estimating},
\begin{align*}
    &I(Y;F) = I(Y; \hat{S}^{F \cup \{\hat{x}\}}_{\hat{x}}, \hat{X}) = I(Y; \hat{S}^{F \cup \{\hat{x}\}}_{\hat{x}}, h^{-1}(\hat{X})) = I(Y; \hat{S}^{F \cup \{\hat{x}\}}_{\hat{x}}, X_j)\\
    &I(Y;F)=I(Y; \hat{S}^{F \cup \{\hat{x}\}}_{x_j}, X_j) = I(Y; \hat{S}^{F \cup \{\hat{x}\}}_{x_j}, h(X_j)) = I(Y; \hat{S}^{F \cup \{\hat{x}\}}_{x_j}, \hat{X})
\end{align*}
Since all 3 properties are equivalent for both optimal preprocessings, we conclude that $[\hat{S}^F_{x_i}] \equiv [\hat{S}^{F \cup \{\hat{x}\}}_{x_i}]$.
\end{proof}

We note that preprocessings $S^F_{x_i}$ and $S^{F \cup \{\hat{x}\}}_{x_i}$ may be interchangeable without being optimal, and that the interchangeability of these preprocessings is a sufficient condition for UMFI satisfying the redundant information invariance axiom, as long as $\nu(\cdot)= I(Y; \cdot)$. For example, interchangeability of preprocessings also holds when the removal of dependencies on a feature $x_i$ is done in a pairwise fashion (see Algorithm \ref{algo:S*_OT}), as well as when preprocessings is performed via optimal transport with chaining \citep{johndrow2019algorithm}. 

\begin{theorem}[Blood relation axiom assuming faithfulness]
\label{theorem:blood}
Let $x_i \in F$, $\nu(S)$ is positively linearly related to $I(Y;S)$, and suppose that the data is generated from a structural causal model $C$ with corresponding directed causal graph $G$ so that the entailed distribution is faithful to $G$. Assume also that there exists an structural causal model with graph $G'$ that contains the variables $X_i,\ S^F_{x_i}$, and $Y$, where the distribution of all its variables is faithful to the graph $G'$.
If the preprocessing $S^F_{x_i} \indep X_i$, then $U^{F,Y}_{\nu}(x_i)>0$ if and only if $X_i \in BR_G(Y)$.
\end{theorem}

\begin{proof}
As shown in the proof of Theorem \ref{theorem:umfi_master} in the main text, we may apply the definition of UMFI and $\nu$, the assumption $S^F_{x_i} \indep x_i$, and properties of mutual information and conditional independence to obtain

$$U^{F,Y}_{\nu}(x_i)=0 \iff aI(Y;X_i|S^F_{x_i}) + c - c=0 \iff X_i \indep Y | S^F_{x_i} \iff X_i \indep (Y,S^F_{x_i}) \implies X_i \indep Y$$

We note that the last implication comes from the contraction axiom \citep{dawid1979conditional}, and it can be strengthened to an equivalence as long as we can prove $X_i \indep (Y,S^F_{x_i})$ using the assumptions $X_i \indep Y$ and $X_i \indep S^F_{x_i}$. Since $X_i$, $Y$, and $S_{x_i}^F$ belong to an SCM with graph $G'$, and the entailed distribution of this SCM is faithful to $G'$, then $X_i \indep Y$ and $X_i \indep S^F_{x_i}$ imply that $X_i$ is $d$-separated from both $Y$ and $S^F_{x_i}$ by $\emptyset$. Therefore, $X_i\indep (Y,S^F_{x_i})$, and $U^{F,Y}_{\nu}(x_i)=0 \iff X_i \indep Y$.

Also, if the data is faithful to the causal graph $G$, then $X_i \indep Y$ is equivalent to $X_i \not \in BR_G(Y)$, which would conclude the proof of the blood relation axiom. We explicitly provide the details.

If $X_i \not \in BR_G(Y)$, then $X_i \indep Y$ follows from the global Markov property and the fact that $X_i$ and $Y$ are d-separated by the empty set. Indeed, every path from $X_i$ to $Y$ must have at least one collider. We consider two cases. (1) The edge coming out of $Y$ is outgoing. Then since $X_i$ is not a descendent of $Y$, the path must reverse its orientation at some vertex before meeting $X_i$. That vertex is a collider. (2) The edge connecting to $Y$ points towards $Y$. Then the path must reverse its orientation at some point since $X_i$ is not an ancestor of $Y$. The path must then reverse another time because otherwise, $X_i$ would share a common ancestor with $Y$ (the vertex of the first reversal). The vertex with the second reversal is a collider.

Conversely, let $X_i \in BR_G(Y)$. By the faithfulness assumption, it suffices to show that $X_i$ and $Y$ are d-connected by the empty set. Since $X_i \in BR_G(Y)$, there are two possible cases: either there is a directed path between $X_i$ and $Y$, or $X_i$ and $Y$ share a common ancestor. In the first case, we simply choose the directed path between $X_i$ and $Y$ and observe that there cannot be a collider. Similarly, in the second case, we may pick the path beginning at $Y$ and trace it up to the common ancestor and then travel to $X_i$. There can be no colliders along the path since every vertex has at least one outgoing edge by construction. Also, the empty set cannot contain any non-colliders.

\end{proof}






\begin{theorem}[Blood relation axiom in the absence of interactions]
Suppose that there is no synergistic information $I_{syn}(Y; S^F_{x_i}, X_i)$ about $Y$ between $X_i$ and $S^F_{x_i}$ for all $x_i \in F$, and that $S^F_{x_i} \indep X_i$. Then, if the graphical model obeys the global Markov property and faithfulness and $\nu(S)$ is positively linearly related to $I(Y;S)$, then $U_\nu^{F,Y}>0$ if and only if $X_i \in BR_G(Y)$.
\label{non_gaussian_blood_relation}
\end{theorem}
\begin{proof}
As in the proof of Theorem \ref{theorem:blood}, it suffices to show that $I(Y;X_i|S^F_{x_i})=0$ if and only if $X_i \not \in BR_G(Y)$. We may rewrite $I(Y;X_i|S^F_{x_i})=0$ as $I(Y; S^F_{x_i}, X_i)=I(Y;S^F_{x_i})$.

Though it is fairly controversial \citep{williams2010nonnegative,griffith2014quantifying}, some definitions of partial information decomposition imply that independent predictors cannot contain redundant information between them \citep{kolchinsky2022novel,harder2013bivariate}. Using partial information decomposition \citep{williams2010nonnegative}, and since $S^F_{x_i} \indep X_i \implies I_{red}(Y; S^F_{x_i}, X_i)=0$, we may decompose $I(Y;S^F_{x_i}, X_i)$ as
\begin{align*}
    I_{unq}(Y;X_i)+I_{unq}(Y;S^F_{x_i})+I_{syn}(Y; S^F_{x_i}, X_i).
\end{align*}
where we note that, because of the lack of redundancy, $I(Y;X_i)=I_{uniq}(Y;X_i)$ and that $I(Y;S^F_{x_i})$ captures the unique information that $S^F_{x_i}$ shares with $Y$ as well as synergistic information within the random vector $S^F_{x_i}$ that is shared with $Y$. As proven in Theorem \ref{theorem:blood}, $I(Y;X_i)=0$ if $X_i \not \in BR_G(Y)$ and $I(Y;X_i)>0$ if $X_i \in BR_G(Y)$ by the global Markov property and faithfulness. Since $I_{syn}(Y; S^F_{x_i}, X_i)=0$ by assumption, this gives us the desired statement $I(Y;S^F_{x_i}, X_i)=I(Y;S^F_{x_i})$ if and only if $X_i \not \in BR_G(Y)$. 
\end{proof}

\section{Additional information about other feature importance methods}

Historically, feature importance methods were developed in the pursuit of scientific questions, but current research in this area typically focuses on model explainability or model optimization. Early forms of feature importance assessed the strength of the relationships between variables within animal biology or human psychology using methods such as the correlation coefficient \citep{galton1889co}, Spearman's rank correlation coefficient \citep{spearman1961general}, multiple linear regression \citep{darlington1968multiple}, and partial correlation \citep{wright1921correlation}. Although these methods are perfectly interpretable, they are inadequate for modelling and therefore explaining complex data, since they cannot quantify the unknown interactions between multiple features. To counteract this severe limitation, Breiman was instrumental with his introduction of variable importance within classification and regression trees \citep{breiman2017classification}. At that time, Breiman seemed more concerned about the true strength of the relationships between the explanatory variables and the response, as he posited that a feature that is related to the response should be given some importance even if it does not appear in the final model \citep{breiman2017classification}. However, starting with Breiman's random forests, feature importance began to prioritize machine learning model explanation rather than data exploration. A good overview of the properties of some popular feature importance metrics is shown in \citet{covert2020understanding}.

\section{Preprocessing methods for removing dependencies}
\label{sec:RemovingDependencies}

Finding information preserving independent representations of our data is the central step of UMFI. These representations were first considered for AI fairness and privacy algorithms in order to give unbiased predictions in the face of sensitive attributes. For example, if one wants to remove the influence of race on recidivism likelihood predictions, preprocessing methods can be used to alter the original dataset such that the set of predictors are independent of race. In the following subsections, we discuss how optimal transport and linear regression can be used for finding these representations.

\subsection{Optimal transport}

Most of the results and methods explained in this section can be found in \citet{johndrow2019algorithm}. In this section, we denote features in the feature set $F$ by $X_j$ or $X_i$ to emphasize that they are random variables, rather than the previously used $x_j$ and $x_i$, where the former is used to denote observations $x_j$ sampled from $X_j$ instead. To obtain a preprocessing $S^F_{X_i}$, we may remove the dependencies of $x_i$ from each $X_j \in F \setminus \{X_i\}$ with minimal information loss with respect to $X_j$. To do so using optimal transport, we consider the Monge problem:

\begin{align*}
    g_c(X_j, \tilde{X}_j) = \inf_{g: g(X_j) \sim \tilde{X}_j} \mathbb{E} [c(X_j,g(X_j))] = \inf_{g: g(X_j) \sim \tilde{X}_j} \int_{\mathbb{R}} c(x_j,g(x_j)) d\mu(x_j).
    \tag{2.1.1} \label{monge}
\end{align*}

The quantity $g_c(X_j, \tilde{X}_j)$ represents the transportation cost of moving $X_j$ to $\tilde{X}_j$ with respect to some cost function $c$, and in our case, we desire $\tilde{X}_j \indep X_i$. It is natural to use $c(x_j, \tilde{x}_j)= d^q(x_j, \tilde{x}_j)$, where $d$ is the Euclidean norm. The transportation cost is also given by the Wasserstein-$q$ distance, $g_c(X_j, \tilde{X}_j)=\mathcal{W}_q^q(X_j, \tilde{X}_j)$, defined below for one-dimensional distributions. 
\begin{align*}
        \mathcal{W}_q(X_j, \tilde{X}_j)^q = \int_{0}^{1}|F^\leftarrow(p) - \tilde{F}^\leftarrow(p)|^q dp,
    \end{align*}
where $F_j$ and $\tilde{F}_j$ are the CDFs of $X_j$ and $\tilde{X}_j$, and $F_j^\leftarrow(p)=\sup_{x_j \in \mathbb{R}} F_j(x_j) \le p$. It can be shown that given any continuous one dimensional distributions $X_j$ and $\tilde{X}_j$, the optimal transport map $g: X_j \to \tilde{X}_j$ is given by $g=\tilde{F}_j^{\leftarrow} \circ F_j$.

\begin{theorem} \label{univariate_OT}
Let $X$ be a r.v. with density $f$ and CDF $F$. Let $\tilde{X}$ have CDF $\tilde{F}$. Then $g=\tilde{F}^{\leftarrow} \circ F$ is the minimizer to \eqref{monge}. Hence, $g$ optimally transports $X$ to $\tilde{X}=\tilde{F}^{\leftarrow}(F(X))$.
\end{theorem}
\begin{proof}
We show $\mathbb{E}[|X-g(X)|^q]=\int_{0}^{1}|F^\leftarrow(p) - \tilde{F}^\leftarrow(p)|^q dp$ for $g=\tilde{F}^{\leftarrow} \circ F$
\begin{align*}
    &\mathbb{E}[|X-g(X)|^q]=\int_{-\infty}^{\infty}|x-\tilde{F}^{\leftarrow}(F(x))|^q f(x)dx\\
    &= \int_{-\infty}^{\infty}|F^{\leftarrow}(F(x))-\tilde{F}^{\leftarrow}(F(x))|^q f(x)dx=\int_{0}^{1}|F^{\leftarrow}(p)-\tilde{F}^{\leftarrow}(p)|^q dp
\end{align*}
\end{proof}
\begin{theorem} \label{independent_OT}
Let $F_{j|x_i}(x)=P(X_j \le x_j | X_i=x_i)$ denote the CDF of $X_j|\{X_i=x_i\}$ . Then $g=\tilde{F}^\leftarrow \circ F_{j|x_i}$ optimally transports $X_j|\{X_i=x_i\}$ to $\tilde{X}_j \indep X_i$ for any CDF $\tilde{F}$
\end{theorem}
\begin{proof}
We apply Theorem \ref{univariate_OT} on the random variable $X_j| \{X_i=x_i\}$ and note that $X_j|\{X_i=x_i\}$ is independent of $X_i$. In particular, $g(X_j | X_i=x_i) \indep X_i$ for any choice of $\tilde{F}$.
\end{proof}
Theorem \ref{independent_OT} suggests an algorithm for transporting data $(x_{j1}, ..., x_{jn})$ sampled from $X_j$, to $(\tilde{x}_{j1}, ..., \tilde{x}_{jn}) \indep (x_{i1}, ..., x_{in})$. Since $x_{jk}$ is taken jointly with $x_{ik}$, as they are attributes coming from the $k$th sample in the dataset, then $x_{jk}$ is a realization of the distribution $X_j | \{X_i=x_{ik}\}$. Consequently, for each $k=1,...,n$, we should transport $x_{jk}$ to $\tilde{x}_{jk}=\tilde{F}^\leftarrow(F_{j|x_{ik}}(x_{jk}))$, where we may pick any CDF $\tilde{F}$ . This procedure can also adapted for features sampled from discrete r.v's, as shown in \citet{johndrow2019algorithm}.

\begin{algorithm}
\caption{Algorithm for removing dependencies of $X_i$ from $X_j$}\label{algo:remove_dependencies_OT}
\begin{algorithmic}
\REQUIRE{$X_j=[x_{j1}, ..., x_{jn}], X_i=[x_{i1}, ..., x_{in}]$, $X_j|(X_i=x_{ik}) \sim F_{j|x_{ik}}$, $\tilde{F}$ is a CDF}\\
\FOR{$k=1,...,n$}            
    \STATE $\tilde{x}_{jk} = \tilde{F}^\leftarrow(F_{j|x_{ik}}(x_{jk}))$
\ENDFOR
\RETURN $\tilde{X}_j=[\tilde{x}_{j1}, ..., \tilde{x}_{jn}]$
\end{algorithmic}
\end{algorithm}

We denote the result of the algorithm by $\tilde{X}_j=\tilde{F}^\leftarrow(F_{j|X_{i}}(X_j))$ and would ideally pick $\tilde{F}$ such that it minimizes the transportation cost $g_c(X_j, \tilde{X}_j)=g_c(X_j, \tilde{F}^\leftarrow(F_{j|X_{i}}(X_j)))$ across all CDFs $\tilde{F}$ in order to minimize information loss. However, in practice, the choice of $\tilde{F}$ does not matter much. In fact, as long as the support of $\tilde{F}$ is at least a large as the support of  $F_j$, the cdf of $X_j$, then any rank-based prediction rule, e.g. random forest, will be invariant to the choice of $\tilde{F}_j$ \citep{johndrow2019algorithm}. A standard choice for $\tilde{F}_j$ is $F_j$ so that we can recover the original quantiles of $X_j$.

Furthermore, $F_{j|x_{ik}}$ is not usually known and must be estimated from the data. For example, this can be done by splitting $X_i$ into $N$ quantiles and using the empirical CDF $P(X_j \le x_j | X_i \in x_{ik}\text{'s quantile} )$. The ability of this method to remove dependencies on $X_i$ from $X_j$ relies significantly on the accuracy of this estimate. 

We may iterate Algorithm \ref{algo:remove_dependencies_OT} over each feature in $F \setminus \{X_i\}$ to obtain pairwise independence between the transported variables $\tilde{X}_j$ and $X_i$. It is also possible to iterate Algorithm \ref{algo:remove_dependencies_OT} via chaining to achieve mutual independence between the transformed variables $\tilde{X}_j$ and $X_i$ \citep[Section 3.2]{johndrow2019algorithm}. However, this is computationally expensive, and pairwise independence should suffice for an accurate UMFI score, as will be explored further in Section \ref{sec:LR_vs_OT}. Step 2 of Algorithm \ref{algo:umfi} in the main paper can therefore be implemented with Algorithm \ref{algo:S*_OT}.

\begin{algorithm}[h!]
\caption{Algorithm for estimating $S^{F}_{X_i}$ via pairwise optimal transport}
\label{algo:S*_OT}
\begin{algorithmic}
\REQUIRE{$X_i=[x_{i1}, ..., x_{in}]$, $X_j=[x_{j1}, ..., x_{jn}]$ for $X_j$ in $F \setminus X_i$}
\STATE $S^{F}_{X_i} = \emptyset$
\FOR{$X_j$ in $F \setminus \{X_i\}$}            
    \STATE $\tilde{X}_j =$ output of Algorithm \ref{algo:remove_dependencies_OT} with $X_j$ and $X_i$
    \STATE add $\tilde{X}_j$ to $S^{F}_Z$
\ENDFOR
\RETURN $S^{F}_{X_i}$
\end{algorithmic}
\end{algorithm}

In other words, we may estimate $S^F_{X_i}$ as: 
$$S^{F}_{X_i}= \{F_{j}^\leftarrow(F_{j|X_i}(X_j)): X_j \in F \setminus \{X_i\} \}.$$

\subsection{Linear regression}
The most basic method for removing dependencies is linear regression. Even though it is quite simple, it can be shown to be optimal with a few assumptions (Theorem \ref{existence_optimal_preprocessing}). This preprocessing technique is implemented in the popular Python package \emph{fairlearn} \citep{bird2020fairlearn,scikitFairness}.

To reiterate, removing dependencies requires methods to make a feature or set of features $S$ independent of a protected attribute $x_i$, while keeping as much of the original information as possible. The overarching idea of linear regression as a preprocessing technique is that under the assumption that the residuals and the protected attribute are jointly Gaussian, the residuals can be utilized as a representation of $S$, which is independent of $x_i$.

\begin{theorem}
\label{Lin_optimal}
The residuals $\epsilon$ of a simple linear regression model have zero covariance with the predictor $X$.
\end{theorem}

\begin{proof}
(1) From the normal equations, the definition of covariance, and the fact that $\mathbb{E}[\epsilon]=0$, it follows that
\begin{align*}
    & Cov(X,\epsilon)=\mathbb{E}[X^T \epsilon] - \mathbb{E}[\epsilon]\mathbb{E}[X]= \mathbb{E}[ X^T \epsilon] = \mathbb{E}[X^T(Y-X \beta)]\\
    &=\mathbb{E}[X^T(Y- X(X^TX)^{-1}X^TY))]= \mathbb{E}[X^T Y- X^T X(X^TX)^{-1}X^TY]=\mathbb{E}[X^TY- X^T Y]=0
\end{align*}

\end{proof}

Thus, in step 2 of the algorithm for UMFI (Algorithm \ref{algo:umfi}), we can estimate
$$S^F_{X_i}= \{\epsilon_j=X_j - \beta_{0,j} -\beta_{1,j}X_i: X_j \in F \setminus \{X_i\}\}$$

where $\beta_{0,j}$ is the intercept term and $\beta_{1,j}$ is the slope term of the linear regression model $X_j=\beta_{0,j}+\beta_{1,j}X_i+\epsilon_j$.

\section{Experiments comparing linear regression and optimal transport}
\label{sec:LR_vs_OT}
In the following subsections, we compare the ability of linear regression and pairwise optimal transport to remove the information of a feature from data while distorting the original data as little as possible. It can be concluded that while linear regression works optimally when the data is jointly Gaussian, on real data, such as the BRCA dataset, pairwise optimal transport can find independent representations of the data, while linear regression fails (Section \ref{sec:removeDep}).

To implement UMFI paired with linear regression, we only remove dependencies when the regression slope coefficient is statistically significant (p-value $< 0.01$). To implement UMFI paired with pairwise optimal transport, when removing dependencies on the feature $x_i$ from the dataset, we estimate $F_{j|x_{ik}}$ by breaking up $x_i$ into quantiles of size $150$ and running linear regression on each quantile. The new independent predictors are then given by the values of the inverse empirical CDF of the residuals from the mentioned linear regression models.

\subsection{Removing dependencies}
\label{sec:removeDep}

It is crucial for our linear regression and optimal transport preprocessing methods to remove the information associated with the feature of interest, $x_i$, from the rest of the dataset $F \setminus \{x_i\}$. Therefore, we would like the preprocessed dataset $S^F_{x_i}$ to share zero mutual information with $x_i$. The mutual information $I(X_i; S^F_{x_i})$ is difficult to calculate, but it is closely related to the optimal predictor of $x_i$ given $S^F_{x_i}$ \citep{song2019learning}. For example, if $I(X_i; S^F_{x_i})=0$, as is desired, then the optimal predictor of $x_i$  will have zero accuracy when given $S^F_{x_i}$ as input. If the opposite is true and $S^F_{x_i}$ contains all of the information from $x_i$, then an optimal predictor of $x_i$ should be able to perfectly predict $x_i$ from the given information in $S^F_{x_i}$. In the following experiments, we assume that random forests can form the optimal predictor of $x_i$ given $S^F_{x_i}$. We use the OOB-$R^2$ coming from the random forest model as a approximate measure of the mutual information between $x_i$ and the transformed dataset $S^F_{x_i}$.

We used the BRCA dataset with 50 features to test the ability of optimal transport and linear regression to remove dependencies \citep{covert2020understanding,catav2020marginal}. All 50 features are continuous and the response is categorical. For each individual feature, we first use random forest OOB-$R^2$ to give a approximate measure of the mutual information $I(X_i; F \setminus \{x_i\})$ between the feature of interest $x_i$ and the other 49 features.  We then consider the case where the 49 remaining features are preprocessed to have dependencies on $x_i$ removed via linear regression or pairwise optimal transport. Similarly, random forest's OOB-$R^2$ is used to give a approximate measure of $I(X_i; S^F_{x_i})$.


The results are plotted in Figure \ref{fig:depRemove}. It is clear that the raw data (black line) shares considerable information across features. Most features can be predicted from the other untransformed features with an accuracy of $R^2>0.2$ and many can even be predicted with accuracies over $0.4$. Since the data has extremely nonlinear dependencies between features, simple linear regression is unable to remove all the mutual information between the protected attributes and the rest of the features. Indeed, the data certainly cannot be approximated with multivariate Gaussians. Conversely, pairwise optimal transport can successfully remove most of the mutual information present in the data. For all 50 features in the dataset, $x_i$ cannot be predicted successfully by random forest (OOB-$R^2 = 0$) from the other features after $F \setminus x_i$ is transformed with pairwise optimal transport.

\begin{figure}
  \centering
  \includegraphics[width=0.65\textwidth]{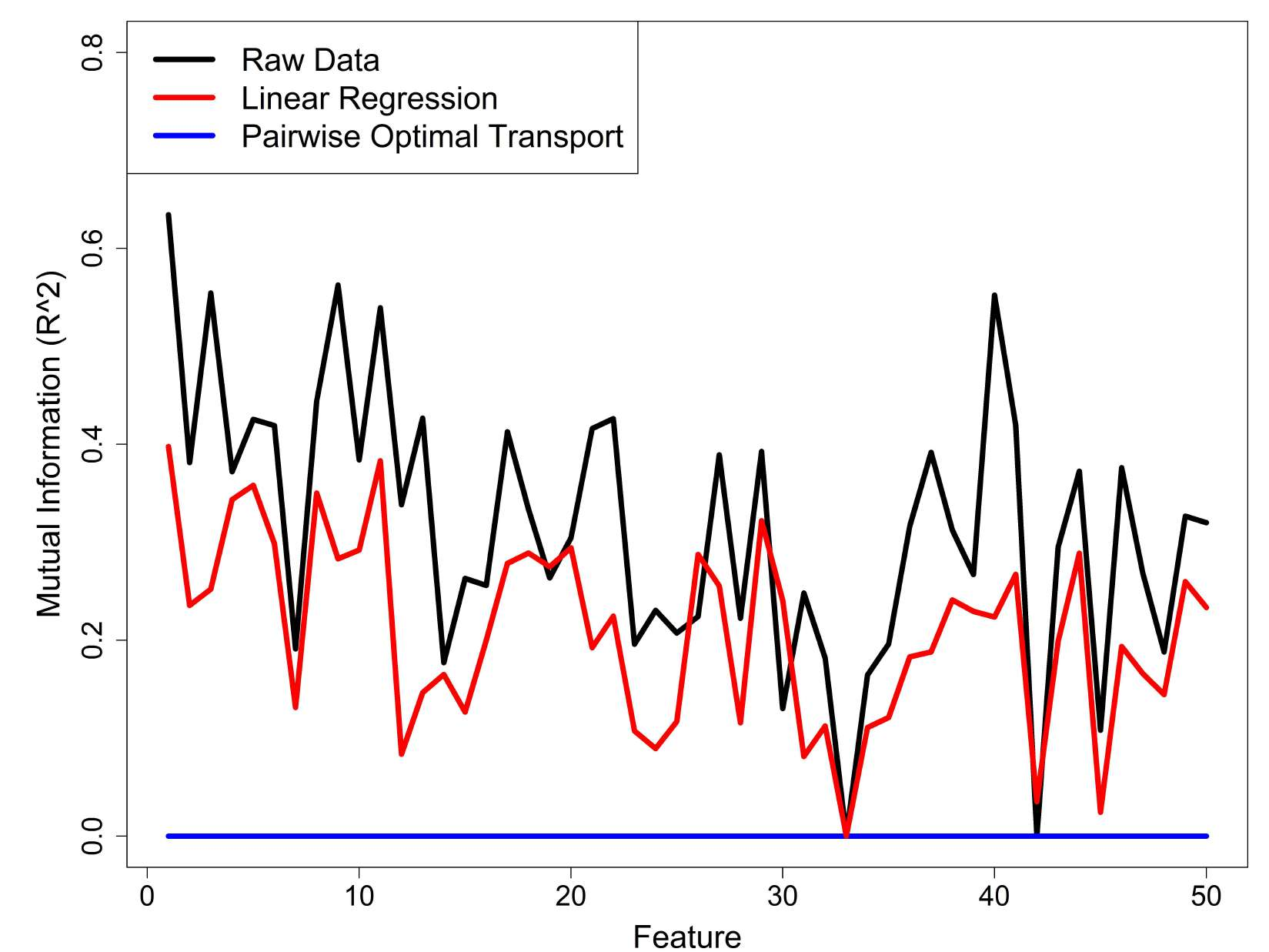}
  \caption{The coefficient of determination (random forest OOB-$R^2$) between the $i$th feature in the BRCA dataset and all other features is plotted (black) for each $ i \in \{1,2,...50\}$. The $R^2$ value between the $i$th feature and all other features after preprocessing with linear regression (red) and optimal transport (blue) is also plotted.}
  \label{fig:depRemove}
\end{figure}

\subsection{Distortion}
\label{sec:Distortion}
Not only do we require that the transformed features are independent of the feature of interest, but we also require that as much of the information present in the original data is preserved in the transformed data. To measure the amount of distortion imposed on the original data, we measure the dependence between the original and perturbed data using the maximal information coefficient \citep{kinney2014equitability}. For each feature in the BRCA dataset with 50 features \citep{covert2020understanding,catav2020marginal}, the information from the current feature is removed from all other features with either linear regression or pairwise optimal transport (Figure \ref{fig:distortion}).

Linear regression does not distort the transformed features in most cases. The dependence between the original and perturbed features usually remains near $1$, though the dependence does go as low as $0.42$ in one case (Figure \ref{fig:distortion}). While linear regression transformed these features with minimal distortion, these results are moot since linear regression failed to remove the original dependencies in a significant way, which was the main goal of the method (Figure \ref{fig:depRemove}).

\begin{figure}
  \centering
  \includegraphics[width=1\textwidth]{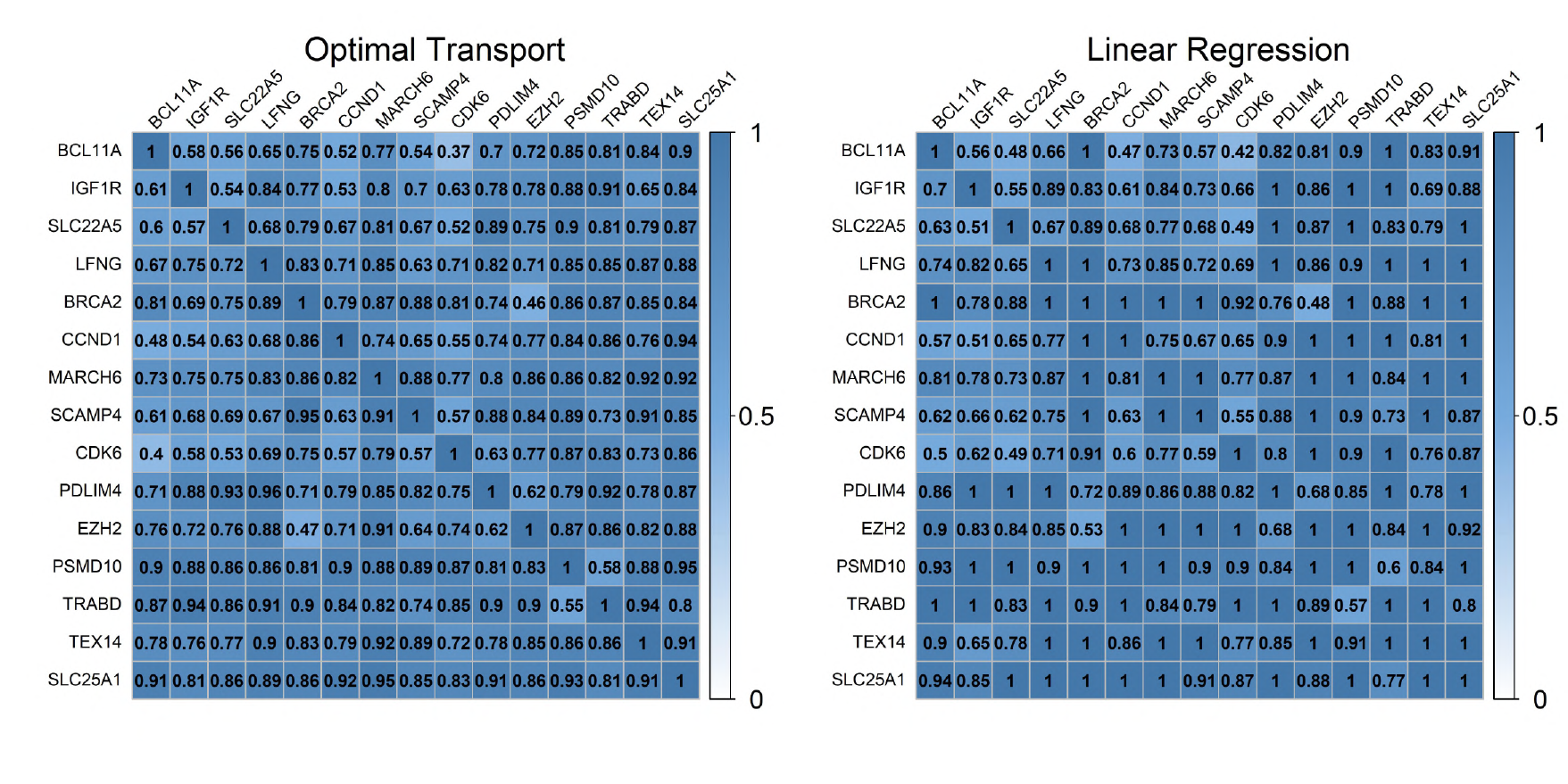}
  \caption{Cell $(i,j)$ indicates how similar the $j^{th}$ variable in the BRCA dataset is compared to its transformation via pairwise optimal transport or linear regression with respect to feature $i$. This is measured with the maximal information coefficient, which is comparable to $R^2$. To make the plots more clear and accessible, only the first 15 features are shown.}
  \label{fig:distortion}
\end{figure}

Compared to linear regression, pairwise optimal transport has a much more sizable effect on the distorted features, though this may have been necessary to completely remove dependence. The dependence between original and perturbed features mostly ranges from $0.6$-$0.9$, though some are as low as $0.37$ (Figure \ref{fig:distortion}). While only the first 15 features are shown, the results are similar for the other 35 features.

\section{Further feature importance experiments}
\label{sec:FurtherEXP_feature_imp}

This section is comprised of additional experiments performed on the simulated data introduced in Section \ref{sec:SimulatedData}, the BRCA dataset with permuted random genes, the original BRCA dataset with unpermuted random genes \citep{tomczak2015cancer,covert2020understanding,catav}, and the CAMELS hydrology dataset \citep{addor2017camels}. MCI and UMFI used either random forests or extremely randomized trees \citep{Breiman2001,geurts2006extremely}. Both of these, as well as ablation and permutation importance were implemented using the \emph{ranger} R package \citep{wright2015ranger}, while conditional permutation importance was implemented with the \emph{randomForest} and \emph{permimp} packages \citep{debeer2021package,liaw2002classification}. All experiments were run in Microsoft R Open Version 4.0.2. 

\subsection{Extra experiments on simulated data}
\label{sec: simulated_data1}

We repeat our previous experiments on simulated data from Section \ref{sec:SimulatedData} to test how ablation, permutation importance (PI), and conditional permutation importance (CPI) behave in the presence of nonlinear interactions (Section \ref{subsec: int1}), correlated interactions (Section \ref{subsec: corint1}), correlation (Section \ref{subsec: cor1}), and blood and non-blood related features (Section \ref{subsec:Blood_relation_exp1}). Further, we test how using extremely randomized trees instead of random forests for MCI and UMFI changes the results of the same simulation experiments. Although other methods such as XGBoost \citep{chen2015xgboost} could have been implemented for these experiments, XGBoost requires greater care when optimizing hyperparameters, so we chose to use extremely randomized trees instead, which is faster than random forests and provides similarly good predictions \citep{geurts2006extremely}. Both random forests and extremely randomized trees are not very sensitive to hyperparameter selection \citep{probst2019tunability}. For these simulation studies, we also perturb the size of the quantiles used by UMFI\_OT. We now use quantiles of size 30 instead of size 150. Quantiles of size 30 worked better on the hydrology data used in later experiments, so we test to see if the simulation results are sensitive to this choice in quantile size for dependency removal via optimal transport.

\subsubsection{Nonlinear interactions}
\label{subsec: int1}
The first experiment on simulated data handles the case where two variables, $x_1$ and $x_2$, interact in a nonlinear way in the response $Y$. As explained in Section \ref{sec:Interactions}, we should expect $x_1$ and $x_2$ to contribute more than half of the total importance, while $x_3$ and $x_4$ should be important, but less important compared to $x_1$ and $x_2$. Figure \ref{fig:int1} shows that ablation, PI, and CPI all provide accurate scores. 

When tested with extremely randomized trees, the nonlinear interactions simulation experiment results for MCI and UMFI, shown in Figure \ref{fig:int2}, remain mostly unchanged compared to the results from the experiment with random forests given in Figure \ref{fig:int}.

\subsubsection{Correlated interactions}
\label{subsec: corint1}
The second experiment considers the case where two correlated variables, $x_1$ and $x_2$, interact together in the response $Y$. Thus, as explained in Section \ref{sec:CorrelatedInteractions}, we should expect $x_1$ and $x_2$ to have more importance compared to $x_3$ and $x_4$. Figure \ref{fig:corInt1} shows that ablation, PI, and CPI all correctly weigh the importance of $x_1$ and $x_2$ as high relative to $x_3$ and $x_4$. The only notable difference is that the ablation method attributes an additional $\sim 3\%$ importance to each of $x_1$ and $x_2$ compared to PI, CPI, MCI, and UMFI (Figure \ref{fig:corInt1}).

When tested with extremely randomized trees instead of random forests, the correlated interaction simulation experiment results (Figure \ref{fig:corInt2}) for MCI and UMFI are similar to the earlier results shown in Figure \ref{fig:corInt}. MCI gave slightly more importance to $x_1$ and $x_2$ compared to $x_3$ and $x_4$, though the differences are seemingly insignificant. On the other hand, both UMFI methods gave significantly more importance to $x_1$ and $x_2$ compared to $x_3$ and $x_4$, as expected. 

\subsubsection{Correlation}
\label{subsec: cor1}
The third experiment tests how the metrics allocate importance to correlated features. As explained in Section \ref{sec:Correlations}, $x_1$ and $x_2$ should remain around the same relative importance, and $x_3=x_1 + \epsilon$, should have just slightly less importance compared to $x_1$ and $x_2$. Figure \ref{fig:cor1} indicates that CPI and ablation give near zero importance to the two heavily correlated features $x_1$ and $x_3$. This aligns with the discussion in Section \ref{sec: mutual_info_feature_importance} about methods motivated by feature selection since these methods base their scores on the importance of a feature conditioned on all other variables present in the model. Ablation performs similarly to CPI in this test, albeit with slightly less drastic results. Finally, we see that PI splits the importance detected from $x_1$ and $x_3$ proportionally across both features. This shows that PI can be viewed as a method for model explanation which in between the scientific inference and feature selection approaches. The scientific inference approaches (MCI and UMFI) allocate all of the redundant information to the feature. The feature selection approaches (CPI and ablation) allocate none of the redundant information to the feature. PI evenly splits the redundant information across the relevant correlated features.

When tested with extremely randomized trees, the correlation simulation experiment results (Figure \ref{fig:cor2}) for MCI and UMFI change slightly compared to the experiment with random forests in Figure \ref{fig:cor}. MCI works well, though it still gives some non-zero importance to $x_4$. With random forests, the relative importance of $x_4$ was usually above $5\%$, but with extremely randomized trees, the relative importance dropped below $5\%$. The performance of UMFI with linear regression got slightly worse as now the importance of $x_1$ is slightly greater than that of $x_2$ on average. The performance of UMFI with optimal transport changed for the better and now the importance of $x_1$ and $x_2$ are almost identical which was not true before. In this experiment, UMFI\_OT performed the best.

\subsubsection{Blood relation}
\label{subsec:Blood_relation_exp1}

For the last simulation experiment, we revisit the blood relation experiment performed in Section \ref{sec:Blood_relation_exp} using data generated from the causal graph in Figure \ref{fig:causalGraph_bloodrel}. The feature $S$ is unobserved, so the only blood related features to $Y$ in $F$ are $x_3$ and $x_4$. $x_3$ and $x_4$ should therefore be given high importance while $x_1$ and $x_2$ should receive zero importance. When tested on ablation, CPI, and PI, we notice that all three metrics fail to capture the desired importance, since they each give significant importance to $x_2$, which is not blood related to $Y$. We also note that this experiment provides an explicit example of UMFI not satisfying the marginal contribution axiom, which states that feature importance metrics should allocate at least as much importance as attributed by the ablation metric. Indeed, as shown in Figure \ref{fig:Blood}, UMFI gives around zero importance to non-blood related features $x_1$ and $x_2$, whereas ablation gives a significant portion of the importance to $x_2$.

When MCI and both implementations of UMFI were re-tested using extremely randomized trees instead of random forest, we observe that UMFI\_{LR} and UMFI\_{OT} both continue to give positive importance to the blood related features $x_3$ and $x_4$, while giving near-zero importance to the two remaining observed features (Figure \ref{fig:bloodrelation_ET}). However, we note that $x_3$ is given much more importance relative to $x_4$ when implemented with extremely randomized trees compared to random forests (Figure \ref{fig:Blood}). On the other hand, MCI gives positive importance to $x_2$ in this experiment. We note that MCI correctly gave $x_1$ almost zero importance while giving $x_3$ and $x_4$ significantly more importance compared to the random forest implementation. Across most simulation studies, it appears MCI performs slightly better using extremely randomized trees compared to random forests.

\begin{figure}[h!]
\label{fig:simulated_experiments_rf_othermethods}
\centering
\begin{subfigure}{0.35\linewidth}
  \includegraphics[width=\linewidth]{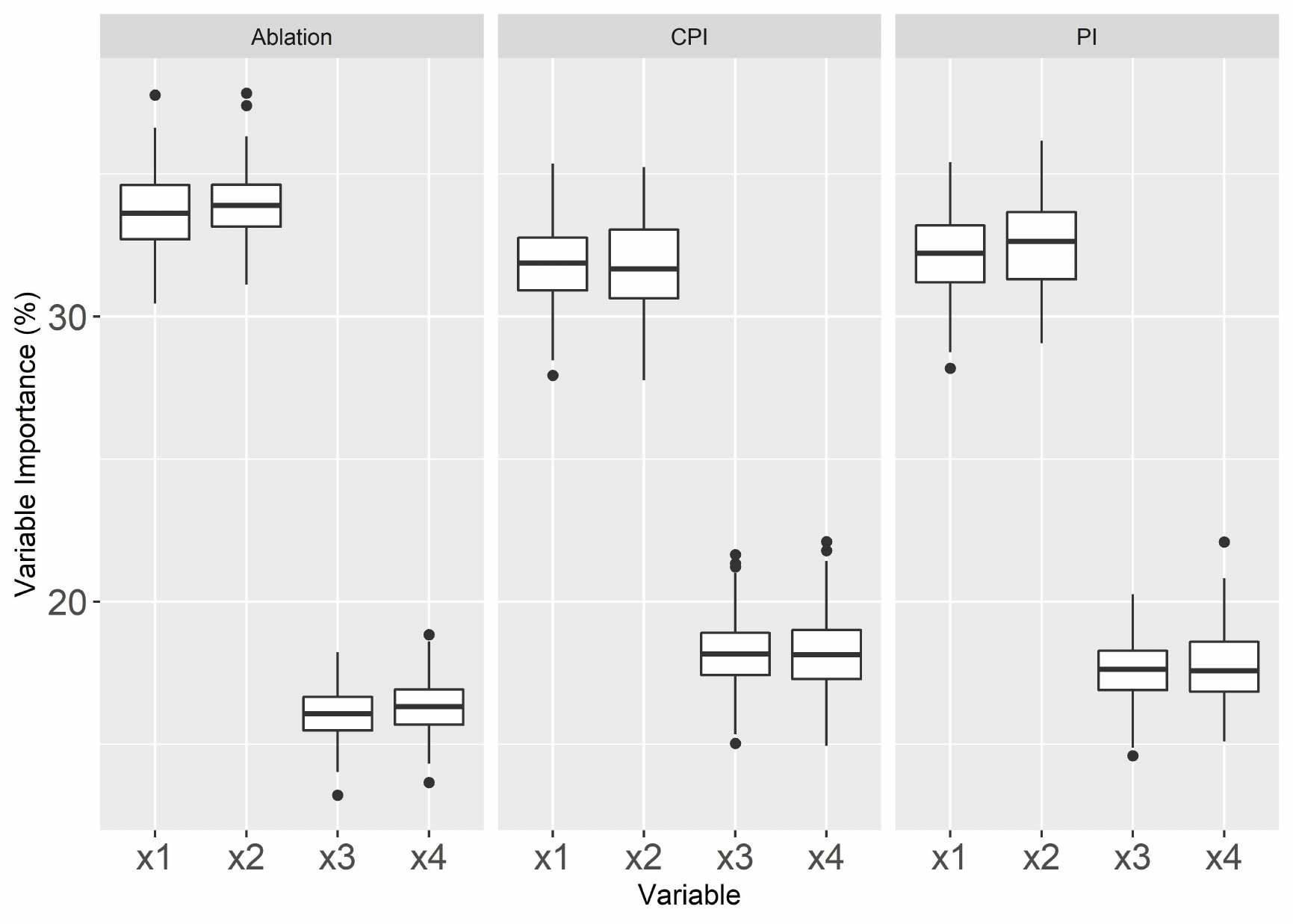}
  \caption{RF: Nonlinear interactions}
  \label{fig:int1}
\end{subfigure}
\begin{subfigure}{.35\linewidth}
  \includegraphics[width=\linewidth]{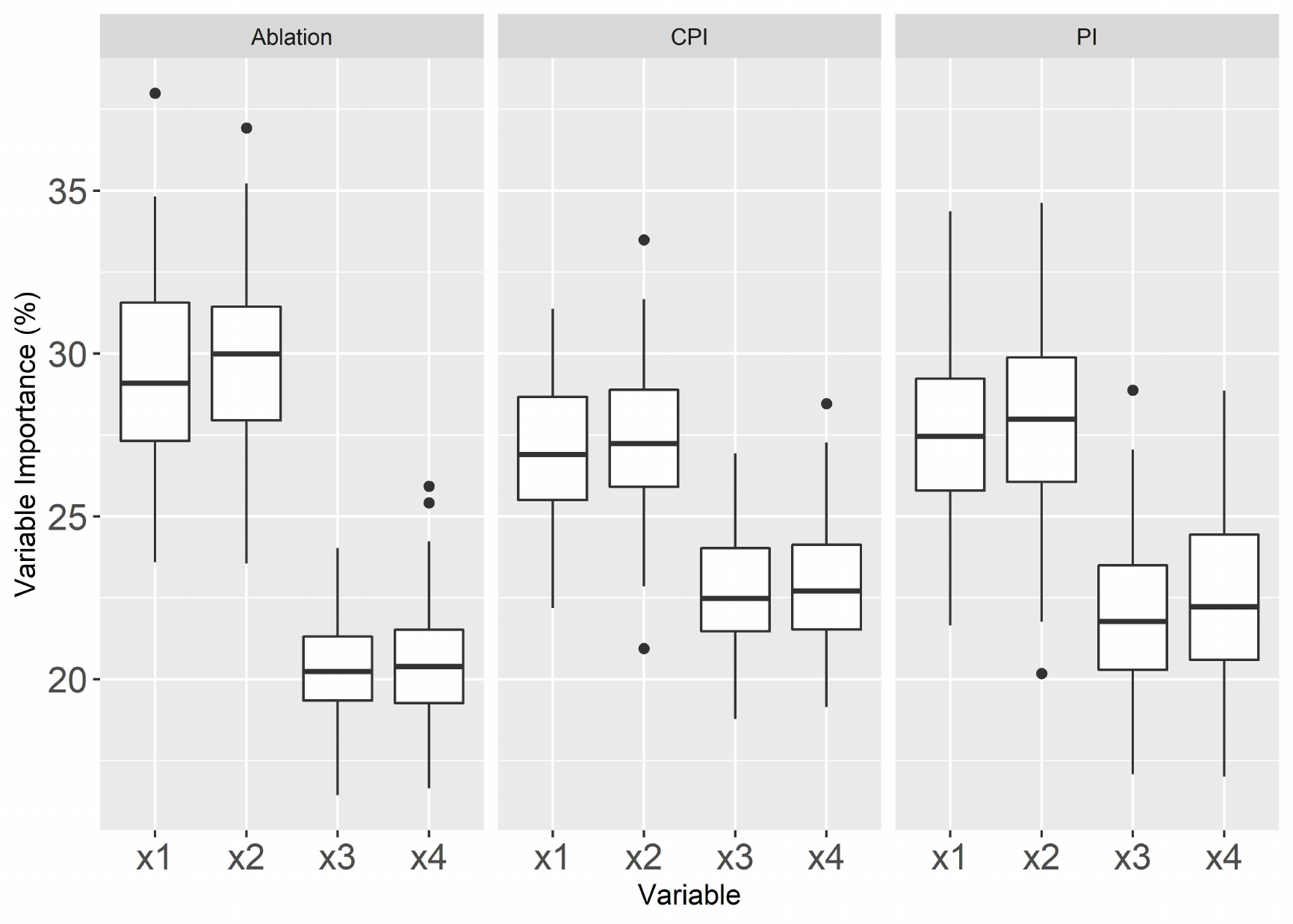}
  \caption{RF: Correlated interactions}
  \label{fig:corInt1}
\end{subfigure}\\[1ex]
\begin{subfigure}{.35\linewidth}
  \includegraphics[width=\linewidth]{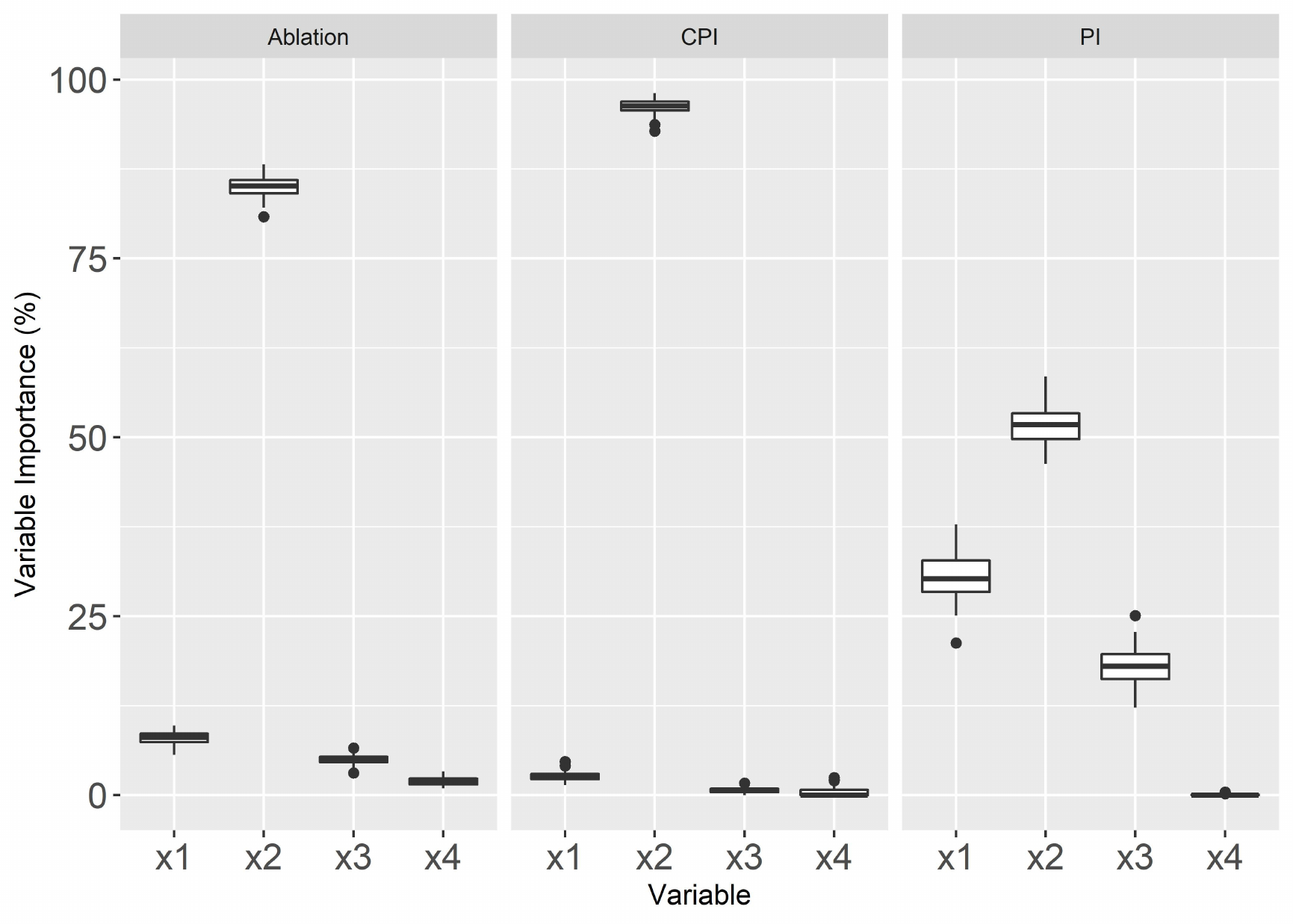}
\caption{RF: Correlation}
  \label{fig:cor1}
\end{subfigure}
\begin{subfigure}{.35\linewidth}
  \includegraphics[width=\linewidth]{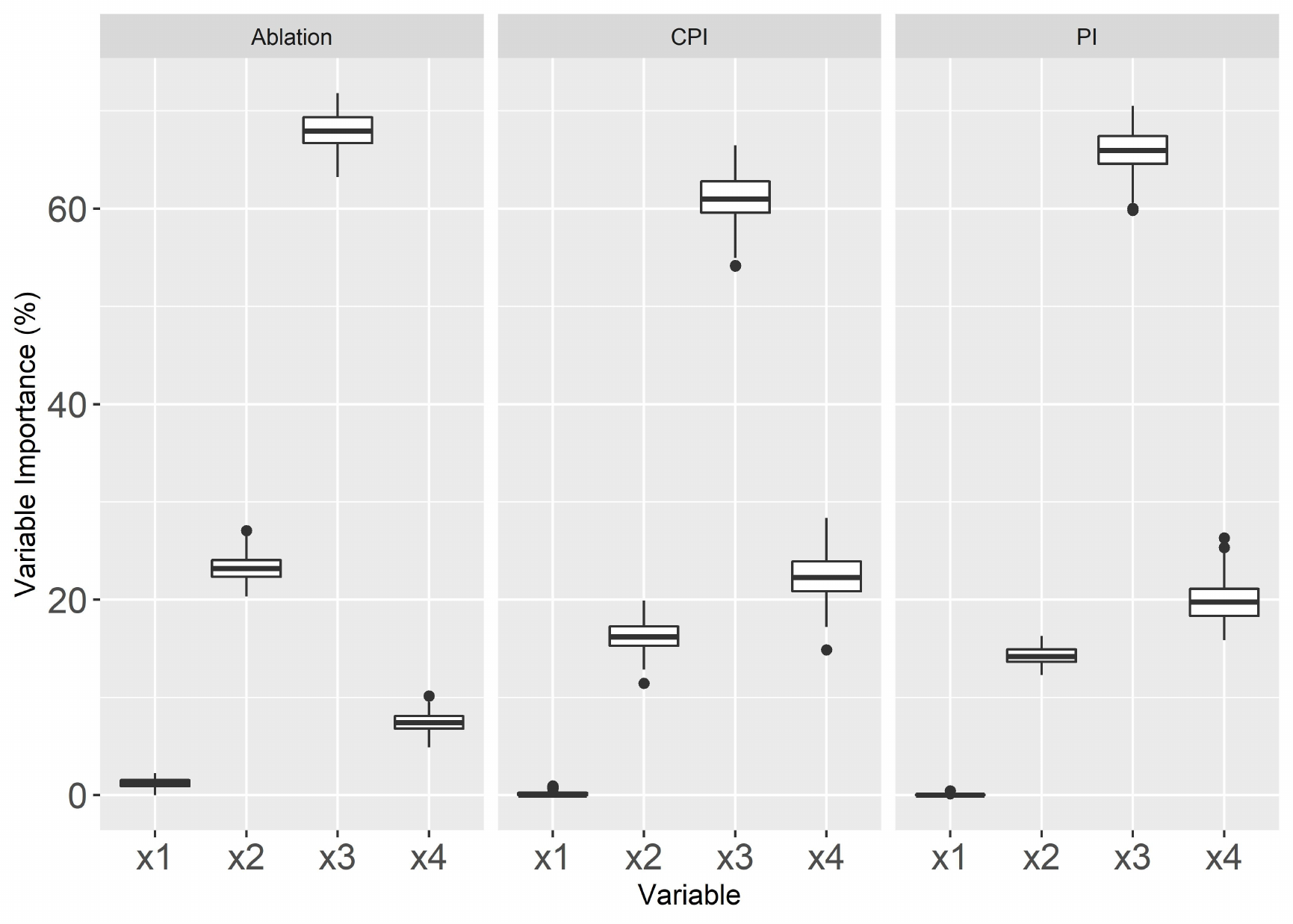}
\caption{RF: Blood relation}
  \label{fig:bloodrelation_othermeth}
\end{subfigure}\\[1ex]
\centering
\begin{subfigure}{0.35\linewidth}
  \includegraphics[width=\linewidth]{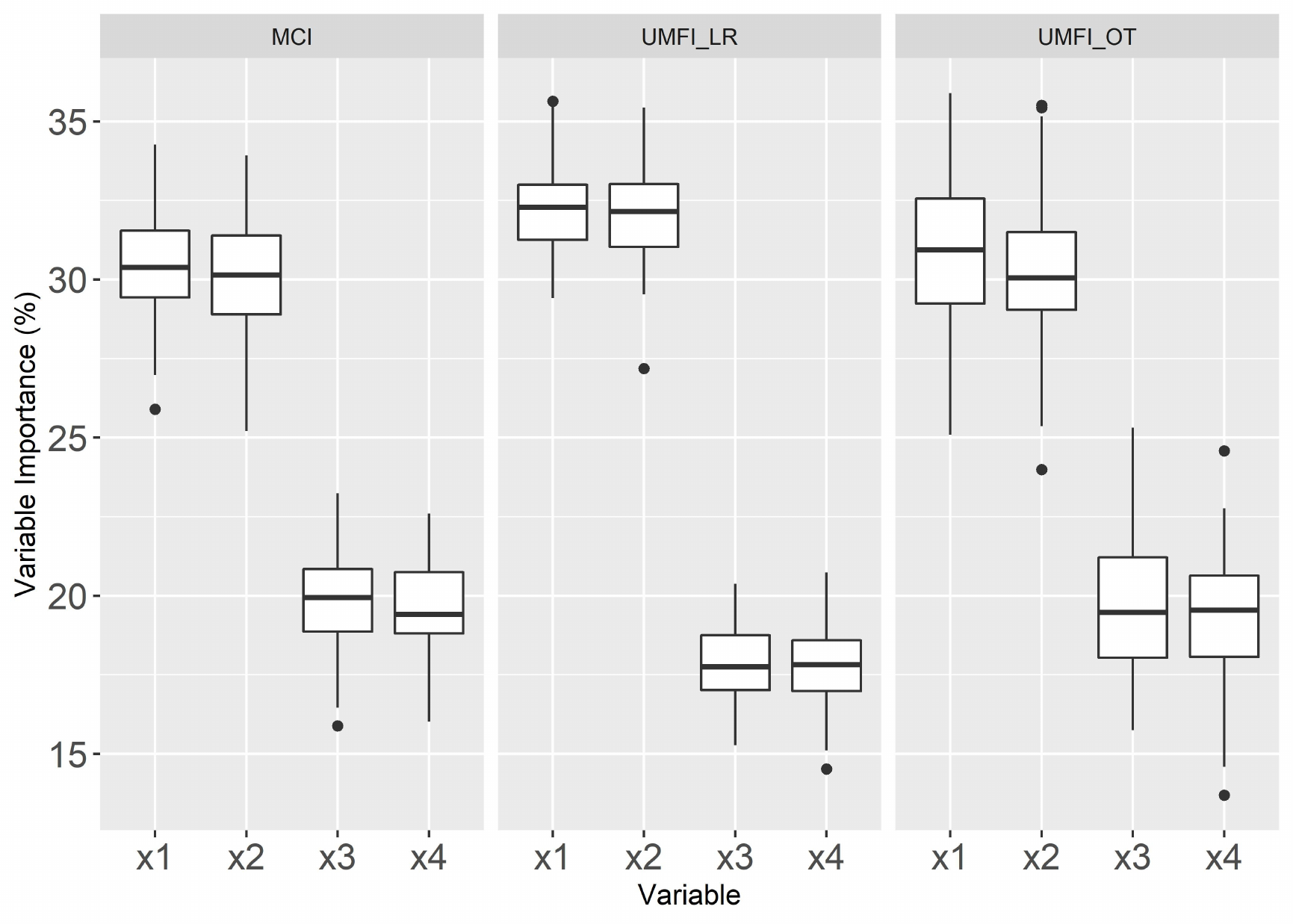}
  \caption{ET: Nonlinear interactions}
  \label{fig:int2}
\end{subfigure}
\begin{subfigure}{.35\linewidth}
  \includegraphics[width=\linewidth]{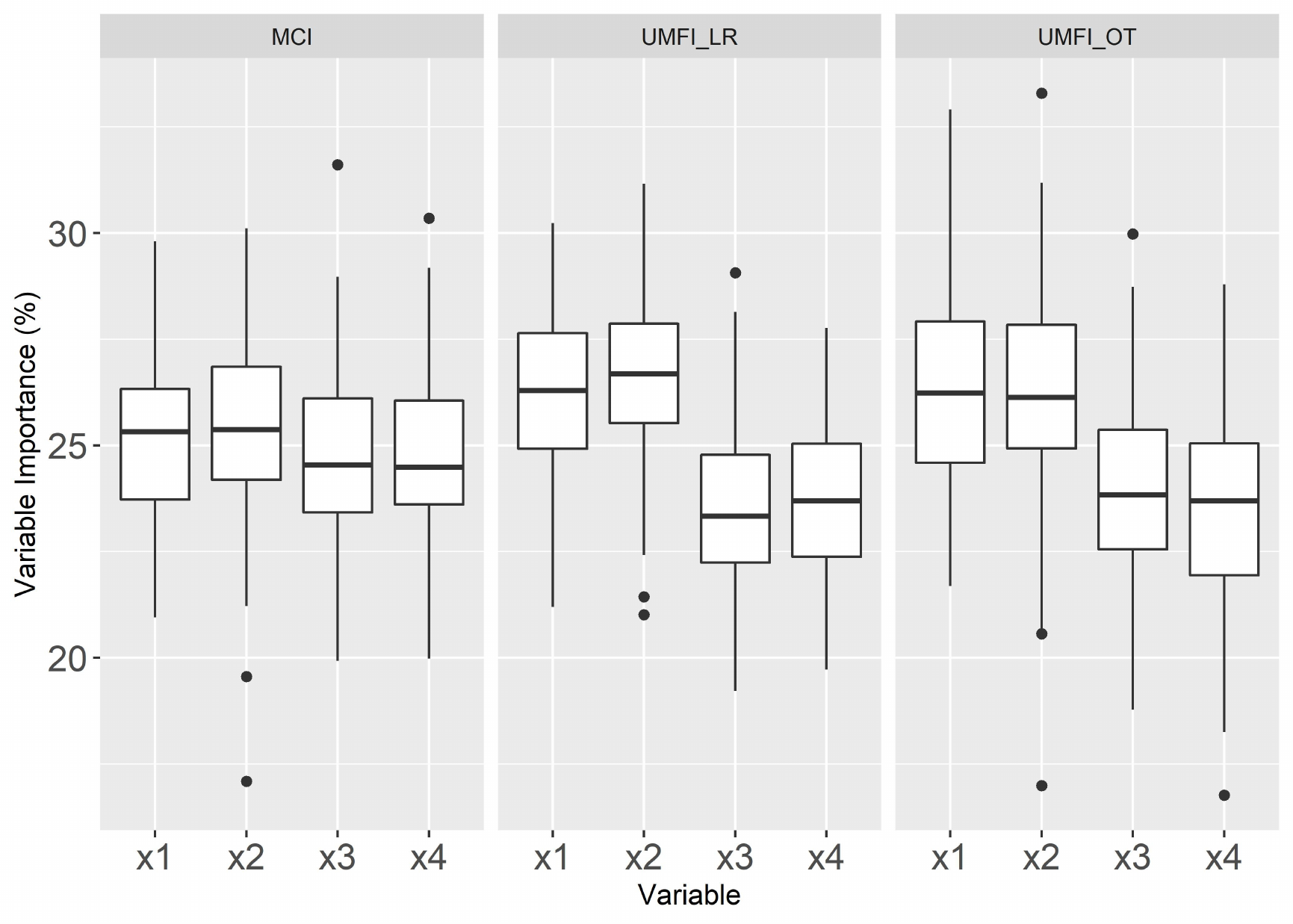}
  \caption{ET: Correlated interactions}
  \label{fig:corInt2}
\end{subfigure}\\[1ex]
\centering
\begin{subfigure}{.35\linewidth}
  \includegraphics[width=\linewidth]{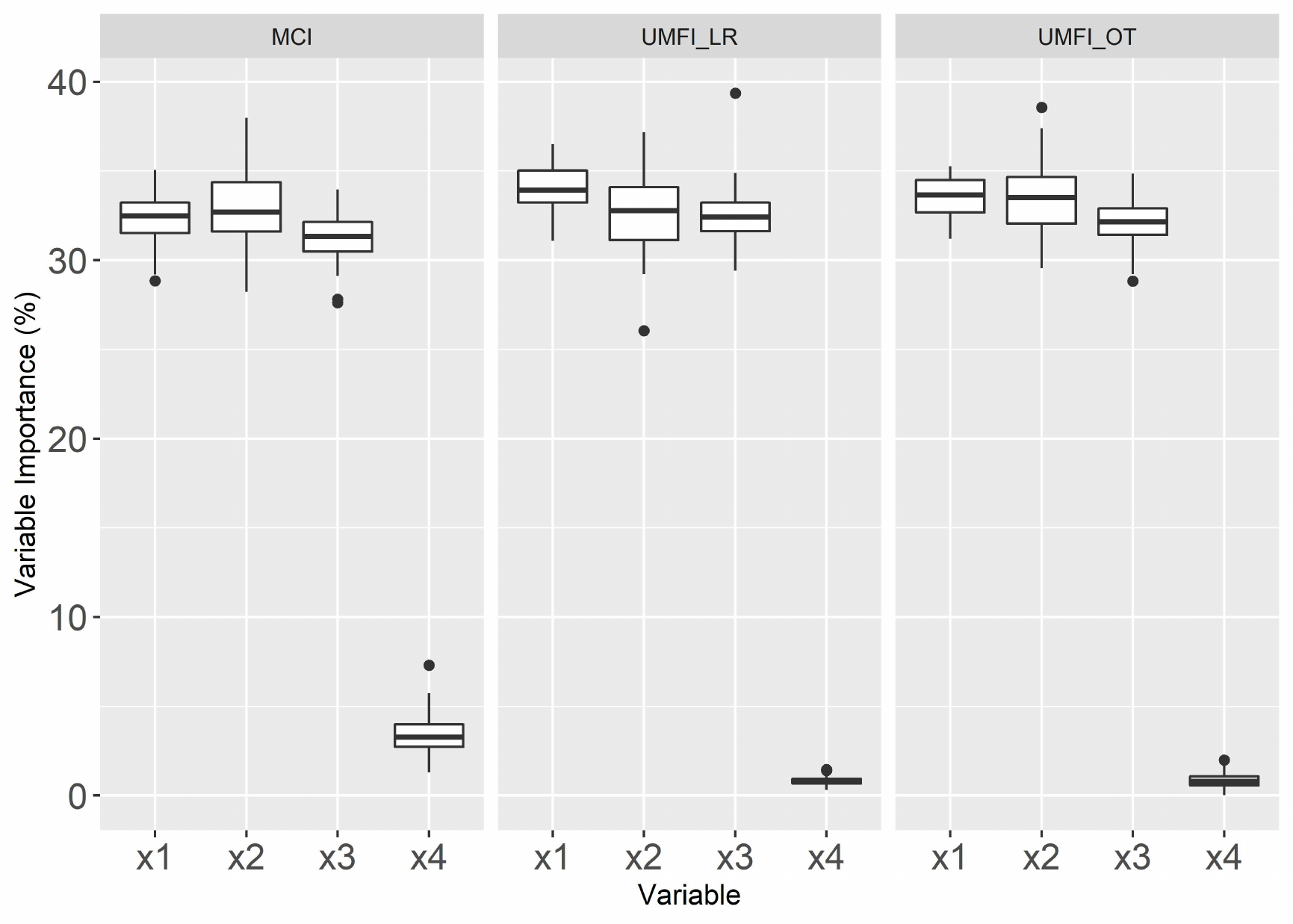}
  \caption{ET: Correlation}
  \label{fig:cor2}
\end{subfigure}
\begin{subfigure}{.35\linewidth}
  \includegraphics[width=\linewidth]{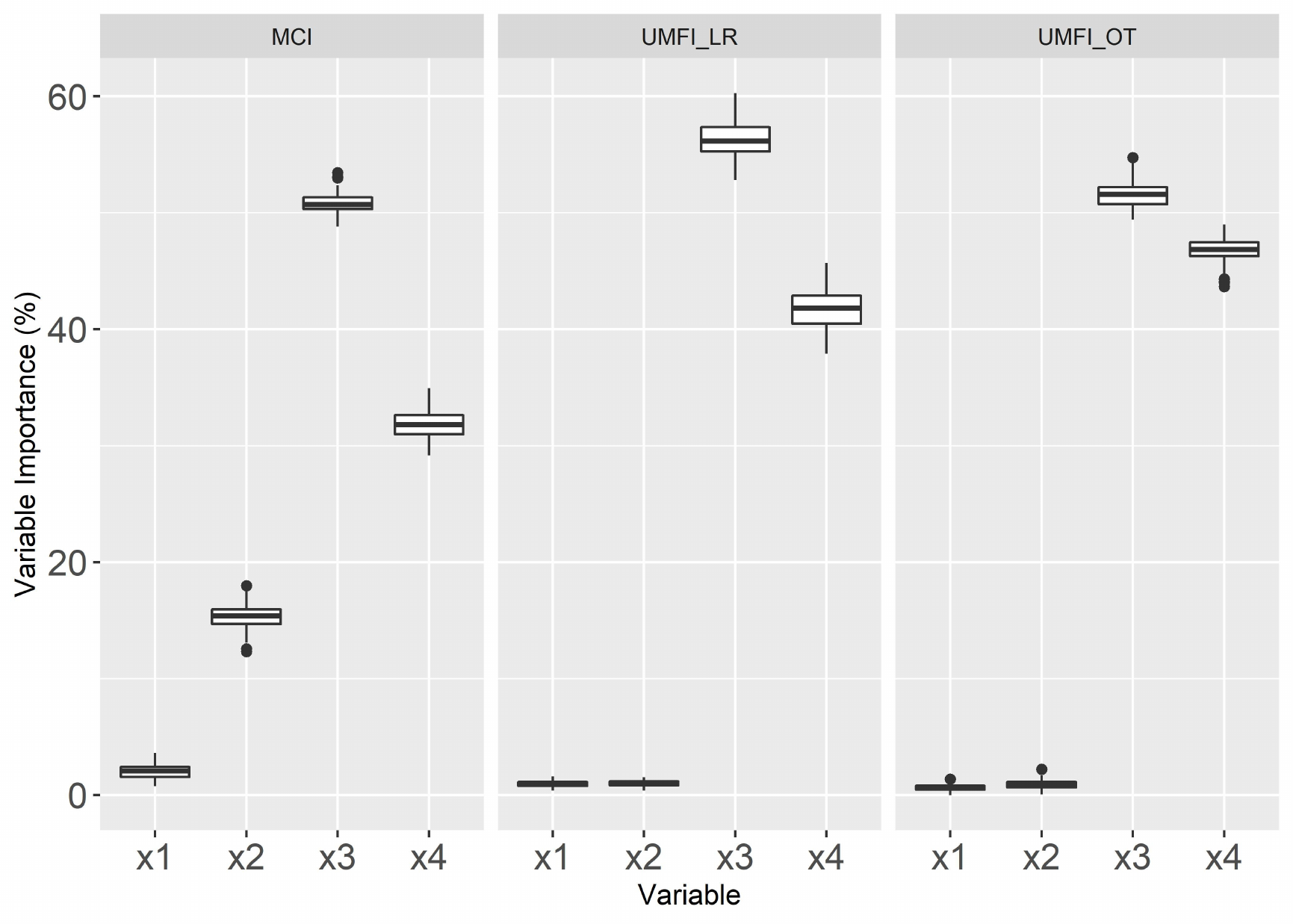}
  \caption{ET: Blood relation}
  \label{fig:bloodrelation_ET}
\end{subfigure}
\caption{Results for the experiments on simulated data from Subsection \ref{sec: simulated_data1}. The results for ablation, conditional permutation importance (CPI), and permutation importance (PI) were implemented with random forest (RF), and are shown in Figures \ref{fig:int1}, \ref{fig:corInt1}, \ref{fig:cor1}, and \ref{fig:bloodrelation_othermeth} . The results for MCI, UMFI\_{LR}, and UMFI\_{OT} were implemented with extremely randomized trees (ET), and are shown in Figures \ref{fig:int2}, \ref{fig:corInt2}, \ref{fig:cor2}, and \ref{fig:bloodrelation_ET}. Feature importance scores are shown as a percentage of the total for each of $x_1$ to $x_4$ from $100$ replications.}
\end{figure}

\subsection{Extra BRCA experiments with known ground-truth feature importance}

The following experiments are performed on the BRCA dataset with $571$ patients, each with one of four breast cancer subtypes, and $50$ continuous predictor genes. The experiments use the same setting as in Section \ref{sec:BRCA}, where the $40$ randomly chosen genes are also permuted so that the ground-truth feature importances are known. We observed that the overall classification accuracy of random forests for this dataset was $0.76$.

\subsubsection{Running 5000 iterations of UMFI}
\label{subsec:BRCA5000}
The original BRCA experiment conducted in Section \ref{sec:BRCA} showed that UMFI\_LR and UMFI\_OT performed impressively on real data, providing significantly more accurate feature importance scores than MCI after $200$ iterations of the experiment. Both UMFI\_LR and UMFI\_OT correctly gave high importance to the ten BRCA-associated genes, while giving zero median importance to about $80\%$ of the unassociated genes. Additionally, in an overnight study spanning less than ten hours, UMFI\_LR and UMFI\_OT displayed ideal results after running $5000$ iterations of the BRCA experiment. As shown in Figure \ref{fig:BRCA5000_fig}, both implementations of UMFI achieve $100$\% overall accuracy by giving high importance to the ten BRCA-associated genes and zero median feature importance to all $40$ unassociated genes. These results indicate that UMFI's relatively low computational cost can be leveraged via aggregation to achieve superior performance on complex data within a reasonable time budget. 

\begin{figure}[h!]
{\centering
\includegraphics[width=1\textwidth]{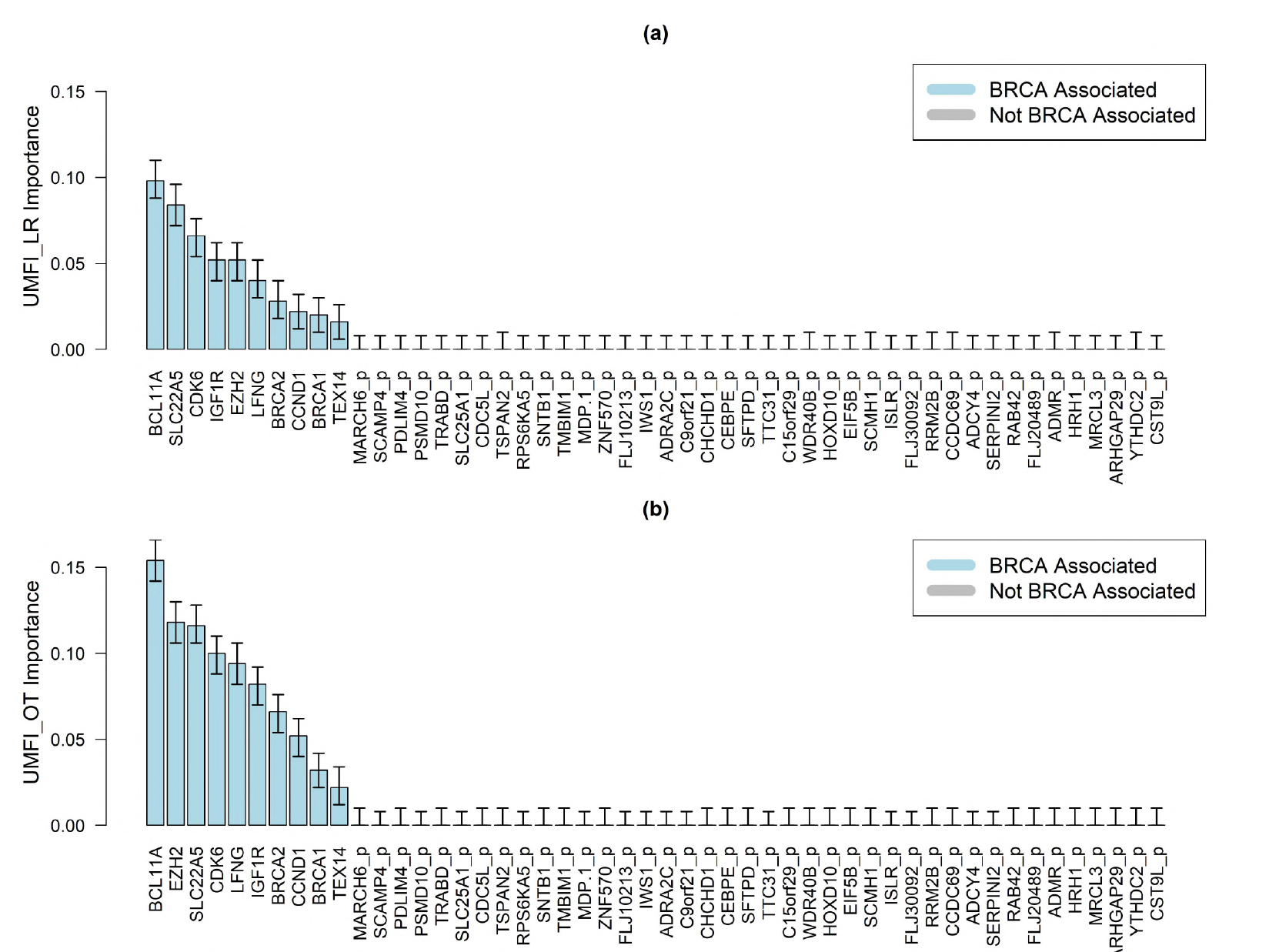}
 { \caption{Median feature importance scores provided by (a) UMFI with linear regression, and (b) UMFI with pairwise optimal transport, for each gene in the permuted BRCA dataset after $5000$ iterations. Genes colored in blue are known to be associated with breast cancer while genes colored in grey are random permutations of randomly selected genes, which we assume to be unassociated with breast cancer subtype. The first and third quantiles of the scores are visualized for each gene.}
  \label{fig:BRCA5000_fig}}
}
\end{figure}

\subsubsection{Ablation, PI, and CPI}
We also test the quality and robustness of other feature importance metrics including ablation, PI, and CPI, by running $200$ iterations of the BRCA experiment from Section \ref{sec:BRCA} for each method. Results are shown in Figure \ref{fig:brcarand_other_methods}. Ablation importance scores are small and have large uncertainties compared to its median importance scores, which makes the scores impractical to interpret. Eight of the ten important genes are identified by ablation, but all other genes are given exactly zero median importance. All ten important genes are given non-zero importance by CPI, however, some randomly permuted genes are given more importance than some genes known to be important, such as CDK6. PI gave more reliable and stable results compared to ablation and CPI in this experiment, exhibiting similar performance to UMFI\_{LR} and UMFI\_{OT} from the analogous experiment shown in Figure \ref{fig:BRCA}. We note that PI assigned zero importance to $29$ of the $40$ unassociated genes, making its TNR of $0.725$ slightly lower than UMFI in the analogous experiment from Section \ref{sec:BRCA}.

\begin{figure}[h!]
{\centering
\includegraphics[width=1\textwidth]{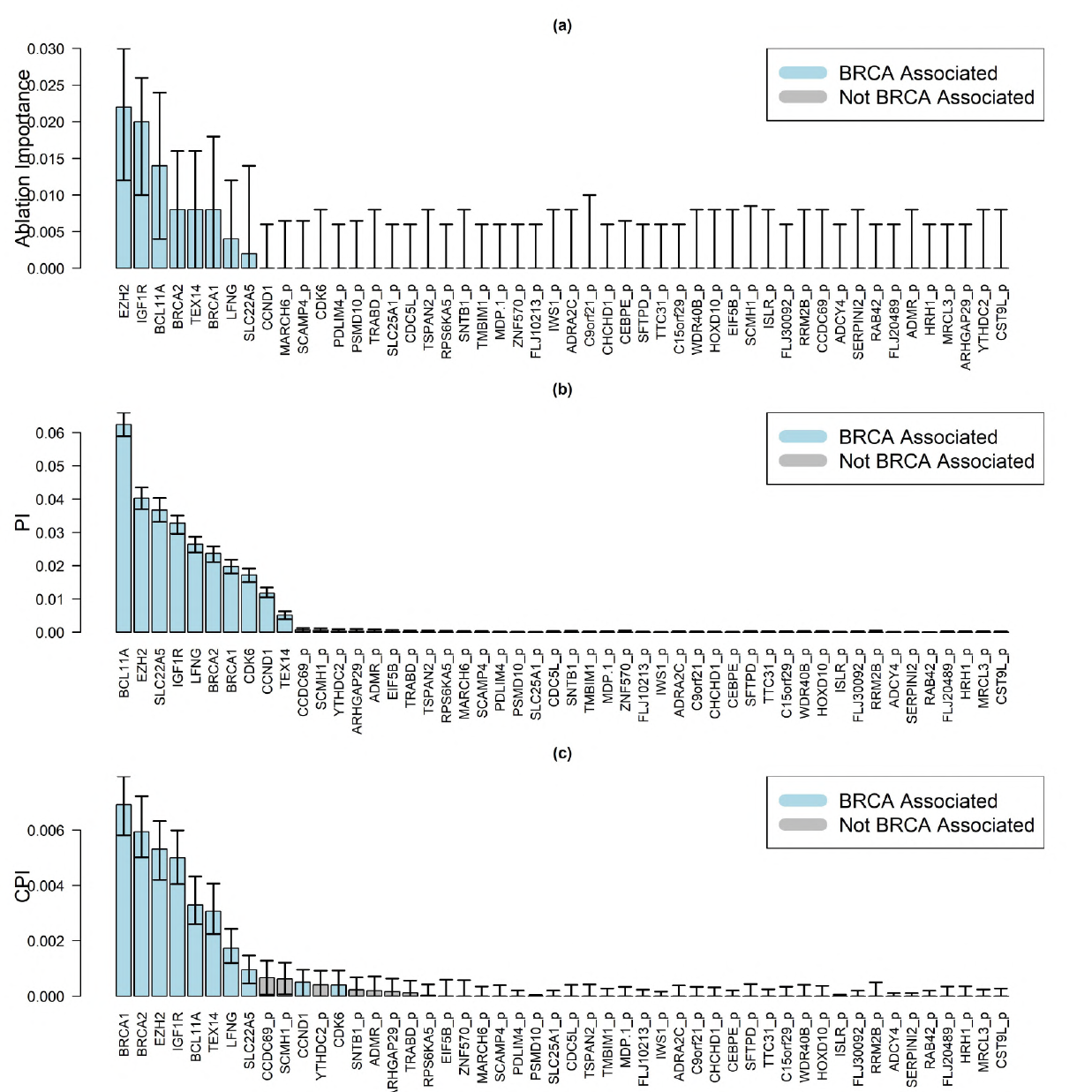}
 { \caption{Median feature importance scores provided by (a) ablation, (b) permutation importance, and (c) conditional permutation importance, for each gene in the permuted BRCA dataset after $200$ iterations. Genes colored in blue are known to be associated with breast cancer while genes colored in grey are random permutations of randomly selected genes, which we assume to be unassociated with breast cancer subtype. The first and third quantiles of the scores are visualized for each gene.}
  \label{fig:brcarand_other_methods}}
}
\end{figure}

\subsection{Experiments on unpermuted BRCA data}
Additional BRCA experiments were performed on the original unpermuted genes, as done in \citet{covert2020understanding} and \citet{catav}. The observed overall classification accuracy of random forests for this dataset was $0.79$.

Feature importance scores on this dataset were first computed with MCI, UMFI\_{LR}, and UMFI\_{OT} over $100$ iterations, as shown in Figure \ref{fig:mci_umfi_no_permutation}. The ordering of the BRCA associated genes is fairly similar across MCI and both UMFI methods. BCL11A and SLC22A5 are always the top two features and TEX14 is always the least important BRCA associated gene. While there are clear similarities in the results of all methods, the glaring difference is the number of features given zero importance. While MCI gives non-zero median importance to all $50$ features, $14$ features are given zero median importance by UMFI with linear regression, and $10$ features are given zero median importance by UMFI with pairwise optimal transport. It is unlikely that all $40$ randomly selected genes, which have not shown any association with breast cancer in previous studies, share information about breast cancer, so in this respect, we conclude that UMFI performs better than MCI.

\begin{figure}[h!]
{\centering
\includegraphics[width=\textwidth]{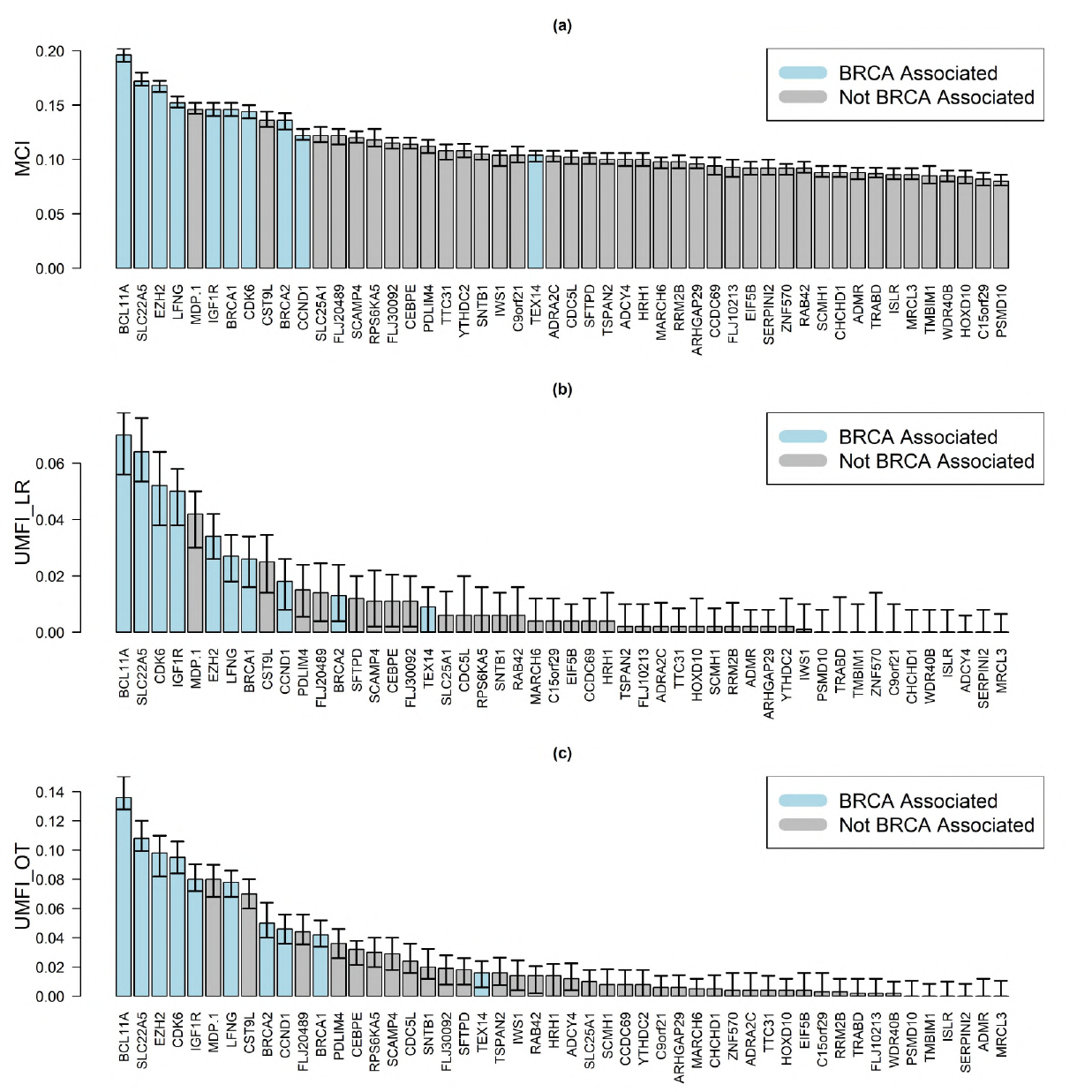}
 { \caption{Median feature importance scores provided by (a) MCI, (b) UMFI with linear regression, and (c) UMFI with pairwise optimal transport, for each gene in the unpermuted BRCA dataset after $100$ iterations. Genes colored in blue are associated with breast cancer while genes colored in grey are randomly selected genes. The first and third quantiles of the scores are visualized for each gene.}
  \label{fig:mci_umfi_no_permutation}}
}
\end{figure}
Feature importance scores on the unpermuted BRCA dataset were also computed with ablation, CPI, and PI over $100$ iterations, as shown in Figure \ref{fig:other_methods_no_permutation}. When also considering these results, we observe that MCI, UMFI, and PI give similar importance scores, while ablation and CPI performed significantly worse. Once again, ablation's high relative variance hampers its interpretability. Meanwhile, CPI gave by far the highest importance to SLC25A1, which is not known to have any association with breast cancer. In the results of MCI, UMFI, and PI, BCL11A is the most important while CST9L is always among the most important non-BRCA associated genes. Contrary to this, ablation and CPI give high importance to BRCA1, BRCA2, TEX14, EZH2, and IGF1R for BRCA associated genes, and SLC25A1 for non-BRCA associated genes.

\begin{figure}[h!]
{\centering
\includegraphics[width=\textwidth]{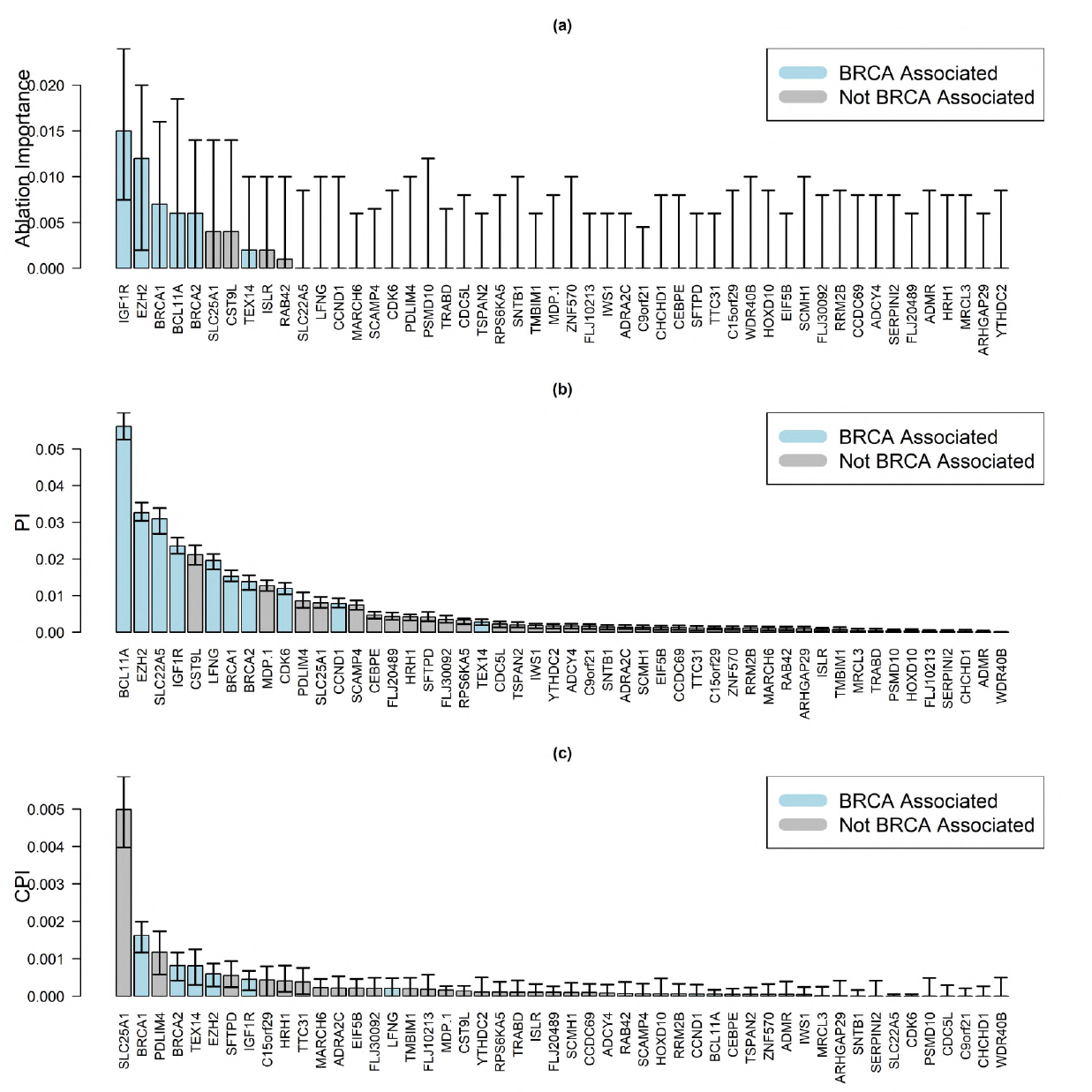}
 { \caption{Median feature importance scores provided by (a) ablation, (b) permutation importance, and (c) conditional permutation importance, for each gene in the unpermuted BRCA dataset after $100$ iterations. Genes colored in blue are associated with breast cancer while genes colored in grey are randomly selected genes. The first and third quantiles of the scores are visualized for each gene.}
 \label{fig:other_methods_no_permutation}}
 }
\end{figure}

\subsubsection{Computational complexity}
\label{sec:time1}

We compare the computational complexity of UMFI and MCI against the other feature importance methods that were explored in this section: ablation, PI, and CPI. To do so, we ran $10$ iterations of the BRCA experiment, which has $50$ features, each with $571$ observations. We recorded the average time for each method to compute feature importance for $5, 10, 15, 20, 25, 30, 35, 40, 45,$ and $50$ features. Figure \ref{fig:time_other_methods} shows that PI is the fastest method, processing $50$ features in $50$ milliseconds on average, followed by ablation ($50$ features in $1.8$ seconds), UMFI ($50$ features in $3$ seconds when parallelized), CPI ($50$ features in $30$ seconds), and finally MCI with soft 2-size submodularity ($50$ features in $205$ seconds).

\begin{figure}
\centering
{\includegraphics[width=0.6\textwidth]{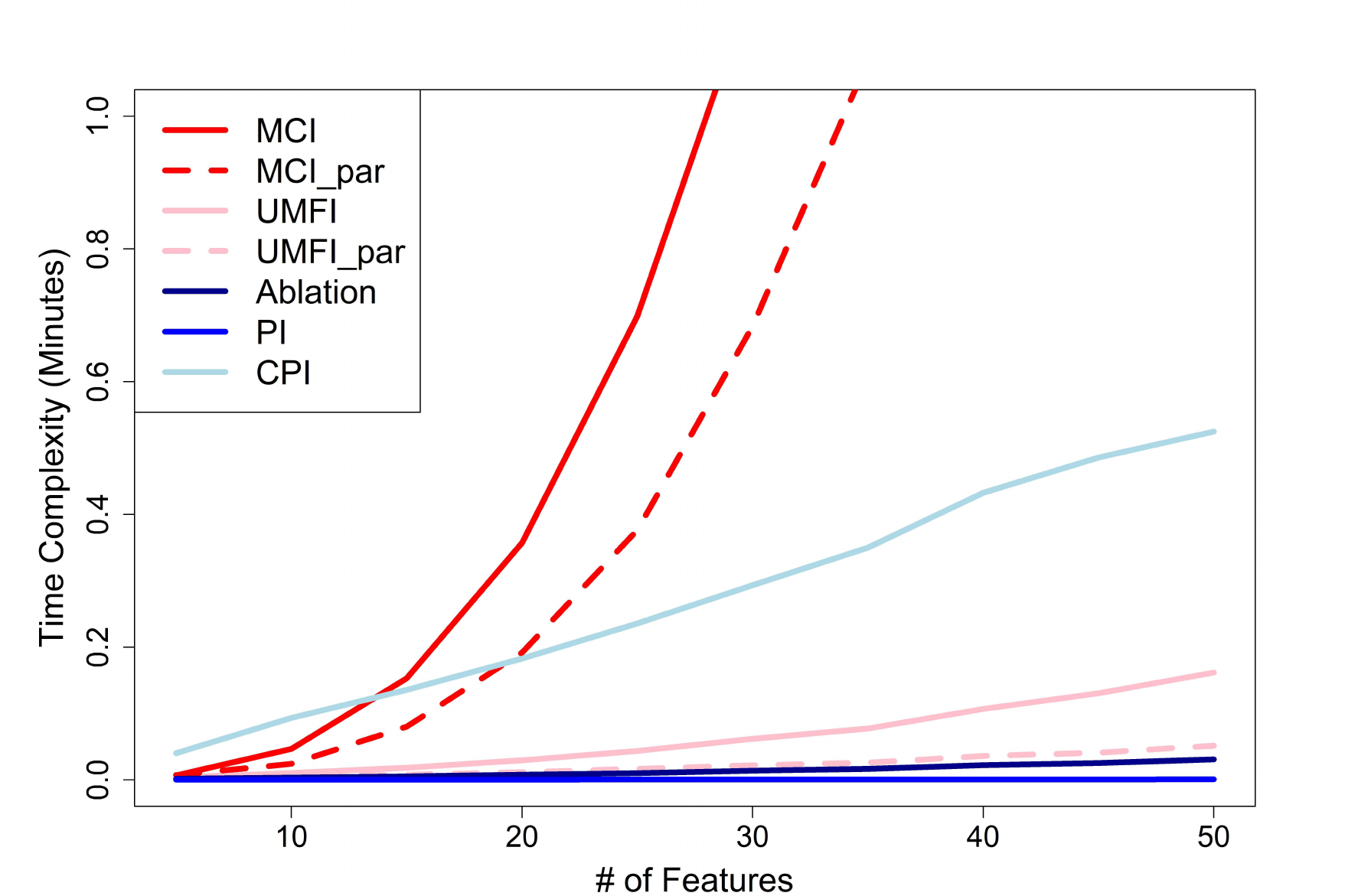}
 { \caption{The average computation time for each method to process $p$ features over $10$ iterations of the original BRCA data is plotted for each $p \in \{5, 10, 15, 20, 25, 30, 35, 40, 45, 50\}$. We assume soft 2-size submodularity to run the MCI results.}
  \label{fig:time_other_methods}}
}
\end{figure}

\subsection{Experiments on hydrology data}
\label{sec:hydrology_exp}
The final experiments for this study were conducted on a large-sample hydrology dataset called CAMELS \citep{addor2017camels}. This dataset records catchment averaged climate, soil, geology, topography, and land cover characteristics for $643$ catchments across the contiguous United States. With these, there are $29$ continuous explanatory variables. The response variable is averaged yearly streamflow, which is also continuous. Extremely randomized trees were used in this experiment with an overall OOB-$R^2$ of $0.91$.

Figure \ref{fig:hydrology_mutual_info}, which is analogous to Figure \ref{fig:depRemove} in Appendix \ref{sec:LR_vs_OT}, shows that both preprocessing methods fail to completely remove dependencies from the CAMELS dataset. This can likely be attributed to the fact that each feature is extremely dependent on the other explanatory features ($R^2 \geq 0.65$).

\begin{figure}[h!]
\centering
{\includegraphics[width=0.6\textwidth]{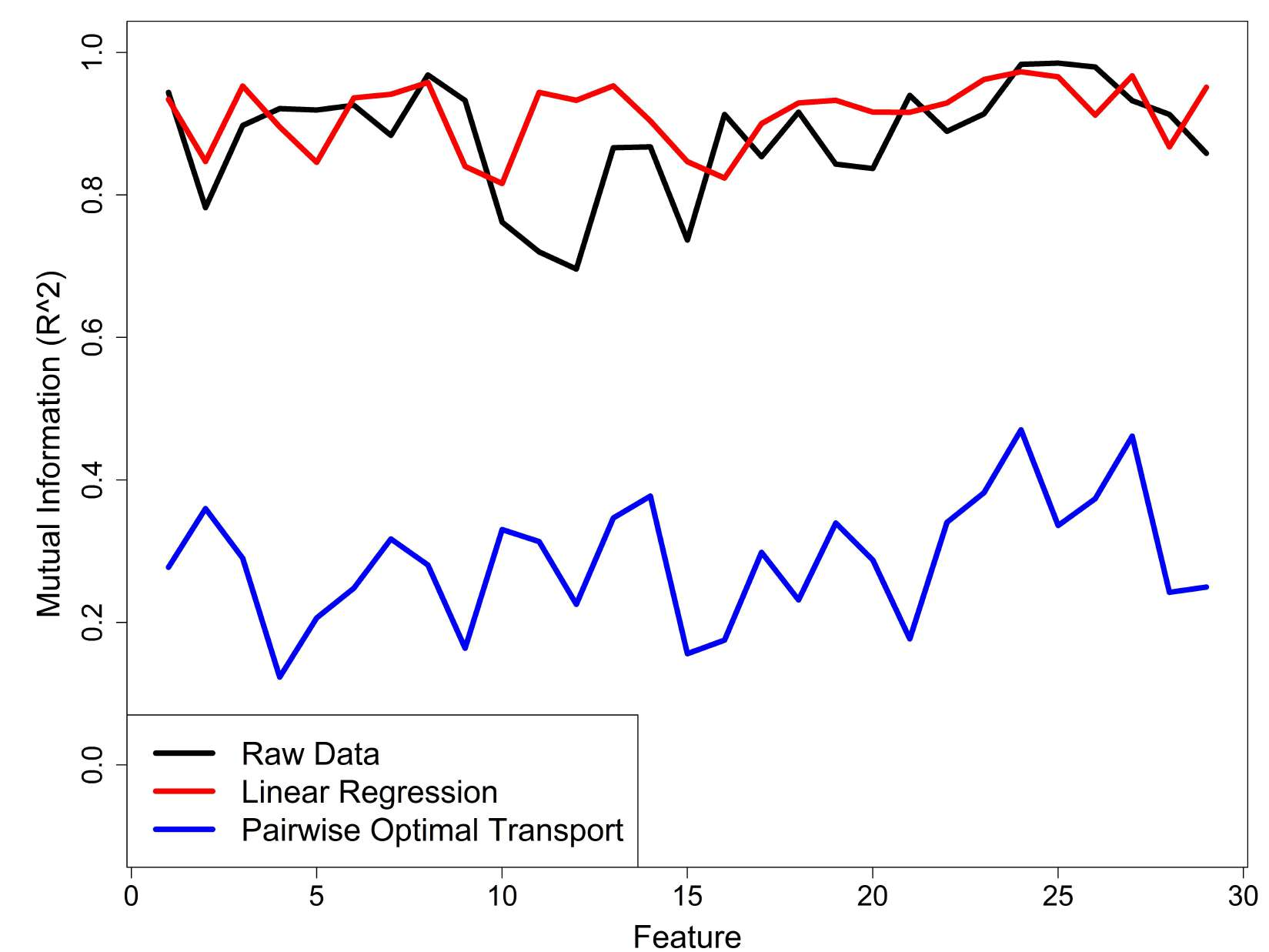}
 { \caption{The coefficient of determination (random forest OOB-$R^2$)  between the $i$th feature in the CAMELS dataset and all other features is plotted (black) for each $ i \in \{1,2,...30\}$. The $R^2$ value between the $i$th feature and all other features after preprocessing with linear regression (red) and optimal transport (blue) is also plotted.}
  \label{fig:hydrology_mutual_info}}
}
\end{figure}

The feature importance scores indicated in Figure \ref{fig:hydrology} show that mean precipitation and aridity index are the features with the strongest relationships with mean annual streamflow. Geology and soil attributes such as bedrock permeability and soil porosity are always among the least important features. These conclusions are in line with previous studies \citep{addor2018ranking,jehn2020using}, thus, even when dependencies can not be completely removed, UMFI can still provide reasonable measurements of feature importance.

\begin{figure}[h!]
{\centering
\includegraphics[width=\textwidth]{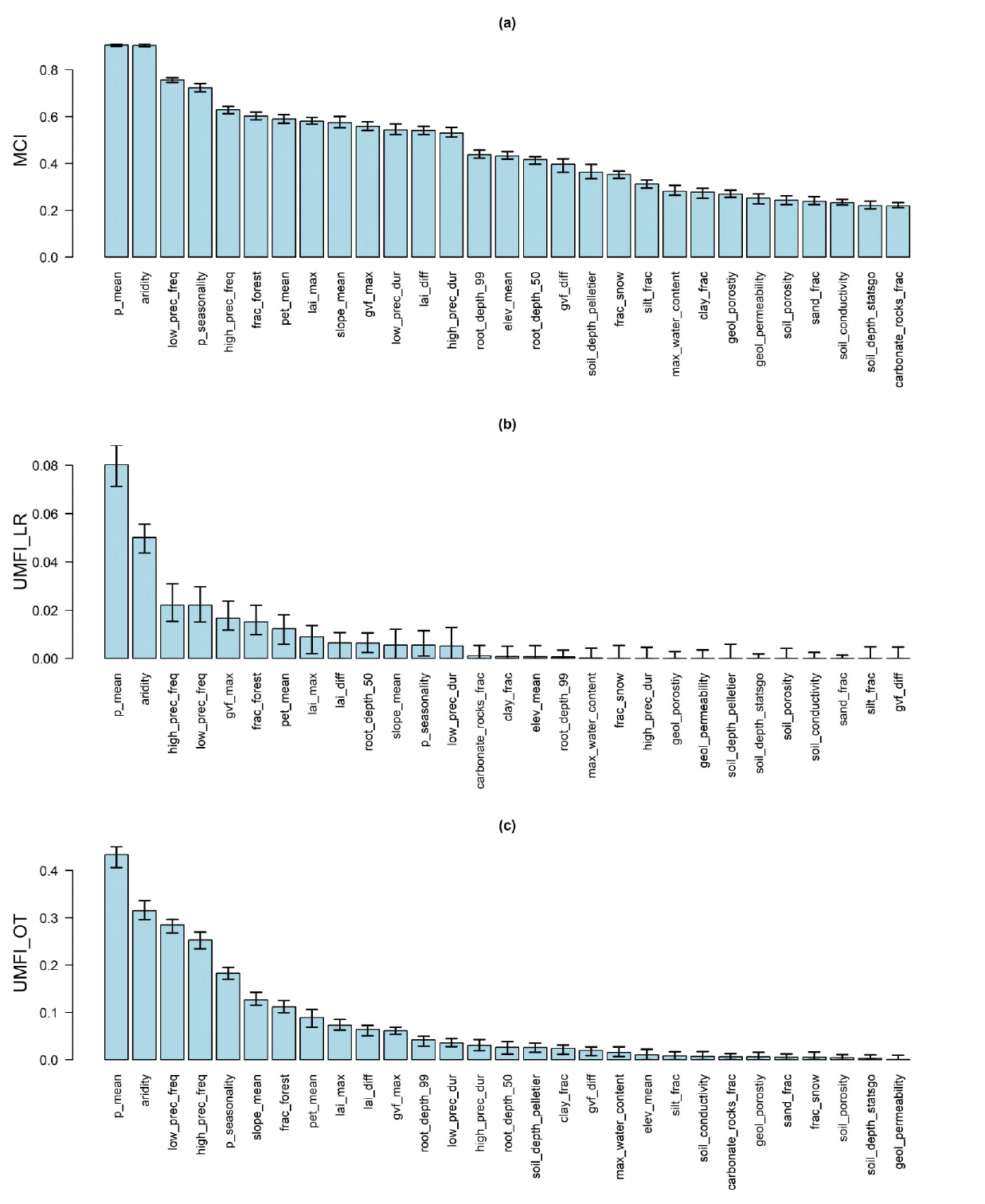}
 { \caption{Median feature importance scores provided by (a) MCI, (b) UMFI with linear regression, and (c) UMFI with pairwise optimal transport, for each explanatory variable in the CAMELS dataset, aggregated after $100$ iterations. The first and third quantiles of the scores are visualized for each feature.}
  \label{fig:hydrology}}
}
\end{figure}

\vfill

\end{document}